%% file: main.tex
\pdfoutput=1
\documentclass[10pt]{article}
\usepackage{enumerate}
\usepackage[OT1]{fontenc}
\usepackage{amsmath,amssymb}
\usepackage{natbib}
\usepackage[usenames,dvipsnames]{xcolor}
\usepackage{geometry}
\usepackage{dsfont}
\usepackage{pgfplots}
\usepackage{smile}
\usepackage{multirow}
\usepackage{rotating}
\usepackage{esvect}
\usepackage{tikz}
\usetikzlibrary{patterns}
\usetikzlibrary{arrows}
\usepackage[colorlinks,
            linkcolor=mydarkblue,
            anchorcolor=blue,
            citecolor=blue
            ]{hyperref} 
\usepackage{algorithm}
\usepackage{algorithmic}
\usepackage[most,skins]{tcolorbox}
\definecolor{rliableolive}{HTML}{BBCC33}
\definecolor{rliableblue}{HTML}{77AADD}
\definecolor{rliablered}{HTML}{EE8866}
\definecolor{LightCyan}{rgb}{0.88,1,1}
\definecolor{darkblue}{HTML}{2878D9}
\definecolor{navyblue}{HTML}{0000FF}

\usepackage{enumitem}
\usepackage{acronym}
\acrodef{svd}[SVD]{Singular-Value Decomposition}
\acrodef{rgi}[RGI]{Relative Generalization Invariance}
\acrodef{argi}[ARGI]{Asymptotic RGI}
\acrodef{llm}[LLM]{Large Language Model}
\acrodef{ntk}[NTK]{Neural Tangent Kernel}
\acrodef{mf}[MF]{Mean-Field}
\acrodef{ccc}[CCC]{Concordance Correlation Coefficient}
\acrodef{ood}[OOD]{Out-Of-Distribution}

\usepackage{wrapfig}

\definecolor{sorange}{RGB}{252,91,90}
\definecolor{sblue}{RGB}{9,48,138}
\definecolor{syellow}{RGB}{253,179,51}
\definecolor{steal}{RGB}{19,119,116}

\newcommand{\spec}{\texttt{spec}}
\newcommand{\sgn}{\texttt{sgn}}

\theoremstyle{plain}

\usepackage{mathrsfs}
\usepackage{fullpage}

\usepackage{hyperref}
\usepackage[protrusion=false,expansion=true]{microtype}

\usepackage{xcolor}
\usepackage{amsmath, amssymb}
\usepackage{float} 
\usepackage{natbib}
\usepackage[usenames, dvipsnames]{xcolor}
\usepackage{bm}
\usepackage{graphicx}
\usepackage{acronym}
\usepackage{hhline}
\usepackage{wrapfig}
\usepackage{dsfont}
\usepackage{pgfplots}
\usepackage{smile}
\usepackage{multirow}
\usepackage{rotating}
\usepackage{enumerate}
\usepackage{esvect}
\usepackage{threeparttable}

\usepackage[utf8]{inputenc} 
\usepackage[T1]{fontenc}    
\usepackage{hyperref}       
\usepackage{url}            
\usepackage{booktabs}       
\usepackage{amsfonts}       
\usepackage{nicefrac}       

\usepackage{xcolor}         

\usepackage{tikz}
\usetikzlibrary{patterns}
\usetikzlibrary{arrows}
\usepackage{longtable}         
\usepackage{booktabs}           
\usepackage{array}          
\usepackage{multirow}  
\usepackage{subcaption}
\usepackage{framed}
\colorlet{shadecolor}{orange!15}
\usepackage{algorithm}
\usepackage{algorithmic}
\usepackage{siunitx}

\definecolor{mydarkblue}{rgb}{0,0.08,0.45}
\hypersetup{
	colorlinks=true,
	urlcolor=magenta,
	citecolor=mydarkblue,
}

\usepackage[most,skins]{tcolorbox}
\definecolor{rliableolive}{HTML}{BBCC33}
\definecolor{rliableorange}{HTML}{9A5F51}
\definecolor{rliableblue}{HTML}{77AADD}
\definecolor{rliablered}{HTML}{EE8866}
\definecolor{LightCyan}{rgb}{0.88,1,1}
\definecolor{darkblue}{HTML}{2878D9}
\definecolor{navyblue}{HTML}{0000FF}

\newtcolorbox{AIbox}[2][]{aibox,title=#2,colback=rliableblue!10!white,#1}

\def\##1\#{\begin{align}#1\end{align}}
\def\$#1\${\begin{align*}#1\end{align*}}

\usepackage{arydshln}
\usepackage{diagbox}
\newcommand{\soft}{\texttt{soft}}
\acrodef{llm}[LLM]{Large Language Model}
\acrodef{ffn}[FFN]{Feed-Forward Networks}
\acrodef{mha}[MHA]{Multi-Head Attention}
\acrodef{svd}[SVD]{Singular Value Decomposition}
\acrodef{gd}[GD]{Gradient Descent}
\acrodef{oco}[OCO]{Online Convex Optimization}
\acrodef{cnn}[CNN]{Convolutional Neural Network}

\usepackage{tcolorbox}
\usepackage{xcolor}
\definecolor{lightroyalblue}{HTML}{F6F8FD} 
\definecolor{lightorange}{RGB}{252,236,219}
\definecolor{royalblue}{HTML}{4169E1}
\definecolor{lighterblue}{HTML}{f2fafd}  

\definecolor{keycolor}{HTML}{B5605C}         
\definecolor{morandibluelight}{HTML}{E2ECF7} 
\definecolor{morandibluedeep}{HTML}{6E8AAD}  
\newcommand{\kw}[1]{\textcolor{keycolor}{#1}}

\newtcolorbox{abox1}{
  colback=morandibluelight,
  colframe=morandibluedeep,
  boxrule=0.5pt, arc=2pt,
  left=8pt, right=8pt, top=6pt, bottom=6pt
}
\definecolor{LightCyan}{rgb}{.9, .95, 1.}
\definecolor{rowblue}{RGB}{232,241,255}

\usepackage{tcolorbox}
\usepackage{xcolor}

\newcounter{observation}
\renewcommand{\theobservation}{\arabic{observation}}

\newtcolorbox{abox}[1][]{
  enhanced,
  colback=morandibluelight,
  colframe=morandibluelight,
  boxrule=0pt, arc=0pt,
  left=10pt, right=8pt, top=6pt, bottom=6pt,
  borderline west={2pt}{0pt}{morandibluedeep},
  before upper={%
    \refstepcounter{observation}%
    \textbf{Observation~{\theobservation}:}\space
  },
  #1
}
\title{Muon Learns More Robust and Transferable Features than Adam}

\author{
Tianyu Ruan\textsuperscript{1,3,4$\ast$} \quad Fengzhuo Zhang\textsuperscript{1,$\ast$,$\dagger$} \quad Shuche Wang\textsuperscript{2,$\ast$}\quad Shihua Zhang\textsuperscript{3,4} \\
\textsuperscript{1} Yale University\quad
\textsuperscript{2}National University of Singapore\quad\textsuperscript{3}University of Chinese Academy of Sciences\\ \textsuperscript{4}Academy of Mathematics and Systems Science, CAS 
}
\date{}

\begin{document}

    \maketitle
\renewcommand\thefootnote{}\footnotetext{$\ast$ Equal contribution.}\footnotetext{$\dagger$ Project Lead.}\footnotetext{Correspondence to: \textless{}fengzhuo.zhang@yale.edu\textgreater{}, \textless{}zsh@amss.ac.cn\textgreater{}}

\begin{abstract}
Muon has recently emerged as a state-of-the-art optimizer for pretraining Large Language Models (LLMs) and vision classifiers. Despite its efficiency advantage over Adam and SGD, the feature-learning advantage of Muon remains unclear. This paper investigates Muon's feature-learning advantage through the lens of robustness and transferability. First, by evaluating pretrained models on corrupted images and texts, we show that features learned by Muon are consistently more robust than those learned by Adam and SGD across different architectures, including transformers and Convolutional Neural Networks (CNNs). Using trained layer-wise probes, we further show that this robustness advantage is reflected in larger logit margins across layers. Second, by training linear classifiers or fine-tuning full models from pretrained parameters on downstream tasks, we demonstrate that Muon-learned features transfer more effectively than those learned by Adam and SGD. This transferability advantage is further supported by the diversity of hidden states across layers, as measured by effective rank. Finally, in a representative classification problem with multi-component features, we prove that Muon attains larger margins and higher effective rank than Adam and SGD, providing theoretical support for our empirical findings.\looseness=-1
\end{abstract}

    \input{main_paper/intro}

    \input{main_paper/related_work}
    \input{main_paper/prelim}
    \input{main_paper/Experiments}
    \input{main_paper/theory}
    \input{main_paper/conclusion}

    \bibliography{ref}
    \bibliographystyle{ims}
    \newpage
    \appendix
    \input{appendix/experimental_details}
    \input{appendix/Notations}

    \input{appendix/proof_thm}

    \input{appendix/proof_prop}
    \input{appendix/proof_lemma}

\newpage
\end{document}

%% file: main_paper/intro.tex
\section{Introduction}
Muon (Momentum Orthogonalized by Newton-Schulz) has emerged as a powerful alternative to Adam and stochastic gradient descent (SGD) for pretraining \acp{llm} and training vision classifiers~\citep{jordan2024muon,shulgin2025beyond}. 
It treats each parameter as a matrix and computes its update direction by orthogonalizing the gradient matrix, replacing the gradient with the matrix obtained by setting all of its nonzero singular values to one. This matrix-aware design yields substantial wall-clock gain.  Muon trains nearly $2\times$ faster than Adam and SGD across a wide range of model sizes and architectures~\citep{shah2025practical,liu2025muon,jordan2025cifar10airbench}, and is now used in production-scale \ac{llm} runs~\citep{liu2025muon,team2025kimi,deepseekai2026deepseekv4,glm5team2026glm5vibecodingagentic}.


A growing line of work analyzes Muon's efficiency advantage over Adam and SGD through the lenses of steepest gradient descent~\citep{bernstein2024old}, associative memory~\citep{wang2025muon}, and heavy-tailed data~\citep{wang2025muon,Vasudeva2025HowMS}. We complement this line of work by taking a first step toward demystifying Muon's behavior from the perspective of feature learning. Specifically, we ask:
\begin{abox1}
Does Muon learn better features than Adam and SGD? If so, what properties characterize this advantage, and how are these properties manifested in model hidden states?
\end{abox1}
Through extensive experiments and theoretical analysis, we show that Muon indeed learns better features than Adam and SGD. In addition, we identify two properties that characterize this advantage --- \kw{robustness} and \kw{transferability} --- and connect both to concrete hidden-state quantities: \kw{logit margin} and \kw{effective rank}.



First, to assess robustness, we evaluate pretrained vision and language models on corrupted inputs, using corrupted inputs from ImageNet-C (vision)~\citep{hendrycks2019benchmarking} and FineWeb10B-C (language), the latter being a new benchmark we construct. Across \acp{cnn}, Vision Transformers (ViTs), and causal transformers, Muon-pretrained models consistently outperform their Adam- and SGD-pretrained counterparts. To examine \kw{how} this robustness advantage is reflected in hidden states, we probe each layer with a trained linear decoder and compute the resulting logit margin, a quantity tightly associated with classifier robustness. The probing results show that \kw{Muon yields larger margins than Adam and SGD} at nearly every layer, indicating that its robustness advantage stems from a wider class-separation gap in the learned representations.    

Second, to evaluate transferability, we adapt pretrained models to a range of downstream tasks. For vision, we train linear classifiers on top of frozen pretrained backbones; for language, we fully fine-tune pretrained models. Across all settings, Muon-pretrained features transfer more effectively than those learned by Adam and SGD, as demonstrated by better performance on downstream tasks. To examine \kw{how} this transferability advantage is reflected in hidden states, we analyze the spectrum of each layer's hidden-state matrix, formed by stacking hidden states across evaluation examples. \kw{Muon attains consistently higher effective rank and lower Top-$k$ spectral energy than Adam and SGD}, indicating that its representations spread information more evenly across orthogonal directions and span a richer feature subspace --- a property naturally aligned with downstream adaptation.  

Finally, we provide theoretical support for these hidden-state findings in a stylized one-layer classification problem with multi-component features (a shared block-level component plus a class-specific component).
In this setting, Muon’s spectral normalization removes singular-value imbalance in the gradient matrix, thereby yielding a smaller imbalance ratio between the complementary representation components. By contrast, Adam’s coordinate-wise normalization and gradient descent leave this ratio larger, resulting in more spectrally imbalanced representations. As a result, at any matched training loss, Muon provably attains a strictly larger classification margin than both Adam and gradient descent (GD), supporting its robustness advantage. In addition, Muon provably attains a strictly higher feature effective rank than Adam and GD, supporting its transferability advantage. \looseness=-1



Together, these results identify Muon's feature-learning advantage as follows:
\begin{abox1}
 The features learned by Muon are more \kw{robust} to input corruption and \kw{transfer} more effectively to downstream tasks.
These advantages are reflected respectively in larger \kw{logit margins} and higher \kw{effective rank} in the model's hidden states.
\end{abox1}


%% file: main_paper/related_work.tex
\section{Related Work}\label{app:related_work}


{\noindent \bf Adam and its variants.} Adam~\citep{kingma2014adam}  has demonstrated strong performance in the \ac{llm} pretraining over the years. A broad range of follow-up work has further improved Adam from different perspectives. On the convergence and stability side, AMSGrad~\citep{reddi2019convergence} takes a running maximum of the second-moment estimate to fix the non-convergence issue, while Yogi~\citep{zaheer2018adaptive} revises the second-moment update to control the growth of the effective learning rate. Rectified Adam (RAdam)~\citep{liu2019variance} further rectifies the variance of adaptive learning rates in early iterations, reducing sensitivity to warmup. To improve generalization, \citet{chen2018closing} propose Partially adaptive momentum estimation (Padam), which unifies Adam with SGD by introducing a partially adaptive parameter.  AdamW~\citep{loshchilov2017decoupled} decouples weight decay from the adaptive update and has become a standard baseline in modern deep learning. In large-scale pretraining, LAMB~\citep{you2019large} introduces layer-wise adaptive scaling to support very large-batch optimization, whereas Adafactor~\citep{shazeer2018adafactor} factorizes second-moment statistics to reduce optimizer memory and is widely used in sequence modeling and large language models. Adam-style methods have also been adapted to specialized settings, including federated learning~\citep{reddi2020adaptive}, and non-Euclidean parameter spaces via Riemannian Adam~\citep{becigneul2018riemannian}.

\vspace{5pt}

{\noindent \bf Adam vs. SGD.}
Understanding why Adam and AdamW work has attracted substantial attention. Early studies observed that adaptive methods often optimize faster than SGD but can generalize worse, suggesting that their advantage is not explained by training loss alone \citep{wilson2017marginal}. A parallel line of work studies convergence. It begins with the non-convergence counterexample for Adam and the AMSGrad remedy \citep{reddi2019convergence}, and later establishes guarantees for broader Adam-type methods \citep{chen2018convergence,Zhou2018OnTC,Zou2018ASC,defossez2020simple} and, under refined assumptions, even vanilla Adam in nonconvex settings \citep{zhang2022adam,li2023convergence}. Complementing this, \cite{pan2023toward} attributes Adam's empirical advantage over SGD to the lower directional sharpness of its update steps, enabled by coordinate-wise adaptive scaling. Mechanistic accounts emphasize Adam's sign-like updates and variance normalization \citep{balles2018dissecting}, its preconditioning effects \citep{das2024towards}, and how its preconditioned curvature evolves during training \citep{cohen2022adaptive}, while recent work argues that such adaptivity is especially valuable in Transformers and language models, due to heterogeneous curvature \citep{zhang2024transformers} and heavy-tailed data statistics \citep{kunstner2024heavy}. Another line studies implicit bias and regularization: the convergent direction of adaptive methods has been characterized on homogeneous networks \citep{wang2021implicit}, Adam and AdamW can converge to solutions with geometry distinct from SGD \citep{zou2021understanding,xie2024implicit}, and decoupled weight decay avoids the interaction between adaptive scaling and L2 regularization, often improving generalization \citep{loshchilov2017decoupled}. For a broader overview, we refer to the survey by \citet{abdulkadirov2023survey}.

\vspace{5pt}

{\noindent \bf Muon and its variants.}  Muon has recently emerged as a promising optimizer for neural network training because of its strong empirical performance and distinctive matrix-structured update rule. Introduced as a method based on orthogonalized updates for hidden-layer weight matrices \citep{jordan2024muon}, it has subsequently been extended in three main directions. First, scalability-oriented variants adapt Muon to large-scale LLM training by improving update calibration, regularization, and distributed execution. \citet{liu2025muon} scale Muon to large models by adding weight decay and adjusting the per-parameter update scale, while MuonClip~\citep{team2025kimi} and Muon with Block-Periodic Orthogonalization (MuonBP)~\citep{khaled2025muonbp} focus on training stability and communication efficiency in tensor-parallel systems, respectively. Second, several methods inject adaptive statistics into Muon’s structured updates. Neuron-wise Normalized Muon (NorMuons)~\citep{li2025normuon} adds neuron-wise learning-rate adaptation and row-wise normalization, Muon+~\citep{zhang2026muon+} applies a lightweight post-orthogonalization normalization, and variance-aware or Adam-style hybrids such as Muon-NSR, Muon-VS~\citep{li2026variance} and Muon$^2$~\citep{liu2026muon} combine noise or second-moment information with Muon’s matrix-level geometry. AdaMuon~\citep{si2025adamuon} takes a related approach, combining element-wise adaptivity with orthogonal updates. Third, recent work generalizes Muon beyond per-layer matrices: TEON~\citep{zhang2026teon} performs tensor-level orthogonalization to capture cross-layer correlations, and MuonRec~\citep{shan2026muonrec} shows that Muon-style updates can also benefit recommendation models.

\vspace{5pt}

{\noindent \bf Advantage of Muon.}
Since Muon was introduced by \cite{jordan2024muon},  recent work has sought to explain its advantage from both optimization-theoretic and mechanistic perspectives. Building on the perspective that classical optimizers can be viewed as steepest descent under different operator norms~\citep{bernstein2024old}, one line interprets Muon through geometric and convergence analyses: viewing it as steepest descent under the spectral norm~\citep{li2025note} or a non-Euclidean trust-region method~\citep{kovalev2025understanding}, showing that with decoupled weight decay it implicitly enforces a spectral-norm constraint on the weights~\citep{chen2025muon}, relating its advantage over GD to the low-rank and approximately block-diagonal Hessian structure~\citep{shen2025convergence}, and clarifying the roles of momentum, weight decay, and critical batch size~\citep{sato2025analysis,shah2025practical}. Related work connects Muon to broader spectral and matrix-aware optimization principles, emphasizing the role of spectral scaling~\citep{yang2023spectral} in feature learning and distinguishing matrix-aware methods that address gradient anisotropy from Adam-style methods that address curvature anisotropy~\citep{lau2025polargrad}. A complementary line studies Muon through spectral dynamics and associative memory. These works relate Muon's gains to its effect on associative memory and improved tail-end learning under heavy-tailed data \citep{wang2025muon}, and show that spectral updates can produce more uniform learning across frequencies \citep{Li2026MuonIA} or principal components, improving both optimization and generalization on imbalanced data \citep{Vasudeva2025HowMS}. Related studies further characterize near-uniform singular-value growth in matrix factorization \citep{Kang2026UniformSG}, formalize orthogonalization as spectral preconditioning \citep{Ma2026PreconditioningBO}, and relate Muon to a broader family of spectral, structured-gradient \citep{an2025asgo}, and norm-constrained \citep{pethick2025training} optimization methods.
Our work complements this line of research by shifting the focus from optimization efficiency to feature quality. Specifically, we establish that Muon learns more robust and transferable features than Adam and SGD, and connect these advantages to larger logit margins and higher effective rank in the hidden states. These properties provide a feature-learning interpretation of Muon's empirical advantages.

\vspace{5pt}
{\noindent \bf  Feature learning: empirical perspectives.}     Feature learning is a central aspect of deep neural networks and is often regarded as a key source of their empirical success, yet the nature of the features they learn often remains poorly understood. This has motivated a large body of empirical work seeking to understand what features neural-network representations encode and how such features can be characterized across different domains. In the computer vision literature, early visualization~\citep{zeiler2014visualizing} and inversion~\citep{mahendran2015understanding} studies suggested that CNN representations become increasingly invariant and semantically structured with depth.
Follow-up work quantified this layerwise organization using transfer learning and representation-similarity metrics, showing that lower layers are typically more general and deeper layers are more task-specific~\citep{yosinski2014transferable}, with lower layers, but not deeper ones, learning similar representations across different datasets~\citep{kornblith2019similarity}. A complementary line asked whether learned features align with human-interpretable concepts, revealing concept detectors~\citep{bau2017network}, texture bias~\citep{geirhos2018imagenet}, and reliance on non-robust yet predictive cues~\citep{ilyas2019adversarial}. More recent work has turned to self-supervised and foundation-model representations, showing that training objectives strongly shape the emergent semantic structure of features~\citep{caron2021emerging}, feature suppression~\citep{xue2023features}, robustness~\citep{oquab2023dinov2}, and the extent to which visual models encode three-dimensional structure~\citep{el2024probing}. Similar questions have been pursued in language models. Probing and interpretability studies reveal a layerwise progression from lower-level to higher-level information \citep{tenney2019bert,jawahar2019does,jin2025exploring}, while their representations also encode information in sparse \citep{huben2023sparse,dunefsky2024transcoders}, distributed \citep{elhage2022toy}, and geometrically structured \citep{hewitt2019structural,gurnee2023language,engels2024not} form. Recent benchmarks in the language domain further evaluate representation quality through downstream transfer \citep{Enevoldsen2025MMTEBMM} and robustness under distribution shift \citep{Yuan2023RevisitingOR}.

\vspace{5pt}
{\noindent \bf  Feature learning: theoretical perspectives.} 
These empirical regularities have motivated a broad theory literature asking when neural networks truly learn useful features, rather than merely fit data with nearly fixed random representations. Classical representation-learning work emphasized invariance and disentanglement as desiderata for useful latent variables \citep{bengio2013representation}. Early deep-linear-network analyses showed that depth alone induces nontrivial optimization and representation dynamics \citep{saxe2013exact}. A central question in later theory is whether features evolve substantially during training. In lazy regimes~\citep{Chizat2018OnLT}, as captured by the neural tangent kernel (NTK)~\citep{jacot2018neural} and related infinite-width analyses, learned representations remain close to initialization. By contrast, mean-field analyses~\citep{Chizat2018OnTG,Mei2019MeanfieldTO} and beyond-NTK results~\citep{li2020learning,allen2019can} study regimes with genuine feature evolution. Building on this distinction, recent works further analyze how optimization, regularization, and training objectives induce feature selection and refinement. Examples include hierarchical feature learning~\citep{allen2023backward}, feature purification under adversarial training~\citep{allen2022feature}, multi-view features in distillation~\citep{allen2020towards}, and optimizer-dependent behavior such as the contrast between GD and Adam~\citep{zou2021understanding}. Orthogonal perspectives relate useful representations to minimality and invariance~\citep{Achille2017EmergenceOI}, connect information-theoretic compression to generalization~\citep{ShwartzZiv2017OpeningTB,Saxe2018OnTI}, characterize the late-stage geometry of supervised features via neural collapse \citep{Papyan2020PrevalenceON}, and analyze unsupervised and self-supervised learning through downstream guarantees~\citep{Arora2019ATA}, alignment--uniformity~\citep{wang2020understanding}, and identifiability limits~\citep{locatello2019challenging}. Our work contributes to this line by isolating the role of the optimizer in feature learning, showing that Muon, Adam, and SGD induce systematically different feature qualities under matched architectures and training budgets.

%% file: main_paper/prelim.tex
\section{Preliminaries}

This section introduces the update rules of Adam and Muon, which form the basis for the empirical comparisons in Section~\ref{sec:experiments} and the theoretical analysis in Section~\ref{sec:case_study}.

\textbf{Adam} has been widely adopted for \ac{llm} pre-training as an adaptive alternative to SGD~\citep{kingma2014adam}. 
Following standard practice, we describe each optimizer at the level of a single parameter matrix; in deep models, the optimizer is applied independently to each matrix-valued parameter. 
Let $W_t \in \Theta$ denote the parameter matrix at step $t$, where $\Theta$ is the parameter space, and let $G_t$ denote its gradient. SGD updates $W_t$ according to
$
W_{t+1} = W_t - \eta_t G_t,
$
where $\eta_t>0$ is the learning rate. In contrast, Adam introduces element-wise adaptive updates based on two exponential moving averages~\citep{kingma2014adam}. The first estimates the gradient direction as
$
M_t = \beta_1 M_{t-1} + (1-\beta_1) G_t,
$
where $\beta_1\in[0,1]$ controls the smoothing level of the gradient estimate. The second estimates the element-wise squared gradient as
$
V_t = \beta_2 V_{t-1} + (1-\beta_2)(G_t \odot G_t),
$
where $\beta_2\in[0,1]$ controls the smoothing level of the magnitude estimate, and $\odot$ denotes the Hadamard product. 
Adam then normalizes each coordinate of the gradient estimate by dividing by the square root of the second-moment estimate:
$$
W_{t+1}
=
W_t
-
\eta_t \hat M_t/(\hat{V}_t^{1/2}+\epsilon),
$$
where $\hat M_t = M_t/(1-\beta_1^t)$ and $\hat V_t = V_t/(1-\beta_2^t)$ are bias-corrected estimates, $\epsilon>0$ prevents division by zero, and the square root is also applied element-wise. AdamW \citep{loshchilov2017decoupled} further improves Adam by decoupling weight decay from the adaptive update as
$$
W_{t+1}
=
(1-\eta_t\lambda)W_t
-
\eta_t \hat M_t/(\hat{V}_t^{1/2}+\epsilon),$$
where $\lambda>0$ is the weight decay coefficient.

\textbf{Muon} departs from Adam by explicitly leveraging the matrix structure of each parameter. Instead of rescaling gradient coordinates element-wise, it orthogonalizes the gradient matrix and uses the result as the update direction. Concretely, given $W_t$ and $G_t$, Muon first forms a momentum-based gradient estimate
$
M_t = \beta M_{t-1} + (1-\beta)G_t,
$
where $\beta\in[0,1]$ controls the exponential smoothing level. It then updates the parameters using a spectrally normalized version  of $M_t$ as
$$
W_{t+1}
=
W_t
-
\eta_t U_tV_t^\top,
$$
where $M_t=U_t\Sigma_tV_t^\top$ is the \ac{svd} of $M_t$. Thus, Muon removes the singular values $\Sigma_t$ and preserves only the orthogonal directions of the gradient estimate. In practice, the exact normalization is approximated by a small number of Newton--Schulz iterations for computational efficiency. Recent studies demonstrate the effectiveness of Muon in large-scale \ac{llm} pre-training~\citep{liu2025muon,team2025kimi}. See Appendix~\ref{app:related_work} for more related works.

%% file: main_paper/Experiments.tex
\section{Feature Learning: Muon vs.\ Adam and SGD}
\label{sec:experiments}

In this section, we empirically demonstrate that Muon induces better feature learning than Adam and SGD across various models and tasks. We assess feature quality from two common and practical perspectives: \kw{robustness} and \kw{transferability}. Robustness experiments, presented in Section~\ref{sec:robustness}, evaluate pretrained models on corrupted datasets and measure the resilience of learned features to input perturbations. Transferability experiments, presented in Section~\ref{sec:transfer}, fine-tune the pretrained models on downstream tasks and measure whether the learned features support effective adaptation to new tasks. \looseness=-1

\vspace{5pt}
\noindent \textbf{Experiment setups.} Our experiments cover both vision- and language-related tasks. In the vision experiments, we pretrain ResNet-18 (11M)~\citep{he2016deep} and ViT-S (22M) \citep{dosovitskiy2020image} on ImageNet-1K~\citep{deng2009imagenet,russakovsky2015imagenet} for object classification. 
ResNet-18 is trained for 100 epochs, consistent with prior supervised training setups~\citep{chrysos2022augmenting}, and ViT-S is trained for 300 epochs, following DeiT~\citep{touvron2021training}. These well-established recipes ensure that all optimizers are evaluated under a standardized and sufficiently long training protocol. For both robustness and transfer evaluations in vision, we report Top-$1$ accuracy. \looseness=-1

In the language experiments, we pretrain nanoGPT-style causal language models on FineWeb~\citep{penedo2024fineweb} for next-token prediction. 
Since SGD does not train GPT-style models effectively in this setting~\citep{zhao2024deconstructing}, we only compare  Muon to Adam for \ac{llm} training. We use a 12-layer model (124M, GPT-2) as the main setting and scale up to a 24-layer model (354M, GPT-2 Medium) to verify that our findings persist at a larger scale. In language robustness and transferability experiments, we use perplexity to evaluate the performance~\citep{touvron2023llama}. 
Moreover, in our LLM experiments, the token-to-parameter ratio is at least $2\times$ the Chinchilla ratio~\citep{hoffmann2022training}. This places our models in a sufficiently trained regime and allows us to examine optimizer behavior after substantial pretraining. 

In all experiments, we use the same hyperparameter search protocol for each optimizer to ensure a fair comparison. Throughout the paper, we compare optimizers under a matched-budget protocol: model architecture, data pipeline, training budget (epochs or tokens), and evaluation protocol are held fixed across optimizers, while each optimizer's learning rate and scheduler are tuned independently. This isolates differences in learned representations from differences in engineering effort. See Appendix~\ref{app:exp_details} for the experimental details and Appendix~\ref{app:results_more} for full numerical results with mean $\pm$ standard deviation. 

\subsection{Robustness to Input Corruptions}
\label{sec:robustness}

\begin{figure}[t]
    \centering
  \includegraphics[width=0.49\linewidth]{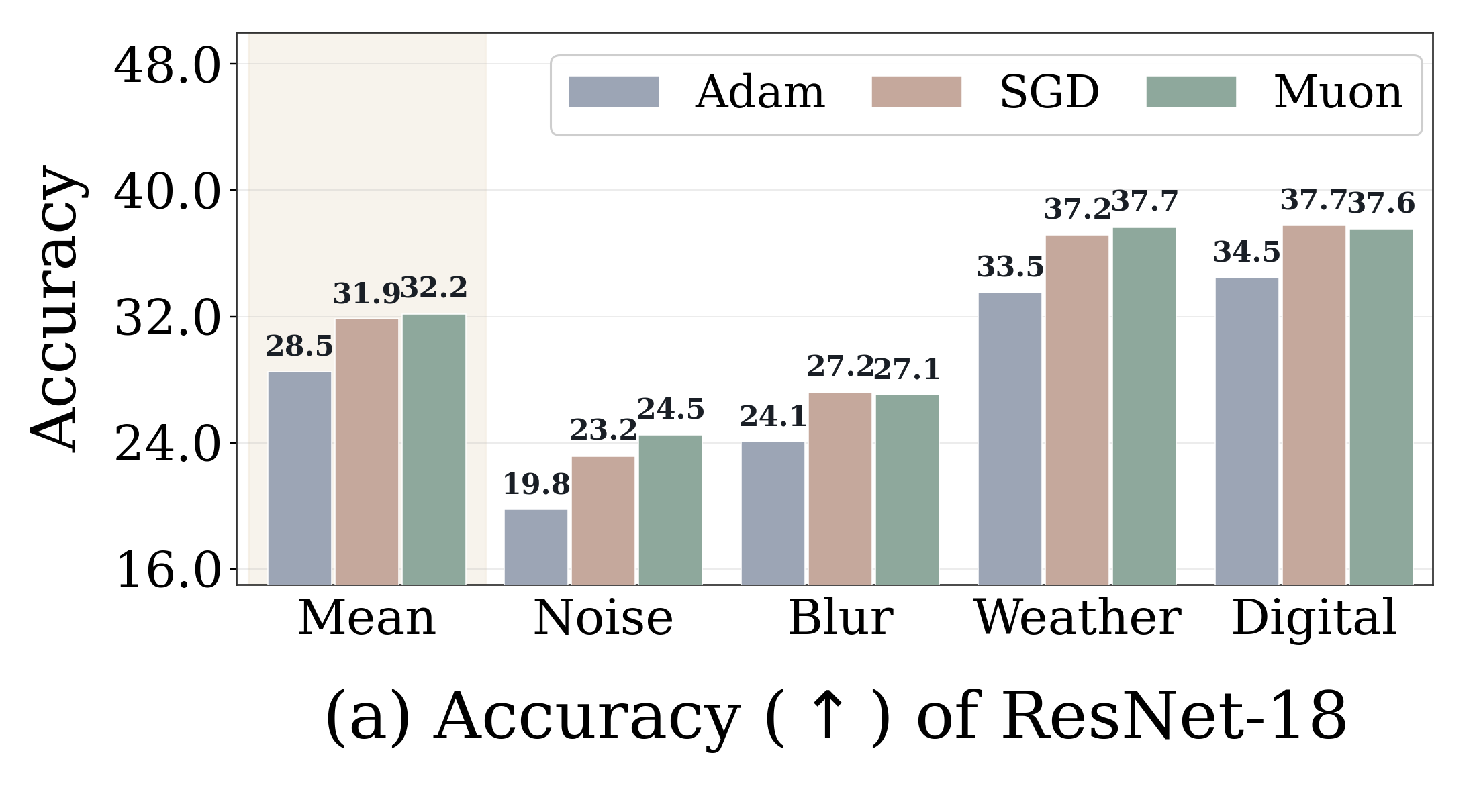}%
  \includegraphics[width=0.49\linewidth]{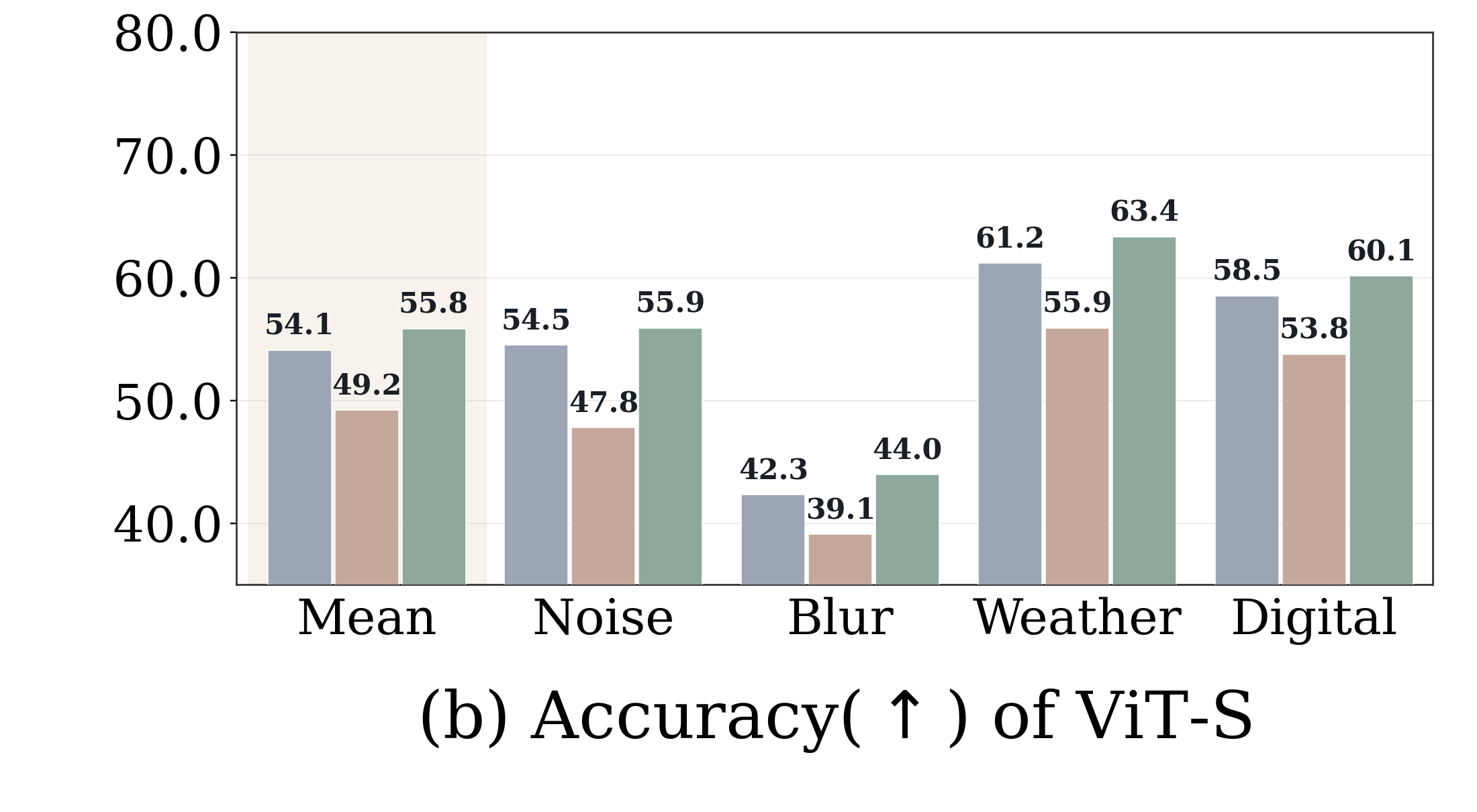}
  \\
  \includegraphics[width=0.49\linewidth]{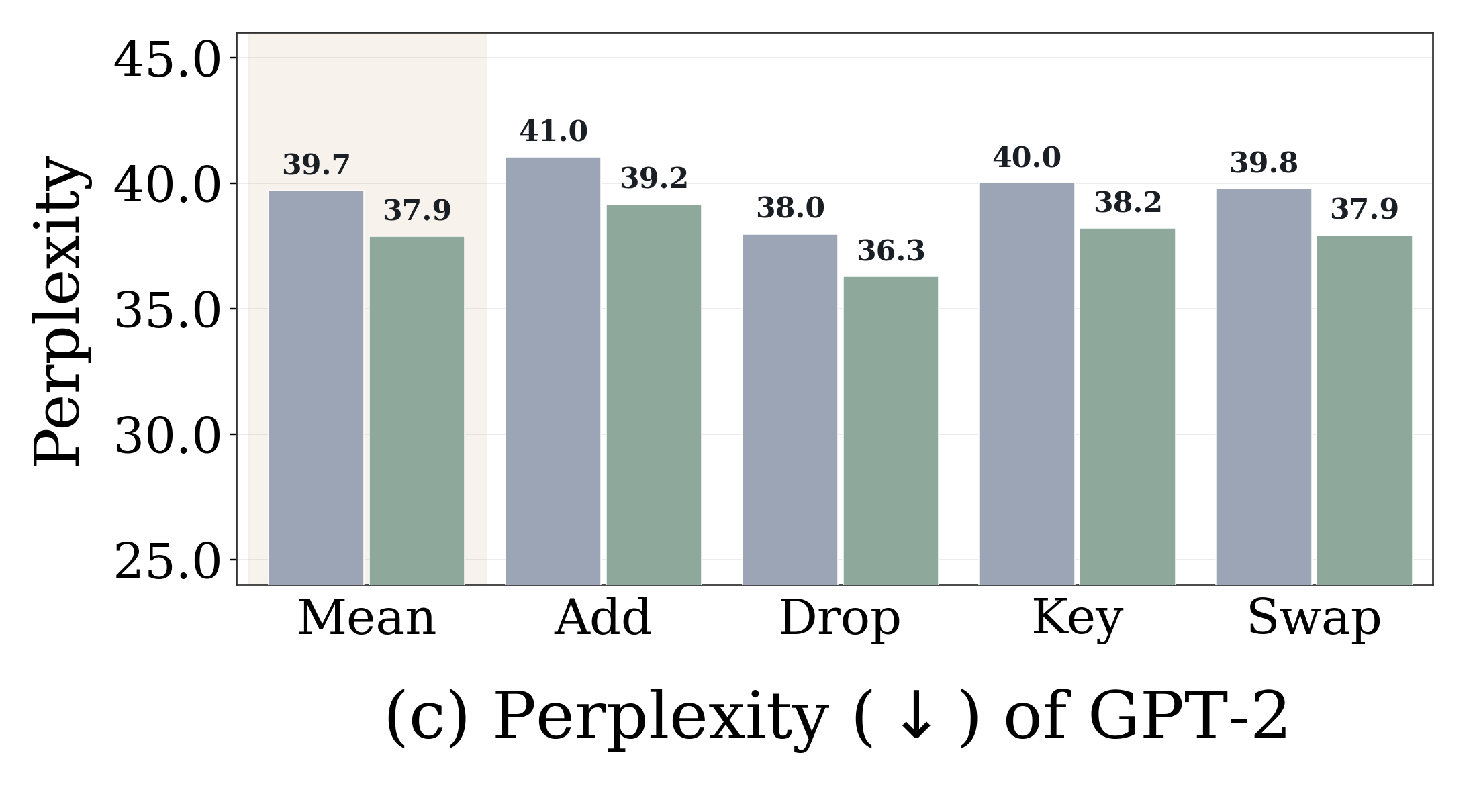}
  \includegraphics[width=0.49\linewidth]{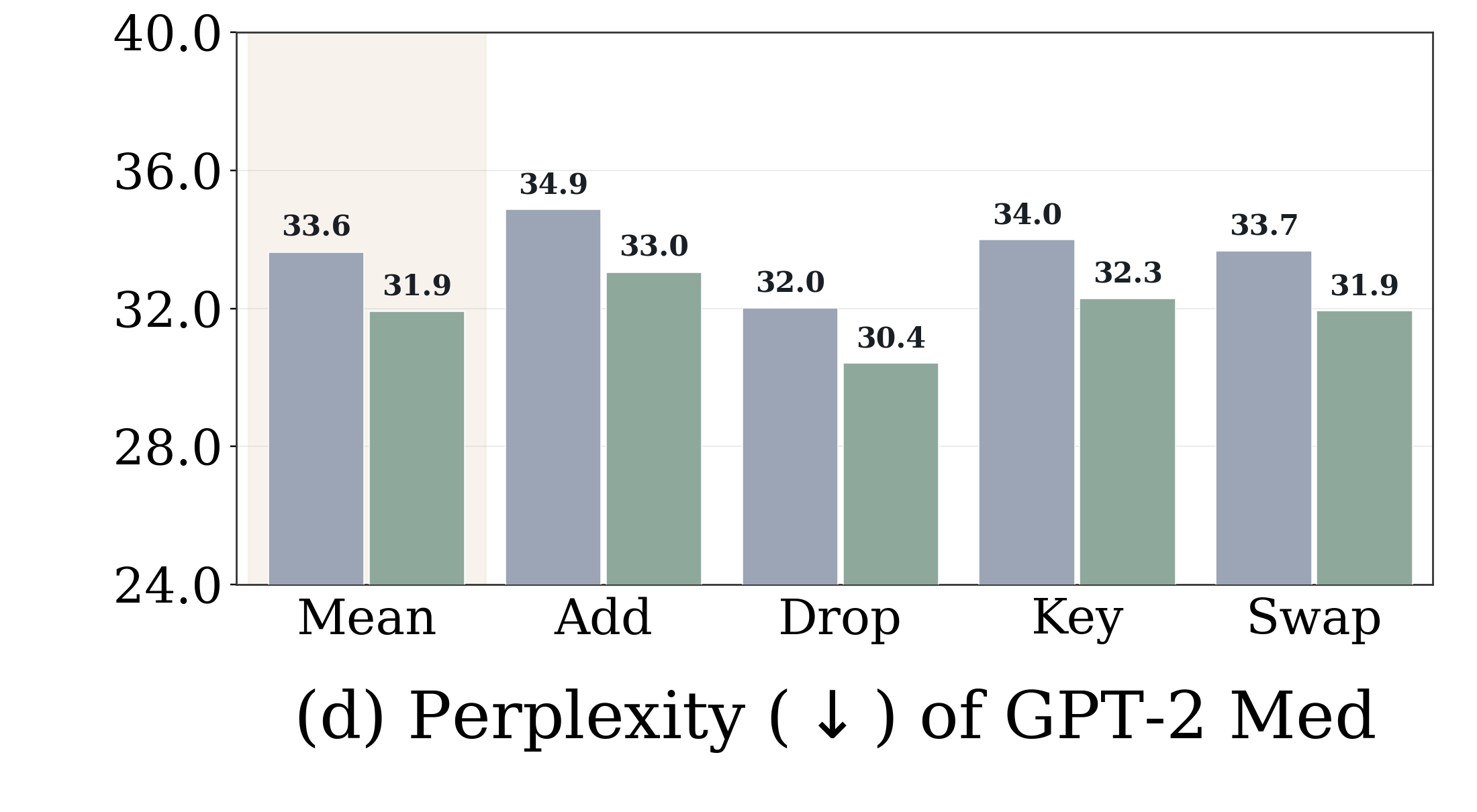}
    \caption{Robustness evaluations on vision and language tasks. (a) ResNet-18 and (b) ViT-S are evaluated on ImageNet-C, while (c) GPT-2 and (d) GPT-2 Medium are evaluated on FineWeb10B-C. Shaded bands denote $\pm$ one standard deviation over three random seeds. Muon achieves the best mean performance in all four panels, with consistent gains across transformer-based vision and language models. These suggest more robust learned features of Muon under corrupted inputs.}
    \label{fig:robustness_main}
\end{figure}

For the robustness evaluation, we first pretrain models on the clean data specified in the experimental setup, and then evaluate the pretrained models on corrupted data. In the vision experiments, pretrained models are evaluated on ImageNet-C, which applies $15$ common visual corruptions to ImageNet validation images at $5$ severity levels. These corruptions are grouped into four families: \texttt{Noise}, \texttt{Blur}, \texttt{Weather}, and \texttt{Digital}. We report results averaged over the 5 severity levels, both for each corruption family and as an overall mean across all 15 corruption types. In the language experiments, pretrained models are evaluated on a constructed corrupted split, denoted FineWeb10B-C. It is constructed using common corruption types studied in natural language processing~\citep{pruthi2019combating}. Specifically, it applies fixed character-level typo attacks to the clean validation text before tokenization, including character insertions, deletions, QWERTY-keyboard neighbor substitutions, and character swaps. We denote these four corruption types as \texttt{Add}, \texttt{Drop}, \texttt{Key}, and \texttt{Swap}, respectively, and report results for each type as well as their average. Details of the FineWeb10B-C construction are provided in Appendix~\ref{app:fineweb10c}.

\vspace{5pt}
\noindent \textbf{Muon improves robustness under input corruptions.} Figure~\ref{fig:robustness_main} reports the performance of models on corrupted data across all four settings. For ResNet-18 (panel (a)), Muon achieves the highest mean accuracy and performs best on \texttt{Noise} and \texttt{Weather}, while being slightly below SGD on \texttt{Blur} and \texttt{Digital}. For ViT-S (panel (b)), Muon achieves the highest mean accuracy and is best under every corruption family. For GPT-2 and GPT-2 Medium (panels (c) and (d)), Muon consistently achieves the lowest mean perplexity and the lowest perplexity under every corruption type. In addition, SGD outperforms Adam on the CNN architecture, whereas Adam outperforms SGD on the transformer architecture, consistent with the observations in \citet{zhang2024transformers}. We summarize the empirical findings as follows.

\begin{abox}
\label{obs:robust}
Muon achieves the  best overall performance on corrupted inputs across both vision and language tasks, indicating stronger \kw{robustness} of the learned features.
\end{abox}


\vspace{5pt}
\noindent\textbf{Layer-wise analysis of logit margins.} 
\begin{figure}[t]
    \centering
    \includegraphics[width=0.4\linewidth]{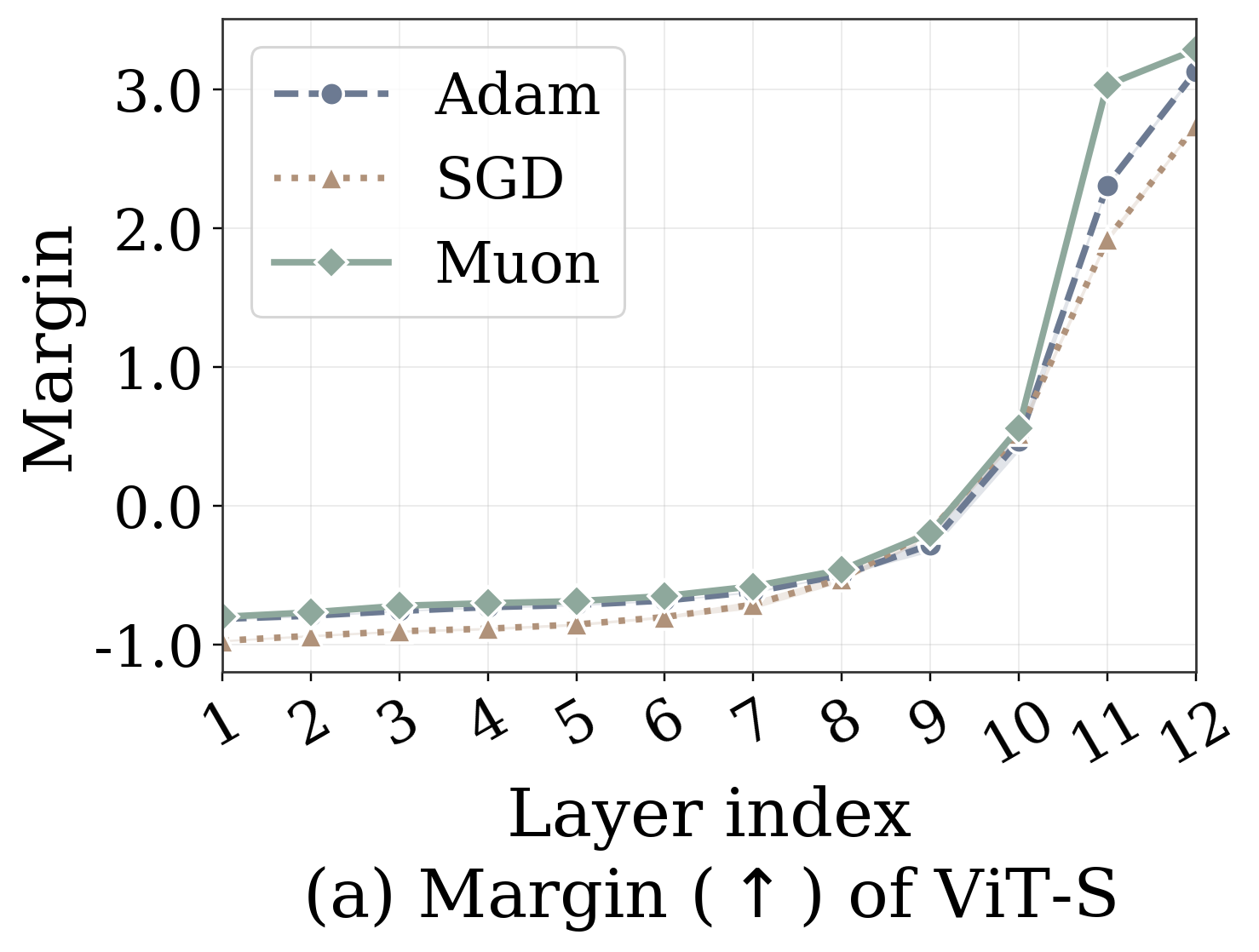}
    \hspace{0.04\linewidth}
    \includegraphics[width=0.4\linewidth]{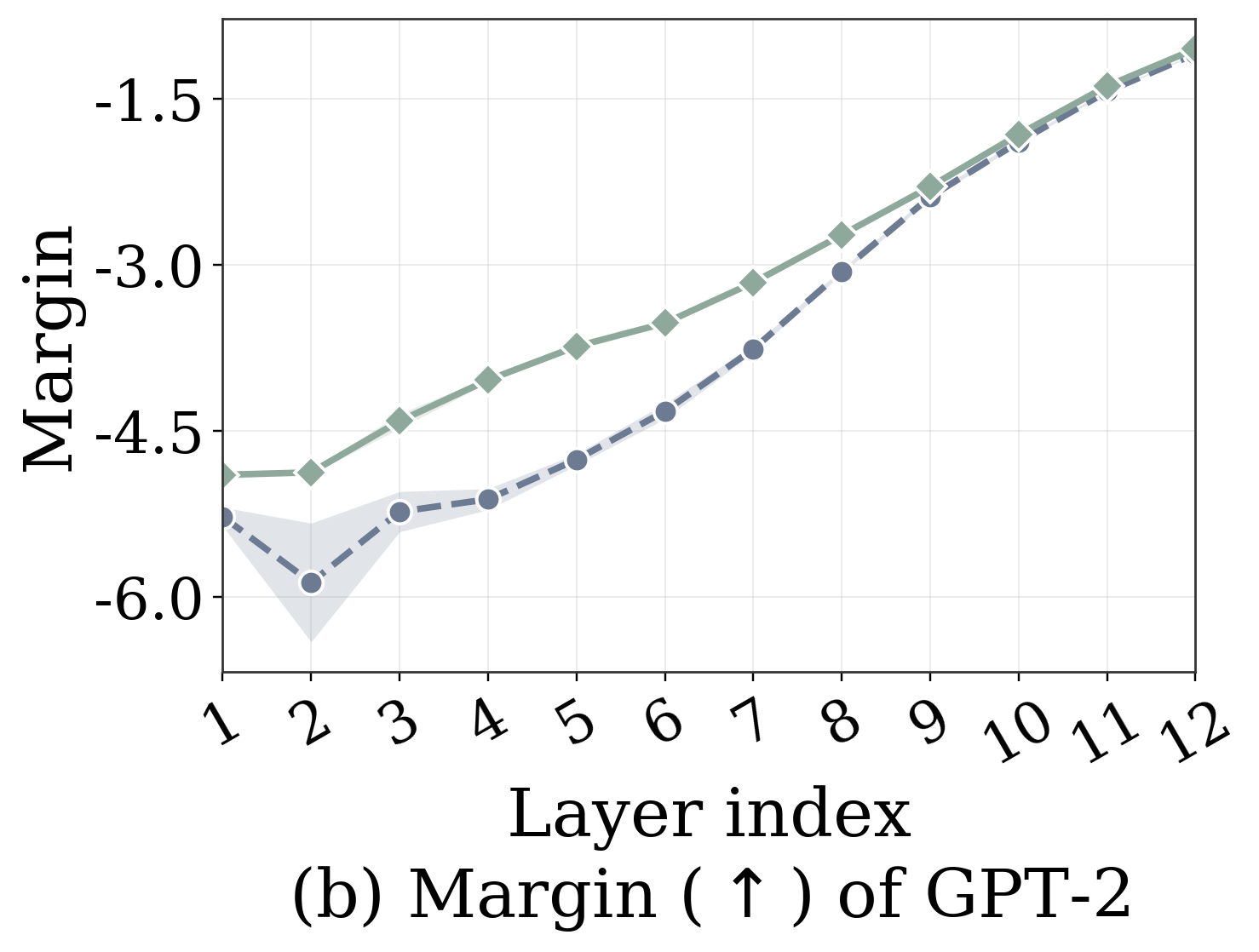}
    \caption{
Layer-wise classification margin for (a) ViT-S and (b) GPT-2. For each layer, we train a linear probe and report the mean sample-level margin (correct logit minus the largest distractor logit) on validation data. Shaded bands denote $\pm$ one
standard deviation over three seeds. Across both architectures, Muon achieves larger margins across most layers, suggesting stronger representation-level robustness.
}
\label{fig:margin}
\end{figure}
Observation~\ref{obs:robust} confirms the superior robustness of Muon over Adam and SGD through end-to-end performance. We next take a closer look at the hidden states across layers of pretrained models to understand how Muon's robustness emerges throughout the model. 
To connect hidden representations with robustness, we compute layer-wise \kw{logit margins}, which have been shown to be positively associated with robustness across a wide range of models, including support vector machines and neural networks~\citep{wei2019improved,dingmma,cheng2022cat}. Specifically, we adopt the tuned-lens framework of \citet{belrose2023eliciting} to train a probe that decodes the hidden state at each layer into a distribution over classes for vision tasks and over the vocabulary for language tasks. Details of probe training are deferred to Appendix~\ref{app:exp_details}. For the \(\ell\)-th layer, let \(z^{(\ell)}(x)\in\mathbb{R}^d\) denote the hidden state at the output of the \(\ell\)-th transformer block, which we use for probing, and let \(T^{(\ell)}\) denote the trained probe mapping \(z^{(\ell)}(x)\) to predicted logits. We define the margin of input $x$ with respect to label $y$ as
\begin{align}
    \gamma_\ell(x,y) = T^{(\ell)}_y(z^{(\ell)}(x)) - \max_{j \neq y} T^{(\ell)}_{j}(z^{(\ell)}(x)).
    \label{eq:layerwise_margin}
\end{align}
Here, the subscript \(i\) in \(T^{(\ell)}_i(\cdot)\) indexes the class, so \(T^{(\ell)}_i(z^{(\ell)}(x))\) is the predicted logit for class \(i\). In our experiments, $y$ denotes the ground-truth class for vision tasks and the ground-truth next-token index for language tasks. A larger margin means that the correct class or token is separated from the strongest competing prediction by a wider logit gap. Such a gap allows the prediction to tolerate larger logit perturbations from the noise before the top prediction changes, and is therefore associated with stronger robustness.


Figure~\ref{fig:margin} reports the average margin on corrupted evaluation data. We focus on transformer architectures, namely ViT-S and GPT-2, because the high dimensionality of CNN feature maps makes efficient probe training difficult. The results show that \kw{Muon attains the largest margin} across most layers of both ViT-S and GPT-2.\footnote{The margins of well-trained GPT models can remain negative even on clean evaluation data, because next-token prediction is performed over a very large vocabulary, which has size $50{,}257$ in our experiments. In this setting, multiple tokens may be plausible continuations, so the ground-truth next token does not need to receive the highest probed logit.} For ViT-S, Muon's margin advantage is amplified in the last two layers. At the final layer, Muon's margin is $3.29$, which is $5\%$ higher than Adam's $3.13$ and $20\%$ higher than SGD's $2.74$. For GPT-2, Muon also achieves consistently larger margins than Adam across most layers. Results for GPT-2 Medium are deferred to Appendix~\ref{app:24d_results}. These results support the following conclusion.

\begin{abox}
\label{obs:margin}
Muon yields larger \kw{logit margins} than Adam and SGD across most layers of both GPT and ViT-S, reflecting its representation-level robustness advantage.
\end{abox}

\subsection{Transferability Across Downstream Tasks}
\label{sec:transfer}
\begin{figure}[t]
    \centering

    \includegraphics[width=0.2685\linewidth]{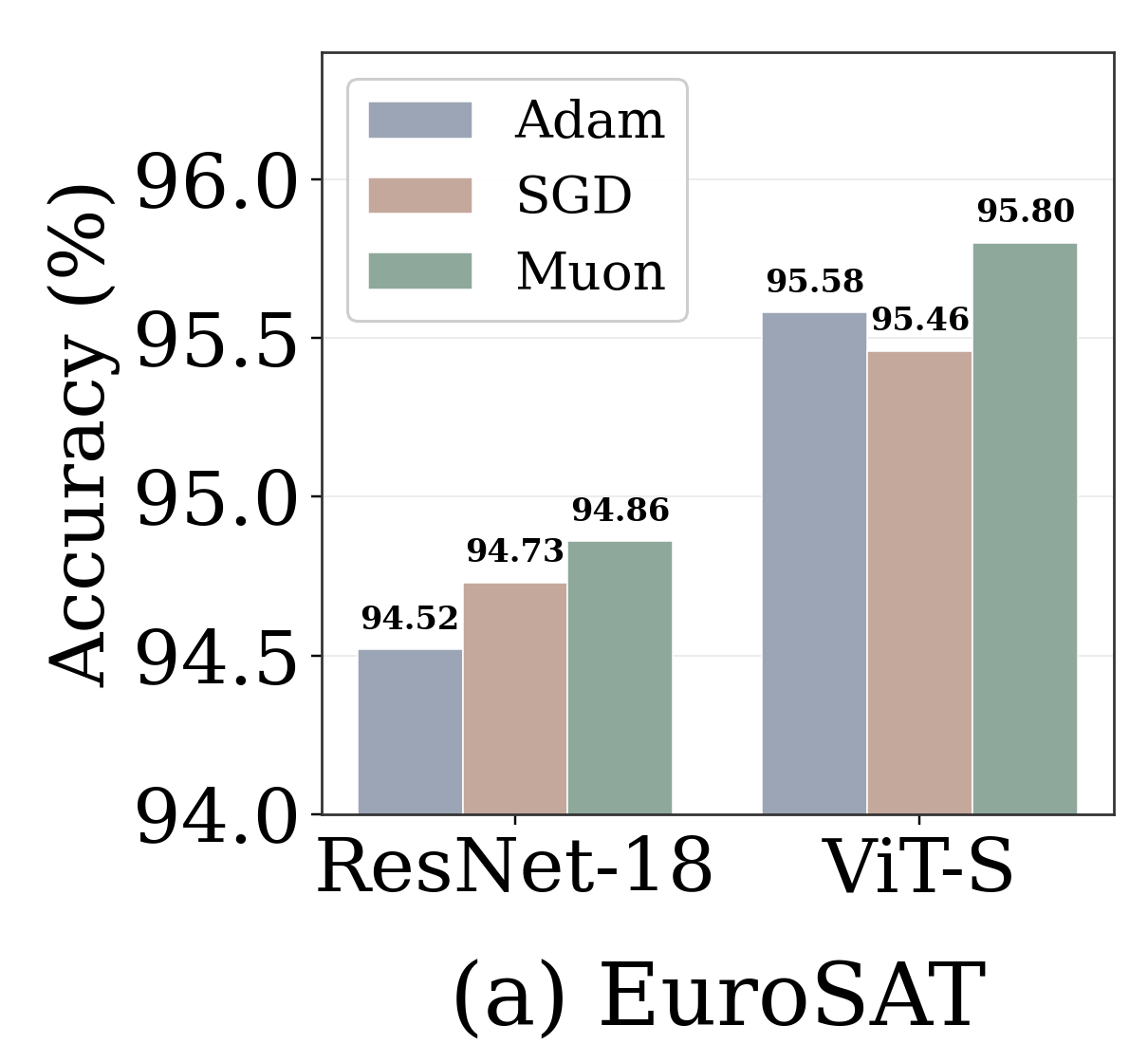}%
\includegraphics[width=0.2372\linewidth]{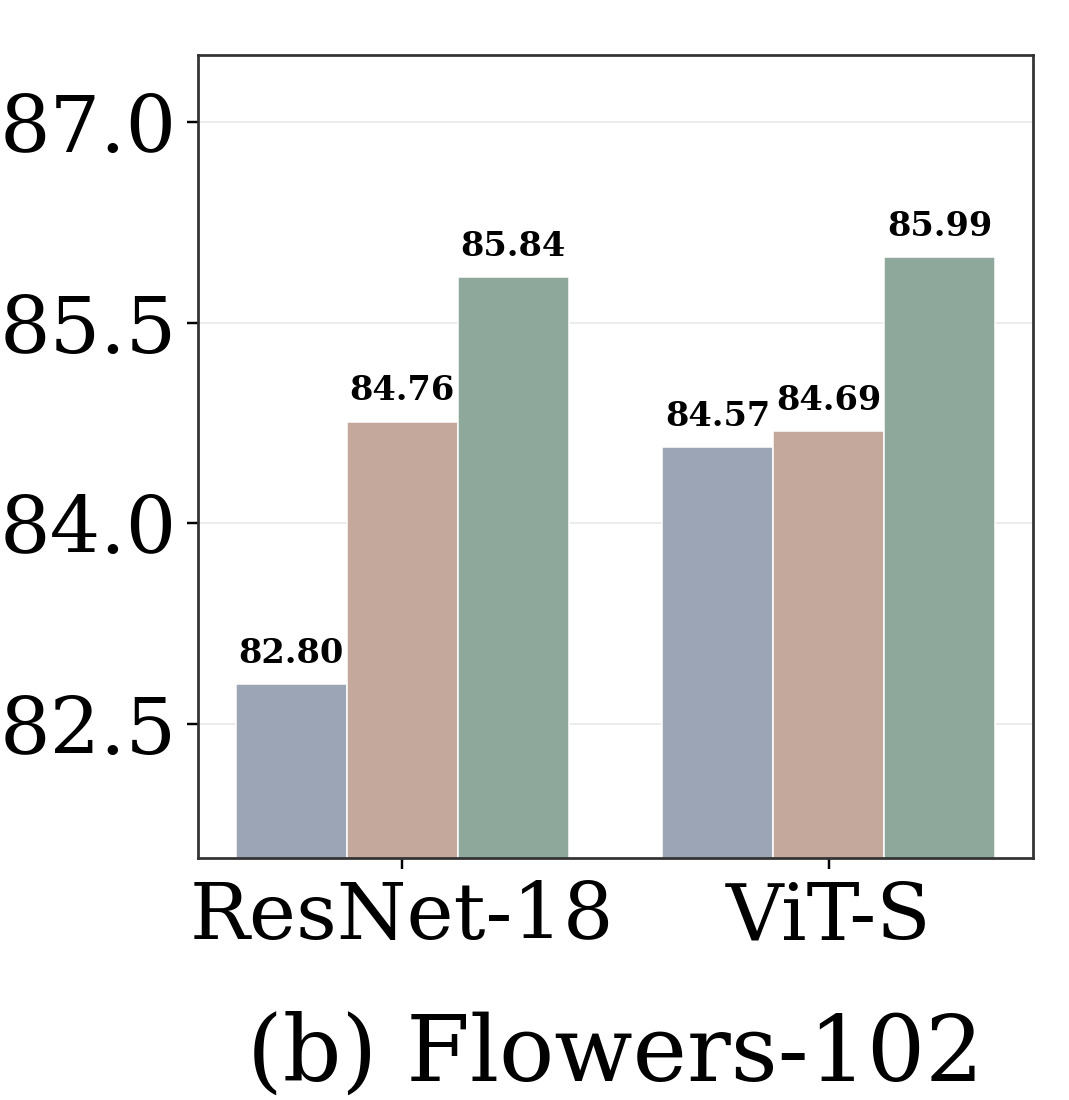}%
\includegraphics[width=0.2372\linewidth]{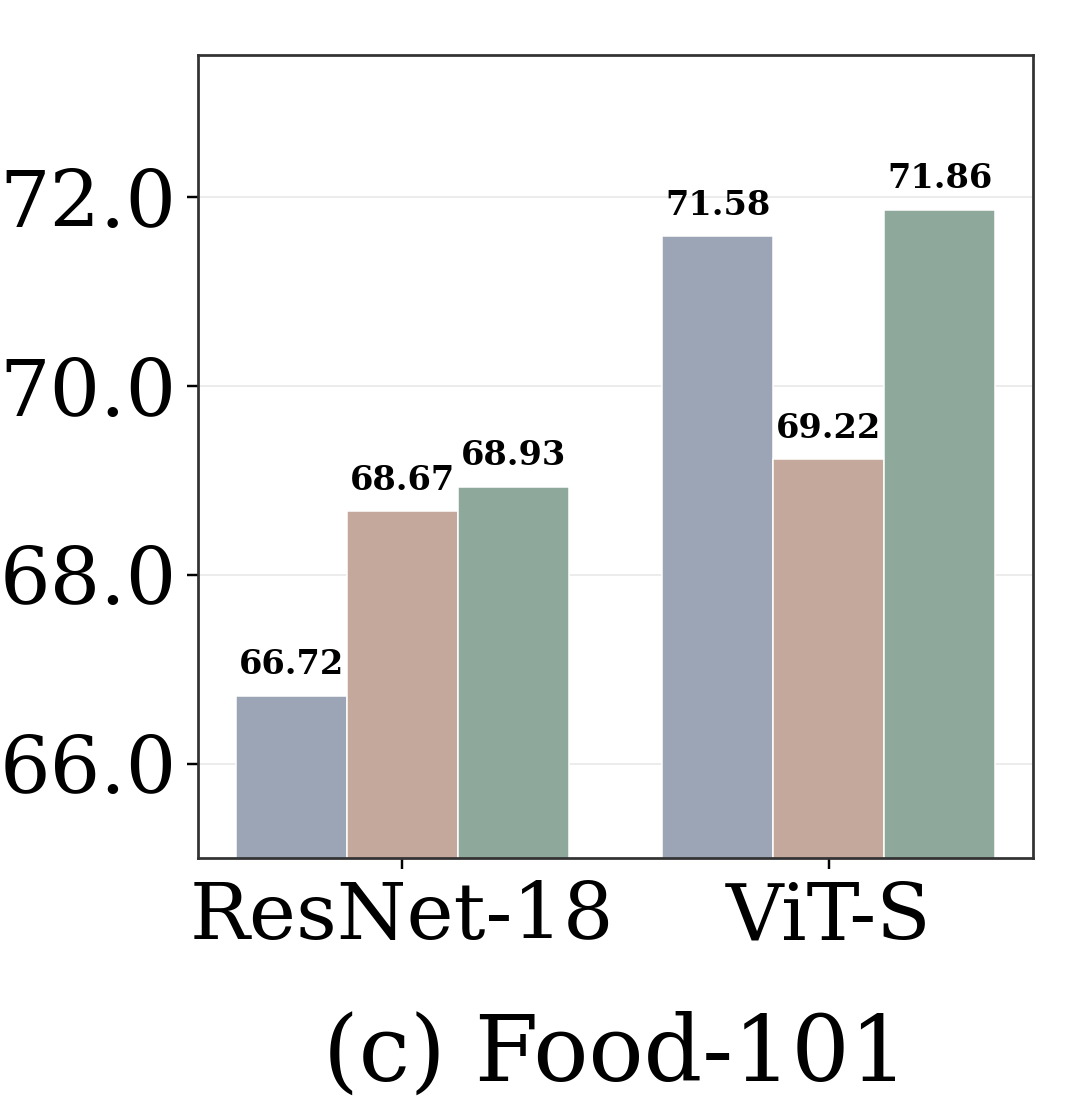}%
\includegraphics[width=0.2372\linewidth]{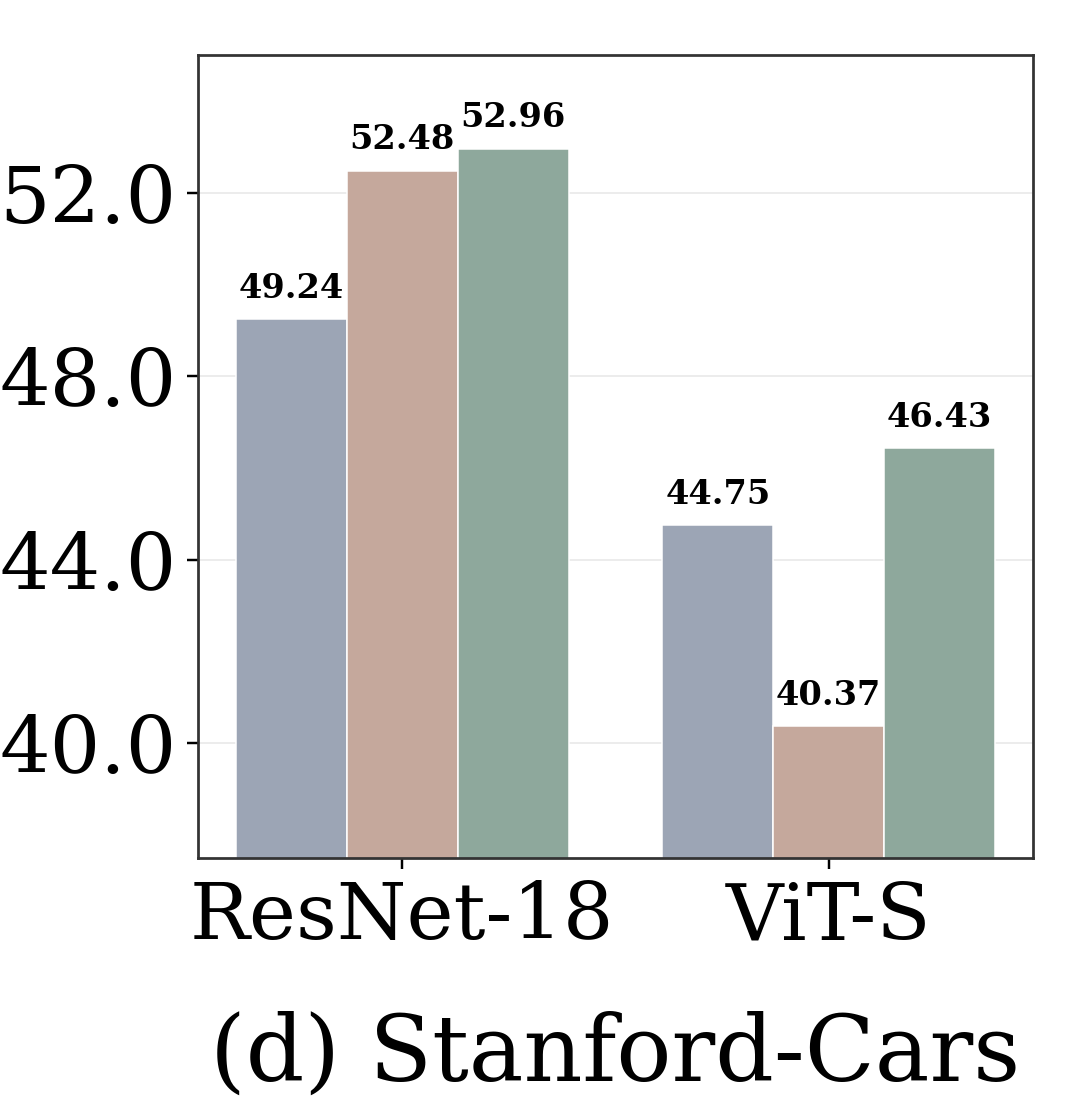}
    \caption{{Classification accuracy ($\uparrow$) across four downstream vision tasks.}
    Pretrained backbones are frozen; only the linear classifier is trained. Muon matches or surpasses Adam and SGD across all eight backbone--task combinations, indicating that its pretrained vision features transfer more effectively.}
    \label{fig:vision_transfer_acc}
\end{figure}

In this section, we evaluate the transferability of pretrained models by adapting them to downstream tasks. For vision experiments, we freeze the pretrained backbone, i.e., ResNet-18 or ViT-S, and train a linear classifier on the penultimate-layer hidden states for each downstream task, following standard protocols adopted in \citet{chen2020simple,he2022masked}. We consider a diverse set of downstream classification tasks: EuroSAT~\citep{helber2019eurosat}, Food-101~\citep{bossard2014food}, Oxford Flowers-102~\citep{nilsback2008automated}, and Stanford Cars~\citep{krause20133d}, covering satellite imagery, food categories, fine-grained flowers, and fine-grained vehicles. For language experiments, we perform supervised fine-tuning on the full parameters of the pretrained models using Stanford Alpaca~\citep{taori2023alpaca}, Databricks Dolly 15K~\citep{conover2023free}, and WizardLM Evol-Instruct~\citep{xu2023wizardlm}. To isolate the effect of the pretraining optimizer, we fine-tune both Muon- and Adam-pretrained models using the same Adam optimizer.

\vspace{5pt}
\noindent \textbf{Muon Learns More Transferable Features.} 
Figures~\ref{fig:vision_transfer_acc} and~\ref{fig:language_transfer_ppl} report downstream performance for vision and language tasks, respectively. In the vision experiments, Muon matches or surpasses Adam and SGD on all four benchmarks, namely EuroSAT, Flowers-102, Food-101, and Stanford Cars, across both backbones, ResNet-18 and ViT-S. Consistent with \citet{zhang2024transformers,kunstner2023noise,keskar2017improving}, our results show that SGD-pretrained CNNs transfer better than Adam-pretrained CNNs, whereas the opposite trend holds for transformers.  In the language experiments, Muon consistently yields lower perplexity than Adam across all three instruction-tuning datasets, Alpaca, Dolly, and WizardLM, and across different model sizes. We summarize this observation as follows. 

\begin{abox}
\label{obs:transfer}
The pretrained representations of Muon \kw{transfer more effectively} than those of Adam and SGD across vision tasks. 
Language models pre-trained by Muon outperform those by Adam in downstream tasks after  fine-tuning.
\end{abox}

\vspace{5pt}
\noindent
\textbf{Layer-wise feature spectrum analysis.} Observation~\ref{obs:transfer} confirms Muon's superior end-to-end transferability. Intuitively, this suggests that Muon learns richer and more diverse features. We next examine this hypothesis by analyzing the hidden states across layers. For the $\ell$-th layer, we stack the hidden states, i.e., the features, $z^{(\ell)}$, over all instances in the evaluation set to form a hidden-state matrix $Z^{(\ell)}$. We then study the richness of the features in $Z^{(\ell)}$ through spectral analysis. Let $\sigma(Z^{(\ell)})=(\sigma_1,\ldots,\sigma_{r_\ell})$ denote its singular values, where $\sigma_1\ge\sigma_2\ge\cdots\ge\sigma_{r_\ell}$ and $r_\ell$ is the number of nonzero singular values. We adopt two intuitive quantities to evaluate feature richness in $Z^{(\ell)}$: \kw{effective rank}~\citep{roy2007effective} and \kw{Top-$k$ energy fraction}~\citep{wang2025muon}. The effective rank is the exponential entropy of the distribution induced by the squared singular values, which is the effective rank applied to the Gram matrix: 
$$
\mathrm{eRank}(Z^{(\ell)})=\exp\bigg(-\sum_{i=1}^{r_\ell} q_i \log q_i\bigg), ~~~~\text{where}~~ q_i = {\sigma_i^2}/{\sum_{j=1}^{r_\ell}\sigma_j^2}.$$ It provides a continuous measure of how many singular directions are effectively used. The Top-$k$ energy fraction is the fraction of spectral energy captured by the largest $k$ singular values:
$$\mathrm{Top\text{-}k\,E}(Z^{(\ell)})=\sum_{i=1}^{k}\sigma_i^2/\sum_{j=1}^{r_\ell}\sigma_j^2.$$
 It measures how concentrated the representation energy is in the leading principal directions. Higher effective ranks and lower Top-$k$ energy indicate a more diverse and more uniformly distributed representation.
\begin{figure}[t]
    \centering
    \includegraphics[height=3.7cm]{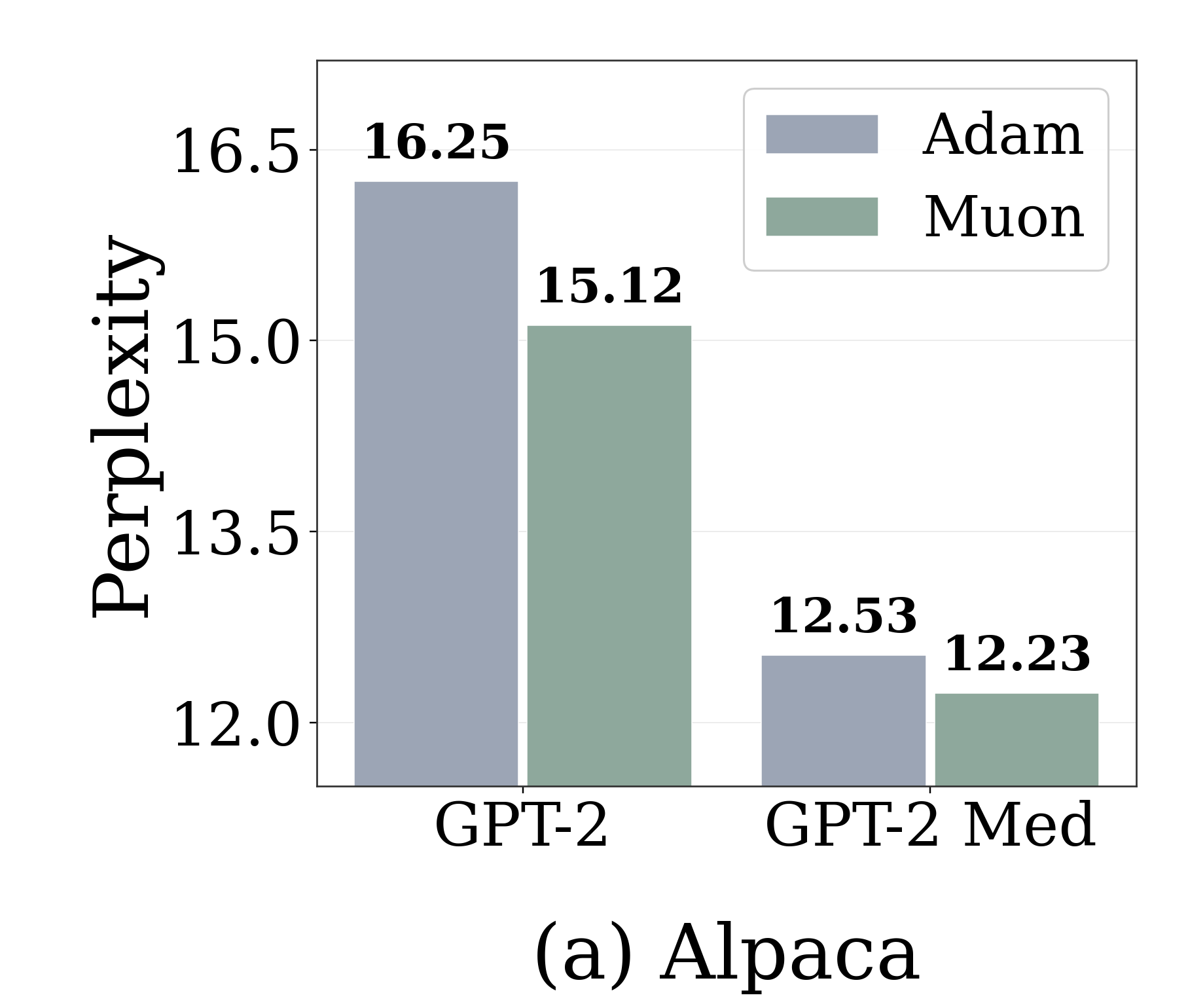}
    \hspace{0.02\linewidth}
\includegraphics[height=3.7cm]{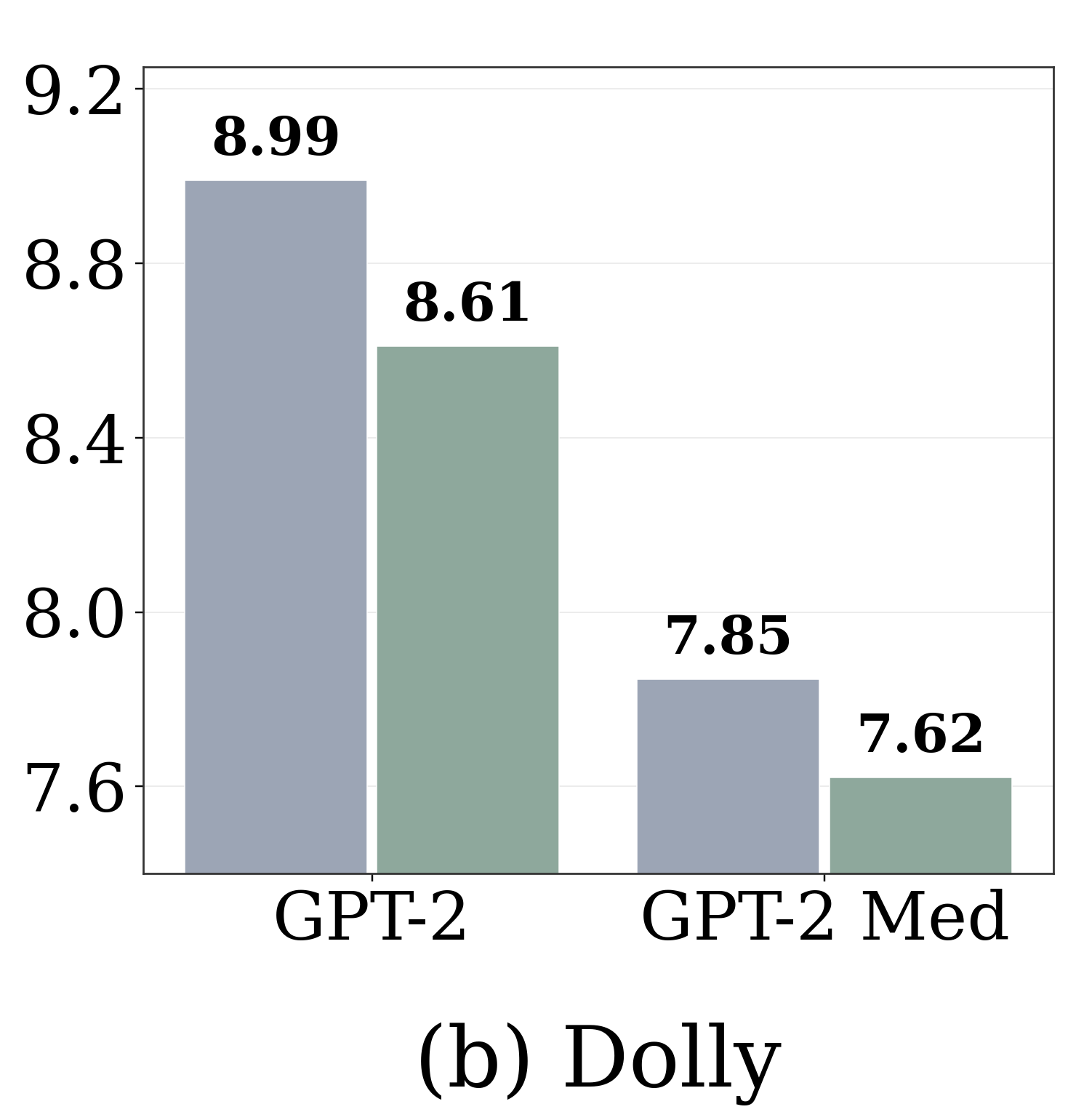}
\hspace{0.02\linewidth}
\includegraphics[height=3.7cm]{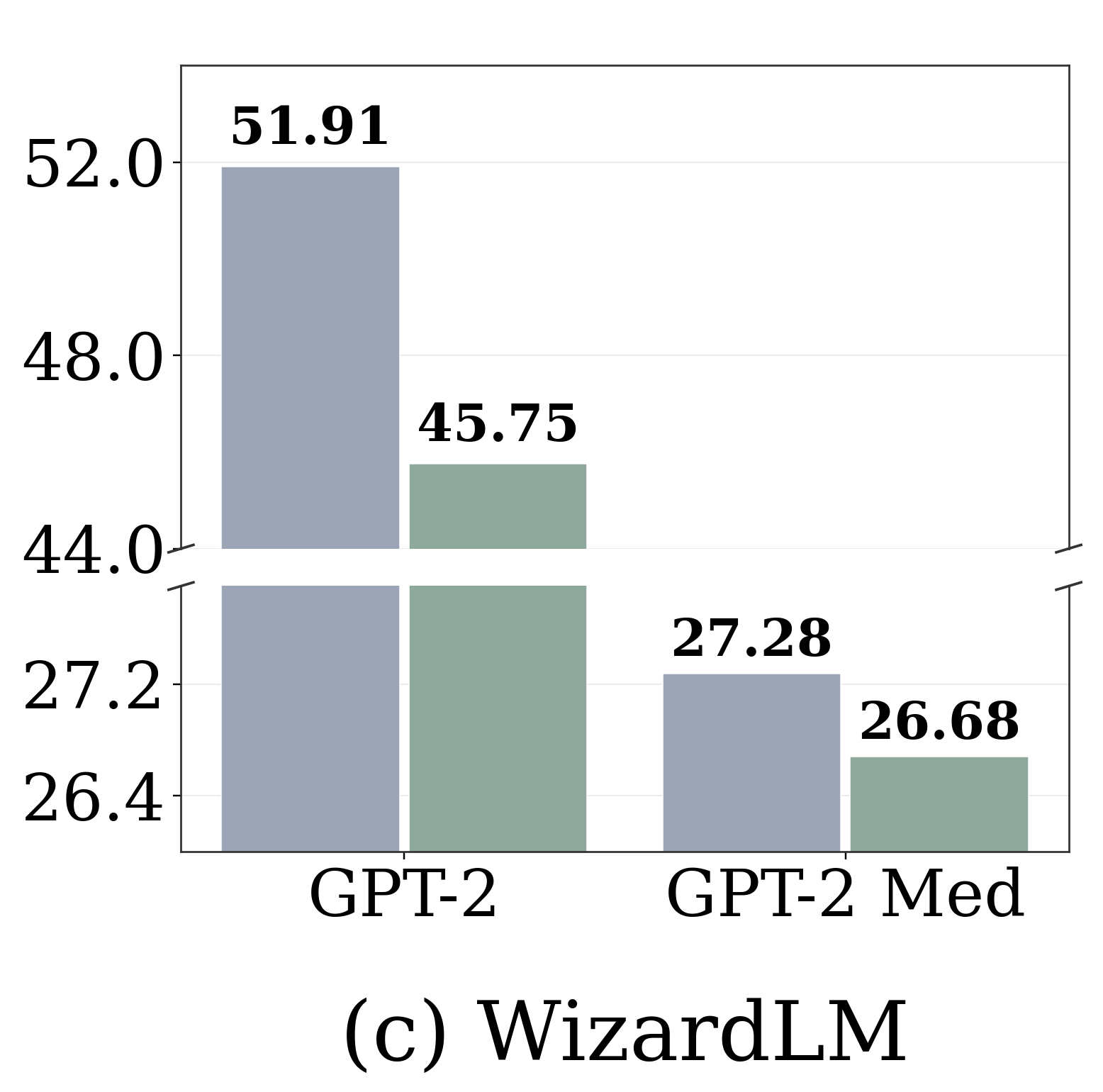}
    \caption{Instruction-tuning perplexity ($\downarrow$) on three downstream datasets. Each nanoGPT checkpoint is fully fine-tuned on each downstream dataset. Muon-pretrained models achieve lower perplexity than Adam-pretrained models across all benchmarks, indicating stronger transferability of the pretrained language representations.}
    \label{fig:language_transfer_ppl}
\end{figure}

As in Section~\ref{sec:robustness}, we focus on ViT-S and GPT-2, and report their layer-wise spectral metrics in Figure~\ref{fig:layerwise_erank}. In both architectures, Muon achieves a higher effective rank than Adam in the vast majority of layers while simultaneously exhibiting lower Top-10 energy, indicating that its hidden features are less dominated by a small number of principal directions. Additional results for GPT-2 Medium are reported in Appendix~\ref{app:24d_results}. For ViT-S, SGD exhibits a higher effective rank than Muon in some intermediate layers, namely layers 7--10, but this advantage does not persist into deeper layers, where Muon again achieves a more diverse spectrum.

\begin{wraptable}{r}{0.45\textwidth}
\vspace{-0.4em}
\centering
\footnotesize
\setlength{\tabcolsep}{3pt}
\renewcommand{\arraystretch}{0.92}
\caption{Depth-averaged spectral metrics across 12 layers. Bold: best.}
\vspace{-0.3em}
\label{tab:spectral_mean}
\resizebox{\linewidth}{!}{%
\begin{tabular}{lcccc}
\toprule
& \multicolumn{2}{c}{GPT-2} & \multicolumn{2}{c}{ViT-S} \\
\cmidrule(lr){2-3} \cmidrule(lr){4-5}
& eRank $\uparrow$ & Top-10E $\downarrow$ 
& eRank $\uparrow$ & Top-10E $\downarrow$ \\
\midrule
Adam & 11.12 & 0.863 & 1.86 & 0.964 \\
SGD  & ---   & ---   & 3.39 & 0.920 \\
Muon & \textbf{16.00} & \textbf{0.827} & \textbf{6.88} & \textbf{0.890} \\
\bottomrule
\end{tabular}%
}
\end{wraptable}
Averaging the spectral metrics across layers provides a clearer global picture of feature richness (Table~\ref{tab:spectral_mean}). In both architectures and under both metrics, Muon attains the most spectrally diverse representations. It achieves the highest mean effective rank and the lowest mean Top-10 energy. The gap is especially pronounced for ViT-S, where Muon's mean effective rank is more than three times that of Adam and roughly twice that of SGD. This consistency across vision and language models suggests that Muon's spectral advantage is not incidental to a particular architecture, but reflects a broader property of the features it learns. This leads to the following conclusion.

\begin{abox}
\label{obs:erank}
Muon learns more spectrally diverse representations than Adam and SGD across transformer layers, with higher \kw{effective rank} and lower \kw{Top-10 energy}.
\end{abox}

\begin{figure}[t]
\centering
\includegraphics[width=0.245\linewidth]{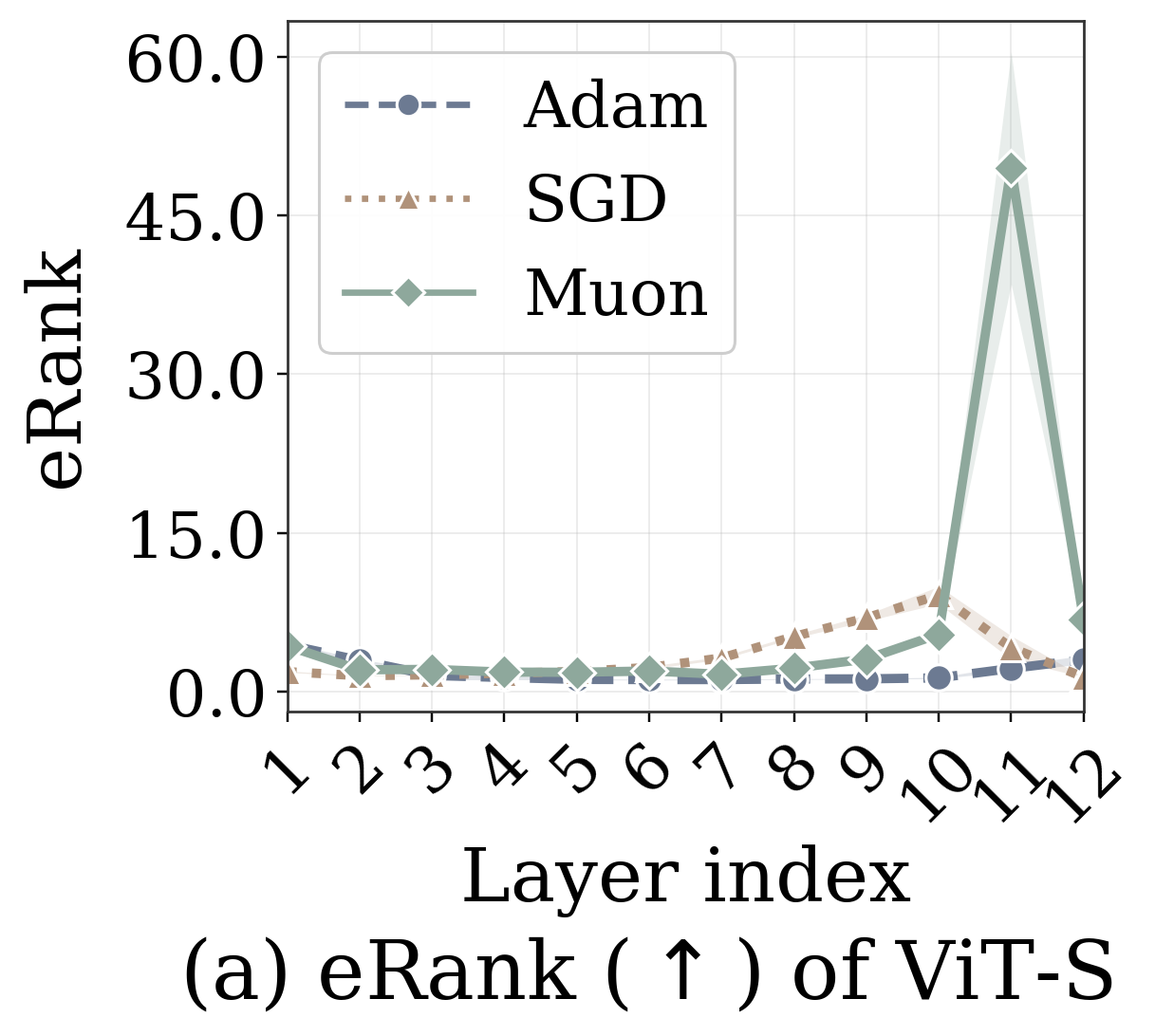}\hspace{-0.005\linewidth}%
\includegraphics[width=0.245\linewidth]{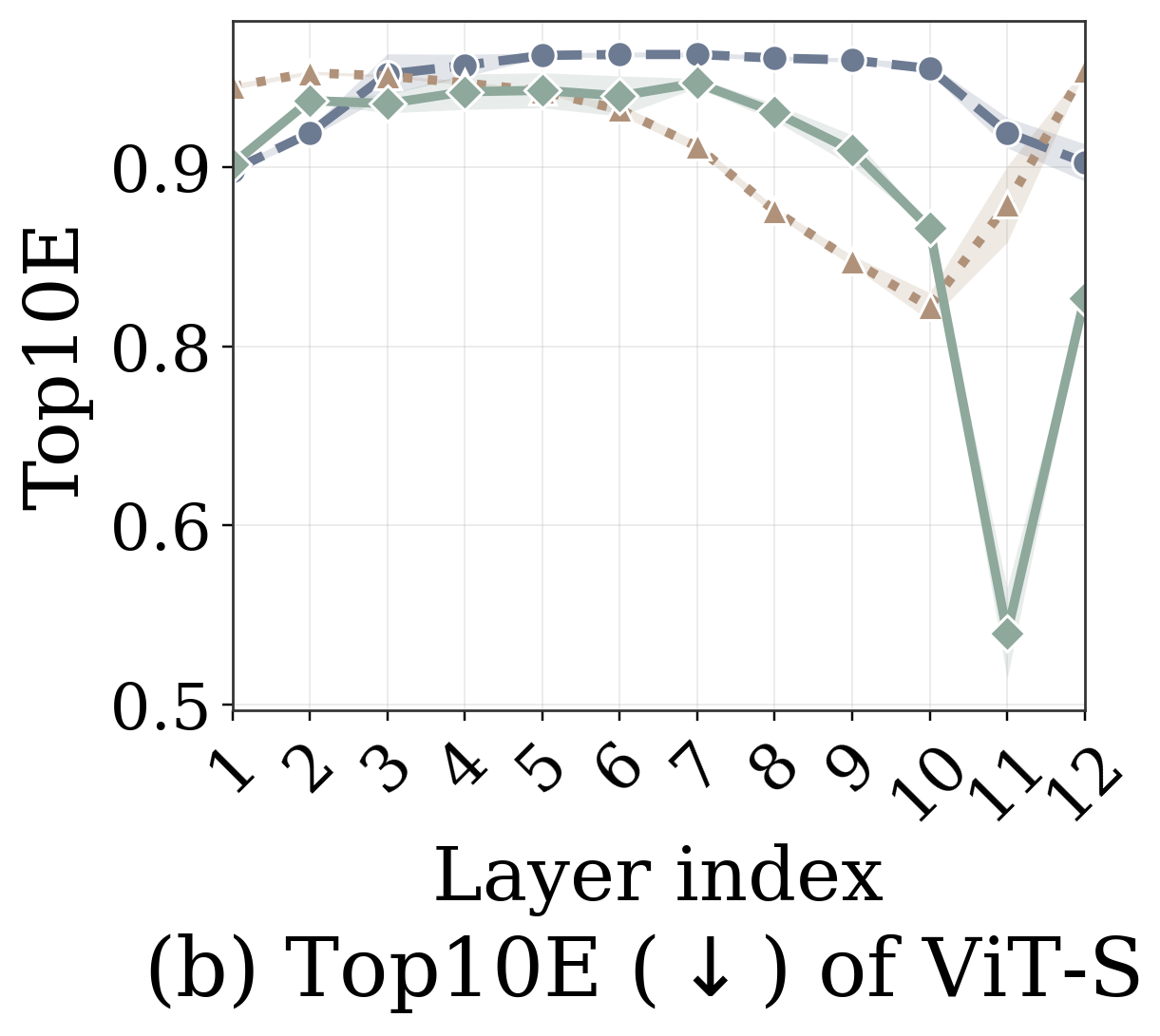}\hspace{-0.005\linewidth}%
\includegraphics[width=0.245\linewidth]{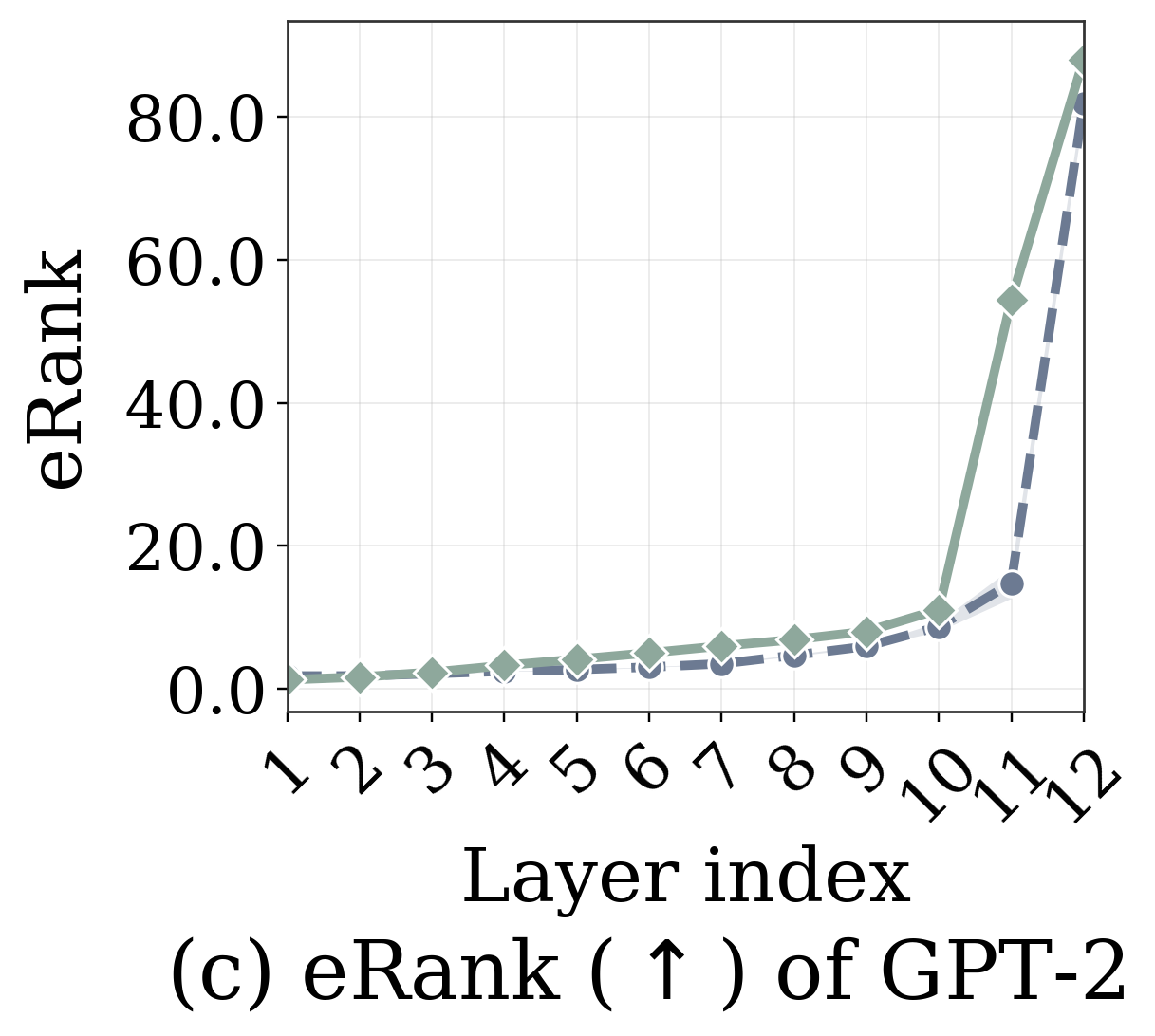}\hspace{-0.005\linewidth}%
\includegraphics[width=0.245\linewidth]{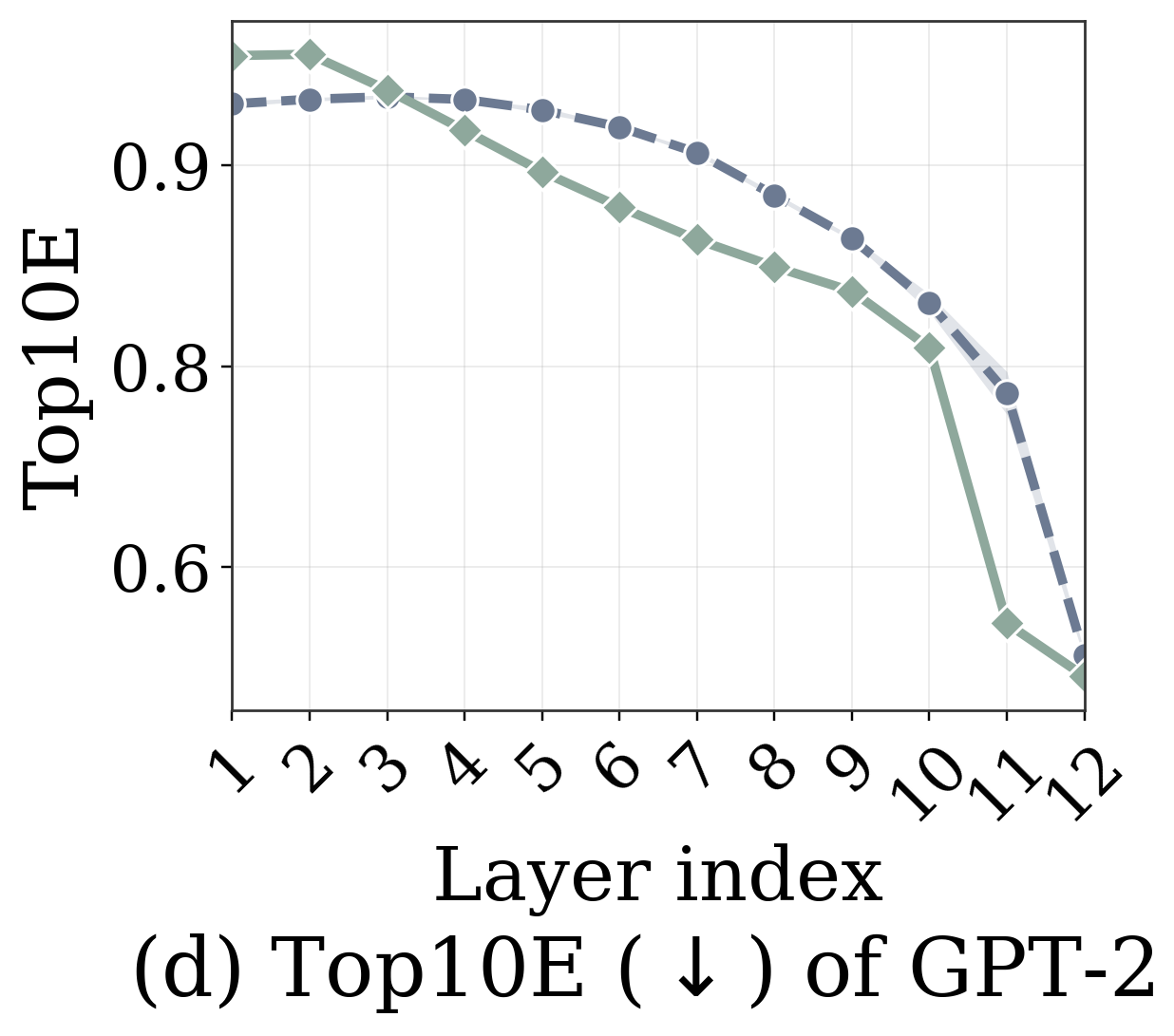}
\caption{
Layer-wise feature spectrum across transformer depth. Panels (a) and (b) report results for ViT-S, while panels (c) and (d) report results for GPT-2. Panels (a) and (c) show the effective rank of hidden representations, and panels (b) and (d) show the top-10 energy fraction. Shaded bands denote $\pm$ one standard deviation over three random seeds. Overall, Muon generally yields higher effective rank and lower Top-10 energy across layers, indicating more spectrally diverse hidden representations.}
\label{fig:layerwise_erank}
\end{figure}

%% file: main_paper/theory.tex
\section{A Case Study of One-Layer Models}
\label{sec:case_study}

Section~\ref{sec:experiments} empirically shows that the features learned by Muon are more robust and transferable than those learned by Adam and SGD, and this advantage is accompanied by larger margins and higher effective rank across layers of the network. In this section, we theoretically establish these properties in a stylized yet representative learning setting.

\subsection{A Classification Model with Multi-Component Features}
\label{subsec:model}

{\noindent \bf Data model.} We focus on a classification problem with class set $\calC=[C]$, where each class $c\in\calC$ is associated with a unique feature $x_c$. To capture the rich and complex features that arise in layer-wise representations of neural networks, we consider a fine-grained block decomposition across classes, together with a multi-component feature representation. This reflects the hierarchical structure common in real-world classification: in ImageNet, for instance, a block corresponds to a coarse superclass (e.g., dog), while the fine-grained classes within it are specific instances (e.g., golden retriever or beagle). Specifically, we decompose the class set as $\calC=[m]\times[n]$ with size $C=mn$, where $m$ denotes the number of blocks and $n$ denotes the number of fine-grained classes within each block. Each class can then be represented as $c=(g,r)$, where $g\in[m]$ denotes the block index and $r\in[n]$ denotes the class index within the block. We assume that the feature of each class $c=(g,r)$ has the following form.

 \begin{assumption}[Multi-component feature representation]
 \label{feature_assumption}
 For each class $c=(g,r)$, the input feature $x_c$  has two components, i.e., \(x_c=[a\cdot e_g,\; b \cdot e_{g,r}]^\top\in\mathbb R^{m+C},\, 0<a\le b,\) where $e_g \in \mathbb{R}^m$ and $e_{g,r}\in\mathbb{R}^C$ are one-hot vectors, and the coefficients $a,b$ control the strengths of the shared block-level and class-specific components, respectively.
 \end{assumption}

\input{figures/tikz/feature_structure.tex}
This assumption states that the features of each class consist of a block-level component and a class-specific component. The constraint $0<a\le b$ ensures that the class-specific component is at least as strong as the shared block-level component. The one-hot structure of $e_g$ and $e_{g,r}$ is imposed only for ease of computation and presentation. Our analysis of Muon still holds when $e_g$ and $e_{g,r}$ are replaced by orthonormal vectors. Under this assumption, the feature matrix that collects the features of all classes is
$F = [a\cdot R^\top,\, b\cdot I_C]^\top\in \mathbb{R}^{(m+C)\times C}$,
where $R\in\{0,1\}^{m\times C}$ is the coarse-membership matrix whose $c$-th column equals $e_g$ for $c=(g,r)$, and $I_C$ is the identity matrix. Figure~\ref{fig:feature_structure_tikz} illustrates an example of a feature matrix with $m=3$, $n=2$, and $C=6$.

\vspace{5pt}
{\noindent \bf Empirical support for Assumption \ref{feature_assumption}.} To examine this assumption in pretrained networks, we compute similarities among input-embedding vectors that represent different classes or token semantics. For ViT-S, each semantic group corresponds to an image class, and its representative vector is obtained by averaging the input embeddings of all images in that class. For GPT-2, we use the $1000$ most frequent tokens as semantic units, with each unit represented by its token embedding. For both ViT-S and GPT-2, we mean-center these representative vectors and measure their pairwise correlations using the absolute cosine similarity. If Assumption~\ref{feature_assumption} holds, the resulting similarity matrix should exhibit a block structure: classes that share the same block-level components $e_g$ should have high similarity, whereas classes from different blocks should have low similarity. Figure~\ref{fig:hierarchical_feature_empirical} plots the resulting similarity matrices for ViT-S and GPT-2. To make the block structure more visible, we apply hierarchical clustering to reorder the classes. More experimental details are deferred to Appendix~\ref{app:reorder}. The results show clear block patterns in the similarity maps of both ViT-S and GPT-2, providing empirical support for Assumption~\ref{feature_assumption}.


\vspace{5pt}
{\noindent \bf Classification model.} 
We adopt a linear classifier parameterized by $W \in \mathbb{R}^{C \times (m+C)}$ to classify the features $x_c$ for $c\in[C]$. Given $x_c$ as input, the classifier predicts its class according to the distribution $\soft(\Pi W x_c)$, where $\soft(\cdot)$ denotes the softmax operator and $\Pi=I_C-(1/C)\mathbf{1}\mathbf{1}^\top$ centers the logits so that their sum is zero. We apply $\Pi$ for ease of calculation, which does not affect the predicted distribution because softmax is invariant to constant shifts of the logits. When the class prior is uniform, the cross-entropy loss is
\begin{align}\label{eq:cross-entropy-loss}
\mathcal L(W)=-1/C\cdot\sum_{c\in\mathcal C}\log\bigl(\soft(\Pi W F)_{c,c}\bigr), \end{align}
where $\soft(\cdot)$ is applied column-wise.

\begin{figure}[t]
    \centering
    \begin{subfigure}[b]{0.32\textwidth}
        \centering
        \includegraphics[height=0.78\textwidth]{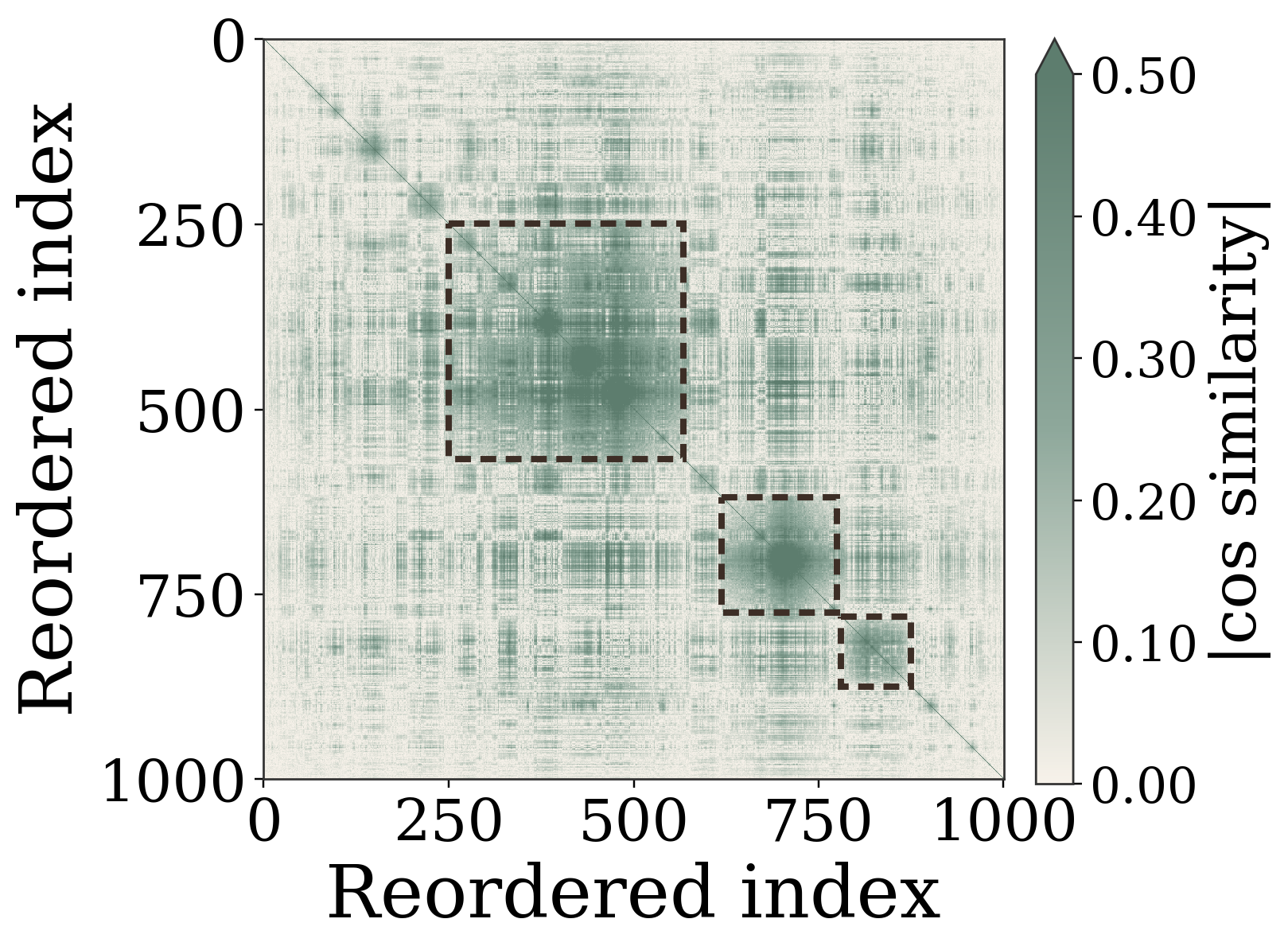}
        \caption{ViT-S}
        \label{fig:hierarchical_feature_empirical:vit}
    \end{subfigure}
    \hspace{3em}
    \begin{subfigure}[b]{0.32\textwidth}
        \centering
        \includegraphics[height=0.78\textwidth]{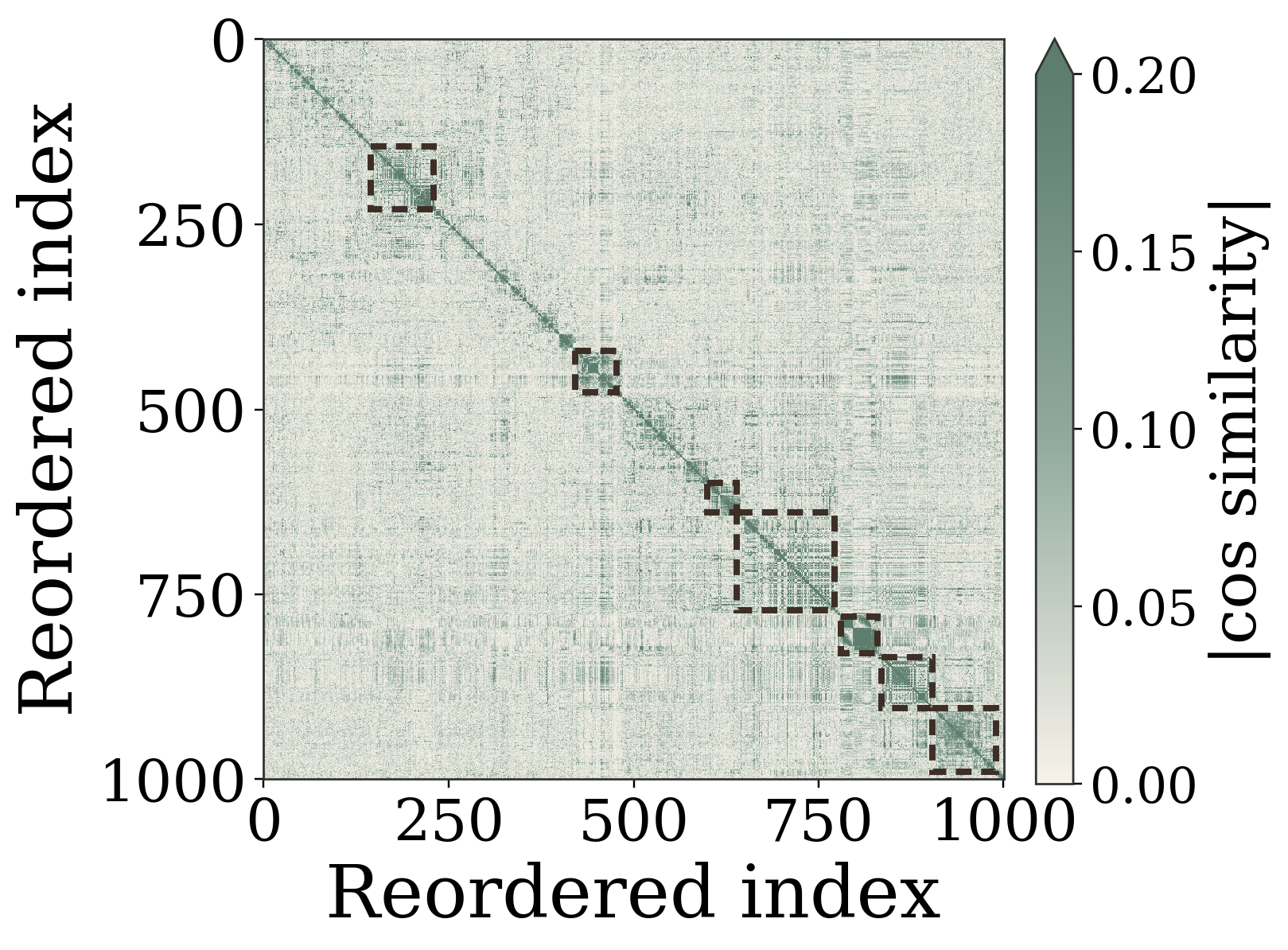}
        \caption{GPT-2}
        \label{fig:hierarchical_feature_empirical:gpt}
    \end{subfigure}
    \caption{%
        Pairwise absolute cosine similarity of centered input-embedding vectors, reordered by hierarchical clustering with optimal leaf ordering. The shared block structure across ViT-S and GPT-2 supports Assumption~\ref{feature_assumption}.
    }    \label{fig:hierarchical_feature_empirical}
\end{figure}

\vspace{5pt}
{\noindent \bf Training dynamics.}
To train the classifier, we initialize $W$ as the zero matrix and optimize it using three optimizers: GD, Adam, and Muon. To simplify the analysis and highlight the main intuition, we adopt infinitesimal step sizes and set momentum to zero, following \citet{zhang2024transformers,wang2025muon}. Concretely, the parameter flow dynamics induced by these optimizers are as follows.
\begin{itemize}[leftmargin=4em]
\item [(GD)] GD updates the parameters along the negative gradient flow:
$\dot W^{\mathsf{GD}}(t)=-\nabla_W\mathcal L(W^{\mathsf{GD}}(t)).$
\item [(Muon)] With momentum set to $0$, Muon updates the parameters using the spectrally normalized gradient:
$\dot W^{\mathsf{Muon}}(t)=-\spec(\nabla_W\mathcal L(W^{\mathsf{Muon}}(t)))$,
where $\spec(G)=UV^\top$ for $G=U\Sigma V^\top$, which normalizes all nonzero singular values $\Sigma$ of $G$ to $1$.

\item [(Adam)] With $\beta_1=\beta_2=0$, Adam reduces to a coordinate-wise normalized gradient flow:
$\dot W^{\mathsf{Adam}}(t)=-\sgn(\nabla_W\mathcal L(W^{\mathsf{Adam}}(t)))$,
where $\sgn(\cdot)$ denotes the entrywise sign operator.
\end{itemize}

To compare these optimizers fairly, we analyze the learned classifier at the first time its training loss reaches a given threshold $\epsilon$. For any optimizer $\mathsf{A}\in\{\mathsf{GD},\mathsf{Adam},\mathsf{Muon}\}$, define
$t_{\epsilon}^{\mathsf{A}}=\inf\{t\ge 0 \mid \mathcal L(W^{\mathsf{A}}(t))\leq\epsilon \}$. We denote the corresponding parameter by $W_{\epsilon}^{\mathsf{A}}=W^{\mathsf{A}}(t_{\epsilon}^{\mathsf{A}})$, which corresponds to the first parameter along the training dynamics that yields a loss value no more than $\epsilon$. 


\subsection{Theoretical results}
\label{subsec:theory}
To connect the model to the empirical observations in Section~\ref{sec:experiments}, we introduce two quantities: the \kw{classification margin}, which captures class separation, and the \kw{feature effective rank}, which captures the spectral diversity of representations. In the linear classifier, we define the representation of an input $x_c$ as the centered logits $\Pi W x_c$ induced by the learned weight $W$. Stacking these representations over all classes, we obtain the representation matrix $Z(W)=\Pi WF\in\mathbb{R}^{C\times C}$, whose $c$-th column is $\Pi W x_c$. This matrix serves as the natural analogue, in our analytical model, of the hidden-state representations probed in Section~\ref{sec:experiments}. The classification margin follows the layer-wise margin definition in  Eqn.~\eqref{eq:layerwise_margin}, with the dependence on the classifier $W$ made explicit. That is,  for each class $c\in\mathcal C$, we define the margin of it as
\begin{align*}
    \gamma(W;x_c,c)=(\Pi W x_c)_c-\max_{j\ne c}(\Pi W x_c)_j,
\end{align*}
where $(\Pi W x_c)_i$ denotes the $i$-th coordinate of $\Pi W x_c$. The average margin is defined as
$\Gamma(W)=1/C\cdot\sum_{c\in\mathcal C}\gamma(W;x_c,c)$.
For feature effective rank, we use $\mathrm{eRank}(Z(W))$, following the definition in Section~\ref{sec:transfer}. We first compare the margins induced by different optimizers. 

\begin{theorem}[Margin advantage of Muon]
\label{thm:main-margin}
Under Assumption~\ref{feature_assumption}, as $m,n\to\infty$ with $m=\Theta(n)$, we obtain the following results.

\begin{itemize}[leftmargin=2em, itemsep=2pt, topsep=2pt]
\item[1.] For every $\epsilon \in (0,\log C)$, at the first time each optimizer reaches loss at most $\epsilon$, Muon attains a strictly larger average margin than Adam:
\[
    \Gamma(W_\epsilon^{\mathsf{Muon}})>\Gamma(W_\epsilon^{\mathsf{Adam}}).
\]

\item[2.] Throughout most of the training, except when the target loss becomes extremely small, Muon attains a strictly larger margin than GD at the same loss threshold. Specifically, for
$\epsilon = \Omega\!\left(\exp(-c\sqrt n\log n)\right)$
with some constant $c>0$, we have 
\[
\Gamma(W_\epsilon^{\mathsf{Muon}})>\Gamma(W_\epsilon^{\mathsf{GD}}).
\]

\end{itemize}
\end{theorem}
The proof is provided in Appendix~\ref{app:proof-margin}. The exponentially small lower bound on the target loss rules out an over-trained regime in which optimization is run for excessively long times~\citep{soudry2018implicit}. Our focus is instead on standard pretraining regimes, such as settings in Section~\ref{sec:experiments} where the Chinchilla ratio is below $10$. Our results establish that, \kw{at any matched loss threshold, Muon attains a strictly larger margin than Adam}. The same advantage holds over GD outside an exponentially small terminal-loss window. These results provide theoretical support for Observation~\ref{obs:margin} in Section~\ref{sec:robustness}, where Muon has a larger margin than Adam and SGD across most layers. The main intuition of our proof is that Muon's spectrally normalized gradient extracts more balanced information from orthogonal feature directions than Adam and GD. By contrast, Adam and GD tend to over-rely on either $e_g$ or $e_{g,r}$. Thus, Muon's class separation is supported by richer multi-component features, whereas that of Adam and GD is mainly driven by a single component of the representation. Quantitatively, this leads Muon to attain the largest margin.



\begin{theorem}[Effective-rank advantage of Muon]
\label{thm:erank-merged} 
Under Assumption~\ref{feature_assumption}, as $m,n\to\infty$ with $m=\Theta(n)$, we obtain the following results.

\begin{itemize}[leftmargin=2em, itemsep=2pt, topsep=2pt]
\item[1.]
For every $\epsilon\in(0,\log C)$, at the first time each optimizer reaches the loss of at most $\epsilon$, the representations learned by Muon attain a strictly higher effective rank than those learned by Adam at the same loss threshold; that is,
\[
    \mathrm{eRank}( Z(W_\epsilon^{\mathsf{Muon}}))
>
\mathrm{eRank}( Z(W_\epsilon^{\mathsf{Adam}})).
\]

\item[2.]
Throughout most of the training, except when the target loss becomes extremely small, the representations learned by Muon have a strictly higher effective rank than those learned by GD at the same loss threshold. Specifically, for 
$\epsilon = \Omega\!\left(\exp(-c\sqrt n\log n)\right)$
with some constant $c>0$, we have 
\[
    \mathrm{eRank}( Z(W_{\epsilon}^{\mathsf{Muon}}))
>
\mathrm{eRank}( Z(W_{\epsilon}^{\mathsf{GD}})).
\]
\end{itemize}
\end{theorem}

Similar to Theorem~\ref{thm:main-margin}, the comparison between Muon and GD focuses on target losses above an exponentially small lower bound. Our results show that, \kw{at any matched loss threshold, the representations learned by Muon attain strictly higher effective rank than those learned by Adam}. The same advantage holds over GD outside an exponentially small terminal-loss window. These results provide theoretical support for Observation~\ref{obs:erank} in Section~\ref{sec:transfer}, where Muon achieves a higher average effective rank than Adam and SGD across layers. The main intuition is that Muon's spectral normalization equalizes the gradient contribution across orthogonal directions, thereby reducing spectral imbalance and improving feature diversity as measured by effective rank. In contrast, GD follows the raw gradient and preserves its imbalanced amplitudes, while Adam applies coordinate-wise normalization that does not fully balance the representation in a spectral sense. Consequently, under matched-loss comparison, Muon attains the highest effective rank among the considered optimizers. We next provide a proof sketch that mathematically formalizes these intuitions.
\input{main_paper/proof_sketch}


%% file: figures/tikz/feature_structure.tex
\begin{figure}[t]
\centering
\begin{tikzpicture}[
    >=stealth,
    cell/.style={draw, minimum width=0.8cm, minimum height=0.8cm, font=\footnotesize},
    brace/.style={decorate, decoration={brace, amplitude=5pt, raise=2pt}},
    braceleft/.style={decorate, decoration={brace, amplitude=5pt, raise=2pt, mirror}},
]


\node[font=\small\bfseries] at (4.0, 5.2) {Feature matrix $F^\top \in \mathbb{R}^{6 \times 9}$};

\node[cell, fill=morandibluedeep!35] at (0.0, 4.2) {$a$};
\node[cell, fill=gray!8] at (0.8, 4.2) {$0$};
\node[cell, fill=gray!8] at (1.6, 4.2) {$0$};
\node[cell, fill=keycolor!40] at (4.1, 4.2) {$b$};
\node[cell, fill=gray!8] at (4.9, 4.2) {$0$};
\node[cell, fill=gray!8] at (5.7, 4.2) {$0$};
\node[cell, fill=gray!8] at (6.5, 4.2) {$0$};
\node[cell, fill=gray!8] at (7.3, 4.2) {$0$};
\node[cell, fill=gray!8] at (8.1, 4.2) {$0$};

\node[cell, fill=morandibluedeep!35] at (0.0, 3.4) {$a$};
\node[cell, fill=gray!8] at (0.8, 3.4) {$0$};
\node[cell, fill=gray!8] at (1.6, 3.4) {$0$};
\node[cell, fill=gray!8] at (4.1, 3.4) {$0$};
\node[cell, fill=keycolor!40] at (4.9, 3.4) {$b$};
\node[cell, fill=gray!8] at (5.7, 3.4) {$0$};
\node[cell, fill=gray!8] at (6.5, 3.4) {$0$};
\node[cell, fill=gray!8] at (7.3, 3.4) {$0$};
\node[cell, fill=gray!8] at (8.1, 3.4) {$0$};

\node[cell, fill=gray!8] at (0.0, 2.6) {$0$};
\node[cell, fill=morandibluedeep!35] at (0.8, 2.6) {$a$};
\node[cell, fill=gray!8] at (1.6, 2.6) {$0$};
\node[cell, fill=gray!8] at (4.1, 2.6) {$0$};
\node[cell, fill=gray!8] at (4.9, 2.6) {$0$};
\node[cell, fill=keycolor!40] at (5.7, 2.6) {$b$};
\node[cell, fill=gray!8] at (6.5, 2.6) {$0$};
\node[cell, fill=gray!8] at (7.3, 2.6) {$0$};
\node[cell, fill=gray!8] at (8.1, 2.6) {$0$};

\node[cell, fill=gray!8] at (0.0, 1.8) {$0$};
\node[cell, fill=morandibluedeep!35] at (0.8, 1.8) {$a$};
\node[cell, fill=gray!8] at (1.6, 1.8) {$0$};
\node[cell, fill=gray!8] at (4.1, 1.8) {$0$};
\node[cell, fill=gray!8] at (4.9, 1.8) {$0$};
\node[cell, fill=gray!8] at (5.7, 1.8) {$0$};
\node[cell, fill=keycolor!40] at (6.5, 1.8) {$b$};
\node[cell, fill=gray!8] at (7.3, 1.8) {$0$};
\node[cell, fill=gray!8] at (8.1, 1.8) {$0$};

\node[cell, fill=gray!8] at (0.0, 1.0) {$0$};
\node[cell, fill=gray!8] at (0.8, 1.0) {$0$};
\node[cell, fill=morandibluedeep!35] at (1.6, 1.0) {$a$};
\node[cell, fill=gray!8] at (4.1, 1.0) {$0$};
\node[cell, fill=gray!8] at (4.9, 1.0) {$0$};
\node[cell, fill=gray!8] at (5.7, 1.0) {$0$};
\node[cell, fill=gray!8] at (6.5, 1.0) {$0$};
\node[cell, fill=keycolor!40] at (7.3, 1.0) {$b$};
\node[cell, fill=gray!8] at (8.1, 1.0) {$0$};

\node[cell, fill=gray!8] at (0.0, 0.2) {$0$};
\node[cell, fill=gray!8] at (0.8, 0.2) {$0$};
\node[cell, fill=morandibluedeep!35] at (1.6, 0.2) {$a$};
\node[cell, fill=gray!8] at (4.1, 0.2) {$0$};
\node[cell, fill=gray!8] at (4.9, 0.2) {$0$};
\node[cell, fill=gray!8] at (5.7, 0.2) {$0$};
\node[cell, fill=gray!8] at (6.5, 0.2) {$0$};
\node[cell, fill=gray!8] at (7.3, 0.2) {$0$};
\node[cell, fill=keycolor!40] at (8.1, 0.2) {$b$};

\draw[thick, dashed, gray!60] (2.4, 4.6) -- (2.4, -0.2);

\draw[braceleft] (-0.7, 4.5) -- (-0.7, 3.1) node[midway, left=6pt, font=\footnotesize] {$a\, e_1$};
\draw[braceleft] (-0.7, 2.9) -- (-0.7, 1.5) node[midway, left=6pt, font=\footnotesize] {$a\, e_2$};
\draw[braceleft] (-0.7, 1.3) -- (-0.7, -0.1) node[midway, left=6pt, font=\footnotesize] {$a\, e_3$};

\node[font=\scriptsize, left] at (3.6, 4.2) {$b\, e_{(1,1)}$};
\node[font=\scriptsize, left] at (3.6, 3.4) {$b\, e_{(1,2)}$};
\node[font=\scriptsize, left] at (3.6, 2.6) {$b\, e_{(2,1)}$};
\node[font=\scriptsize, left] at (3.6, 1.8) {$b\, e_{(2,2)}$};
\node[font=\scriptsize, left] at (3.6, 1.0) {$b\, e_{(3,1)}$};
\node[font=\scriptsize, left] at (3.6, 0.2) {$b\, e_{(3,2)}$};

\node[font=\scriptsize, right] at (8.6, 4.2) {$(1,1)$};
\node[font=\scriptsize, right] at (8.6, 3.4) {$(1,2)$};
\node[font=\scriptsize, right] at (8.6, 2.6) {$(2,1)$};
\node[font=\scriptsize, right] at (8.6, 1.8) {$(2,2)$};
\node[font=\scriptsize, right] at (8.6, 1.0) {$(3,1)$};
\node[font=\scriptsize, right] at (8.6, 0.2) {$(3,2)$};

\draw[braceleft] (-0.4, -0.5) -- (2.0, -0.5) node[midway, below=6pt, font=\small] {$a \cdot R^T$ \footnotesize($m{=}3$)};
\draw[braceleft] (3.7, -0.5) -- (8.5, -0.5) node[midway, below=6pt, font=\small] {$b \cdot I_C$ \footnotesize($C{=}6$)};

\end{tikzpicture}

\vspace{0.8em}  

\caption{\textbf{Multi-component feature structure} (Assumption~\ref{feature_assumption}), shown as the transposed feature matrix $F^\top \in \mathbb{R}^{C \times (m+C)}$ for the case $m{=}3$, $n{=}2$, $C{=}m\cdot n{=}6$. Each row corresponds to one class $c=(g,r)$, where $g$ indexes its block and $r$ its position within the block. The feature of every class decomposes into two parts. The \emph{shared component} $a \cdot R^T$ (\textcolor{morandibluedeep}{blue}, first $m$ columns) is determined solely by the block index $g$: all classes in the same block activate the same coordinate $e_g$ with magnitude $a$, encoding their semantic relatedness. The \emph{unique component} $b \cdot I_C$ (\textcolor{keycolor}{red}, remaining $C$ columns) assigns each class its own coordinate $e_{(g,r)}$ with magnitude $b$, ensuring all classes remain individually distinguishable. 
}
\label{fig:feature_structure_tikz}
\end{figure}

%% file: main_paper/proof_sketch.tex
\subsection{Proof Sketch}
\label{subsec:proof_sketch}

We describe the main mechanism behind Theorems~\ref{thm:main-margin} and~\ref{thm:erank-merged}. The full proof is deferred to Appendix~\ref{sec:proofs}. The key is to analyze the representation $Z(W)=\Pi WF$ induced by the optimized weight matrix. Under Assumption~\ref{feature_assumption}, the coarse-to-fine block structure of the data is preserved by all three optimizer flows. Specifically, their representations remain in a two-dimensional subspace:
\begin{align}
    Z^\mathsf{A}(t)
    =
    u^\mathsf{A}(t)P_f+v^\mathsf{A}(t)P_c,
    \qquad
    \mathsf{A}\in\{\mathsf{GD},\mathsf{Adam},\mathsf{Muon}\}.
    \label{eq:reduce}
\end{align}
Here, the within-block projection matrix $P_f\in\bbR^{C\times C}$ separates classes within the same block, while the cross-block projection matrix $P_c\in\bbR^{C\times C}$ separates different blocks. Their explicit expressions are deferred to Appendix~\ref{sec:proofs}. Eqn.~\eqref{eq:reduce} reduces each optimizer trajectory to the two coefficients $(u^\mathsf{A},v^\mathsf{A})$. The relevant imbalance between the two complementary directions is captured by the \textbf{representation imbalance ratio}
\begin{equation}
    \rho^\mathsf{A}(t)=\frac{v^\mathsf{A}(t)}{u^\mathsf{A}(t)}.
\end{equation}
This ratio serves as the key parameter for comparing the matched-loss margin and effective rank of the resulting representations across optimizers.

\vspace{3pt}
\noindent\textbf{Step 1: Compare the optimizer-induced representation imbalance ratios.}
To analyze the ratios $\rho^\mathsf{A}$ along the training process, we derive the representation dynamics induced by each optimizer along the two complementary directions $P_f$ and $P_c$. The key distinction across optimizers lies in how they normalize the gradient components along these two directions. Muon normalizes the gradient at the singular-value level, so the resulting imbalance is determined only by the feature-induced amplification $F$ in $Z(W)$. By Proposition~\ref{prop:muon-centered-logit}, the representation direction is fixed throughout training, with
\begin{equation}
\rho_{\mathsf{Muon}}=\frac{s}{b},
\qquad
s=\sqrt{a^2n+b^2}>b.
\end{equation}
Adam, in contrast, normalizes gradient coordinates rather than singular values. This entrywise normalization leads to a different fixed representation direction, and Proposition~\ref{prop:signgd-centered-logit} gives
\begin{equation}
\rho_{\mathsf{Adam}}=1+\frac{an}{b}>\rho_{\mathsf{Muon}}.
\end{equation}
Thus, for Muon and Adam, the representation imbalance ratios remain fixed throughout training due to their respective normalization schemes. GD behaves differently: it follows the raw gradient spectrum and therefore has a time-varying ratio $\rho_{\mathsf{GD}}(t)$. The dynamics of this ratio are characterized by the two-dimensional ODE for $u^{\mathsf{GD}}$ and $v^{\mathsf{GD}}$ in Proposition~\ref{prop:gd-centered-logit}. Proposition~\ref{lem:gd-ratio-threshold} further shows that
\begin{equation}
\rho_{\mathsf{GD}}(t)>\rho_{\mathsf{Muon}}
\end{equation}
whenever the loss is above the terminal threshold $\epsilon_\star$, which covers the matched-loss regime stated in the theorems.

\noindent\textbf{Step 2: Convert the representation imbalance ratio ordering into margin ordering.}
We next translate the ratio comparison into a margin comparison at the same loss threshold. Proposition~\ref{lem:fixed-ratio-margin} shows that the matched-loss margin is strictly decreasing in the representation imbalance ratio $\rho$. This monotonicity can be understood as follows: a larger $\rho$ places more weight on the between-block direction $P_c$, allowing the same loss threshold to be reached with a smaller within-block separation, and hence a smaller matched-loss margin. Since Step~1 shows that Muon has the smallest $\rho$ in the relevant regime, Muon attains the largest matched-loss margin.

\vspace{3pt}
\noindent\textbf{Step 3: Convert the representation imbalance ratio ordering into effective rank ordering.}
The same ratio also determines the effective rank. Since effective rank is computed from normalized singular values, it is invariant to the overall scale of the representation and depends only on the representation imbalance ratio $\rho=v/u$. A direct calculation shows that it is strictly decreasing in $\rho$ for $\rho>1$; see Appendix~\ref{sec:proof_of_thmerank}. Thus, the same ratio ordering that proves the margin comparison also implies that Muon attains the highest effective rank at matched loss.

Combining these steps, Muon has the smallest imbalance ratio because its spectral normalization removes gradient anisotropy at the singular-value level. Since both the matched-loss margin and the effective rank decrease as $\rho$ increases, this yields Muon's claimed advantages over Adam for all $\epsilon\in(0,\log C)$, and over GD outside the exponentially small terminal-loss window.

%% file: main_paper/conclusion.tex
\section{Conclusion}
\label{sec:conclusion}
Our work demystifies the advantages of Muon over Adam and SGD from the perspective of feature learning. Concretely, we evaluate the quality of learned features through two practical and important properties: robustness and transferability. For robustness, our experiments show that features learned by Muon are more robust to input corruptions than those learned by Adam and SGD. This robustness advantage is further supported by larger logit margins revealed by layer-wise probes. For transferability, we evaluate how effectively learned features transfer to downstream tasks by training linear classifiers on pretrained hidden states or by conducting supervised fine-tuning of the full pretrained models. The results show that Muon-learned features transfer more effectively than those learned by Adam and SGD. This transferability advantage is further supported by the diversity of hidden states, as measured by effective rank. Finally, we theoretically prove the larger logit margins and higher effective rank of Muon-learned features in a representative classification problem.\looseness=-1

One limitation of our work is that we mainly focus on \acp{llm} and vision classifiers. The performance and mechanism of Muon on other popular model families, such as diffusion models, are not studied in this paper. We leave this direction for future work.





%% file: appendix/experimental_details.tex
\section{Additional Experimental Details}
\label{app:exp_details}

\subsection{Common protocol}

Across all experiments, we match model architecture, data pipeline, training budget, and evaluation protocol across optimizers. Unless otherwise noted, every result is averaged over three random seeds. We tune the learning rate and scheduler hyperparameters for each optimizer under the same budget, following standard hyperparameter selection and training recipes rather than redesigning the full pipeline around a single optimizer \citep{chen2023symbolic}. This matched-budget protocol is intended to isolate differences in learned representations rather than differences in engineering effort.

\subsection{Vision setup}
We study two representative visual backbones: ResNet-18 and ViT-S/16, both trained on ImageNet-1K. Our implementations are adapted from the codebases of \citet{wang2024improving} and \citet{zhang2024transformers}.  ResNet-18 is trained for 100 epochs with a cosine schedule and 5 warmup epochs; ViT-S/16 is trained for 300 epochs with a cosine schedule, 30 warmup epochs, and gradient-norm clipping at $1.0$. Weight decay is fixed to $10^{-4}$ for ResNet-18; for ViT-S/16 it is swept jointly with the learning rate. For each optimizer we sweep the learning rate over $\{1,2,5\}\times\{10^{-4},\ldots,10^{0}\}$ on ResNet-18; on ViT-S/16 we sweep both the learning rate and weight decay over $\{1,3\}\times\{10^{-4},\ldots,10^{0}\}$. The best configuration is selected by ImageNet-1K validation accuracy, and the three reported seeds are retrained at that configuration. In our experiments, Adam refers to its decoupled-weight-decay variant AdamW, and SGD likewise uses decoupled weight decay. Muon is not applied to 1D parameters, embedding layers, or classification heads; these parameters are instead optimized by SGD in ResNet-18 and by AdamW in ViT-S/16. Pretraining experiments on vision models are conducted on 4 NVIDIA A40 GPUs. A single ResNet-18 pretraining run takes approximately 10 hours, ViT-S/16 takes approximately 40 hours.

After pretraining, we evaluate robustness on ImageNet-C and evaluate transferability by training a linear classifier on EuroSAT, Flowers102, Food101, and Stanford-Cars. These transfer datasets cover substantial variation in domain shift and granularity: EuroSAT probes cross-domain transfer to satellite imagery, while Flowers102 and Stanford-Cars emphasize fine-grained recognition. For each downstream dataset we freeze the backbone and train a linear classifier for 100 epochs with SGD (momentum $0.9$, no weight decay) and batch size $256$, holding out a test set for final evaluation and a validation subset for hyperparameter selection. For each checkpoint we run a coarse learning-rate search over $\{10^{-2}, 10^{-1}, 1\}$ followed by a fine refinement at ratios $\{10^{-1/4}, 1, 10^{1/4}\}$ around the best coarse value, and report test accuracy at the configuration with the highest validation accuracy. These experiments are conducted on a single NVIDIA A40 GPU and take less than 10 hours to complete.

\subsection{Language setup}
Our implementations are adapted from the code base of {\url{https://github.com/karpathy/nanoGPT}}. We use nanoGPT-style causal language models pretrained on a FineWeb10B subset with a WSD learning-rate schedule and weight decay fixed to \(0\). For each optimizer, we sweep the learning rate over \(\{1,2,5\}\times\{10^{-5},\ldots,10^{0}\}\) and select the run with the lowest validation perplexity on FineWeb10B. When using Muon, 1D parameters, embeddings, and the LM head are optimized by Adam. For GPT-2, we report results averaged over three random seeds. For GPT-2 Medium, we report results from a single seed due to computational resource constraints. Pretraining experiments on language models are conducted on 4 NVIDIA A40 GPUs. GPT-2 pretraining takes approximately 10 hours, and GPT-2 Medium takes approximately 48 hours.

We evaluate robustness on FineWeb10B-C, where the evaluation text is perturbed by typo-style surface-form corruptions, and evaluate fine-tuning behavior by supervised fine-tuning on Alpaca, Databricks-Dolly-15k, and Evol-Instruct-70k, all formatted with the standard Alpaca prompt template. We run full-parameter fine-tuning with Adam (weight decay $0$) for two epochs at sequence length $512$, with held-out validation and test subsets obtained by random shuffling. Per-device batch size is $4$ with gradient accumulation steps $16$. Cross-entropy is computed on response tokens only, with prompt tokens masked out. The learning rate is selected by lowest validation perplexity from $\{1,2,5\}\times\{10^{-6},10^{-5}\}$, a range consistent with standard SFT recipes \citep{taori2023alpaca, wang2023far}. We report test perplexity at the selected learning rate, averaged over three pretraining seeds. All fine-tuning experiments are conducted on 4 NVIDIA A40 GPUs, and each experiment takes less than 10 hours to complete.

\subsection{Construction of FineWeb10B-C}
\label{app:fineweb10c}
We construct FineWeb10B-C from the clean FineWeb10B validation split. The corruption is applied to raw text before tokenization, and the resulting corrupted files are fixed once generated. All optimizers are evaluated on exactly the same corrupted examples.

Following \citet{pruthi2019combating}, we use four character-level typo operations. \texttt{Swap} exchanges two adjacent internal characters in a word. \texttt{Drop} removes one internal character. \texttt{Add} inserts a random lowercase character at an internal position. \texttt{Keyboard} replaces one character with a neighboring key on a QWERTY keyboard. For each corruption type, we consider five severity levels with increasing perturbation intensity. The full construction code will be released with the paper.

We report model performance separately for each corruption type, averaging perplexity over the five severity levels. We also report the overall average across all corruption types and severity levels. This protocol produces a fixed typo-corrupted benchmark for evaluating pretraining robustness, rather than an adaptive attack that is re-optimized for each model.

\subsection{Layer-wise classification margin}
\label{app:margin}
For each pretrained checkpoint, we compute the layer-wise margin using a probe that mirrors the model's own pre-output normalization followed by a linear map to logits. At depth~$\ell$, we take the output of the $\ell$-th transformer block as the hidden state $z^{(\ell)}(x)$, and apply the native pooling operator for ViT-S when applicable. The probe at layer $\ell$ is defined as $T^{(\ell)}(z) = W^{(\ell)}\,\mathrm{Norm}(z) + b^{(\ell)}$, where $\mathrm{Norm}(\cdot)$ is the same normalization operator the model applies before its output head, kept frozen at its pretrained configuration, and $W^{(\ell)} \in \mathbb{R}^{C \times d}$, $b^{(\ell)} \in \mathbb{R}^{C}$ are the only trainable parameters of the probe. On a subset $\mathcal{D}$ of the validation split, we collect pairs $\big(z^{(\ell)}(x), s(x)\big)$, where $s(x)$ denotes the model's final logits, and fit $W^{(\ell)}, b^{(\ell)}$ by least-squares regression. For the final block, we use the model's own output head directly. The per-layer logits are given by $T^{(\ell)}(z^{(\ell)}(x))$, and the sample-level margin $\gamma_\ell(x,y)$ is defined as in Eq.~\ref{eq:layerwise_margin}. We report the mean clean margin on $\mathcal{D}$; means and standard deviations in Figure~\ref{fig:margin} are computed over three random seeds, each run end-to-end from pretraining through margin evaluation.

\subsection{Layer-wise feature spectrum analysis}
Here, we introduce how to construct features across modalities. Throughout, depth $\ell$ refers to the position immediately after the $\ell$-th transformer block. For language models, we use the hidden state at each token position as a feature vector, so that $Z^{(\ell)}$ stacks all token-level hidden states across the validation set. For ViT models, we instead use class-center features: for each class, we average the pooled representations of all validation samples belonging to that class, and $Z^{(\ell)}$ stacks these class-center features. In both cases, this construction is performed independently at every depth $\ell$.

\subsection{Reordering procedure.}
\label{app:reorder}
Given a symmetric similarity matrix $A \in [0, 1]^{n \times n}$, we form distances $d_{ij} = 1 - A_{ij}$, run agglomerative hierarchical clustering with average linkage to obtain a dendrogram, and select the leaf ordering that minimizes the sum of distances between adjacent leaves~\citep{bar2001fast}. The resulting permutation $P$ is then applied to both rows and columns to produce $PAP^{\top}$. We implement the process using SciPy~\citep{virtanen2020scipy}.

\section{Additional Results}
\label{app:results_more}

Detailed per-corruption robustness and per-task transfer numbers are reported in Table~\ref{tab:robustness_combined} and Table~\ref{tab:transfer_combined}, respectively. To complement the aggregated numbers reported in the main text, we additionally include standard deviations across three random seeds for every entry, except for the GPT-2 Medium experiments which were run with a single seed due to computational cost. The overall trends remain stable across runs: Muon exhibits superior robustness under distribution shift and a consistent advantage on downstream transfer tasks.

\begin{table}[htbp]
\centering
\caption{Robustness of optimizers across image and language pretraining. For ResNet-18 and ViT-S we report Top-1 accuracy (\%, $\uparrow$) on ImageNet-C across the four standard corruption categories (Noise, Blur, Weather, Digital), with the rightmost column averaging over all 15 corruption types. For GPT-2 and GPT-medium we report perplexity ($\downarrow$) on FineWeb-C across four text-corruption types (Add, Drop, Key, Swap), with the rightmost column averaging over the four. Values are mean $\pm$ std over 3 seeds, except GPT-medium which is reported from a single run. The best optimizer per column is in \textbf{bold}.}
\label{tab:robustness_combined}
\small
\setlength{\tabcolsep}{4pt}
\resizebox{\textwidth}{!}{
\begin{tabular}{llccccc}
\toprule
\multicolumn{7}{c}{\textit{ImageNet-C accuracy ($\uparrow$)}} \\
\midrule
Model & Optimizer & Noise & Blur & Weather & Digital & Mean (15) \\
\midrule
\multirow{3}{*}{ResNet-18}
 & Adam & $19.75 \pm 0.13$ & $24.05 \pm 0.13$ & $33.51 \pm 0.29$ & $34.46 \pm 0.11$ & $28.49 \pm 0.09$ \\
 & SGD  & $23.15 \pm 0.20$ & $\mathbf{27.19 \pm 0.17}$ & $37.15 \pm 0.20$ & $\mathbf{37.75 \pm 0.33}$ & $31.85 \pm 0.13$ \\
 & Muon & $\mathbf{24.51 \pm 0.45}$ & $27.05 \pm 0.27$ & $\mathbf{37.65 \pm 0.21}$ & $37.56 \pm 0.11$ & $\mathbf{32.17 \pm 0.22}$ \\
\midrule
\multirow{3}{*}{ViT-S}
 & Adam & $54.52 \pm 0.69$ & $42.35 \pm 0.40$ & $61.19 \pm 0.17$ & $58.54 \pm 0.21$ & $54.12 \pm 0.26$ \\
 & SGD  & $47.78 \pm 0.60$ & $39.09 \pm 0.17$ & $55.91 \pm 0.05$ & $53.75 \pm 0.24$ & $49.22 \pm 0.09$ \\
 & Muon & $\mathbf{55.91 \pm 0.77}$ & $\mathbf{43.97 \pm 0.36}$ & $\mathbf{63.37 \pm 0.40}$ & $\mathbf{60.15 \pm 0.03}$ & $\mathbf{55.84 \pm 0.19}$ \\
\midrule
\multicolumn{7}{c}{\textit{FineWeb-C perplexity ($\downarrow$)}} \\
\midrule
Model & Optimizer & Add & Drop & Key & Swap & Mean (4) \\
\midrule
\multirow{2}{*}{GPT-2}
 & Adam & $41.05 \pm 0.21$ & $37.97 \pm 0.46$ & $40.02 \pm 0.26$ & $39.77 \pm 0.26$ & $39.70 \pm 0.23$ \\
 & Muon & $\mathbf{39.15 \pm 0.10}$ & $\mathbf{36.26 \pm 0.47}$ & $\mathbf{38.21 \pm 0.13}$ & $\mathbf{37.92 \pm 0.13}$ & $\mathbf{37.89 \pm 0.16}$ \\
\midrule
\multirow{2}{*}{GPT-medium$^{*}$}
 & Adam & $34.87$ & $32.01$ & $34.00$ & $33.67$ & $33.64$ \\
 & Muon & $\mathbf{33.05}$ & $\mathbf{30.42}$ & $\mathbf{32.28}$ & $\mathbf{31.94}$ & $\mathbf{31.92}$ \\
\bottomrule
\end{tabular}
}

\vspace{2pt}
{\footnotesize $^{*}$ Single run, std not reported.}
\end{table}

\begin{table*}[htbp]
\centering
\caption{Linear-probe transfer for vision and instruction-tuning adaptation for language. For ResNet-18 and ViT-S, we report linear-probe accuracy (\%, $\uparrow$) on four downstream image classification benchmarks. For GPT-2 and GPT-medium, we report fine-tuning test perplexity ($\downarrow$) on three instruction datasets. Vision results and GPT-2 results are mean $\pm$ std over 3 seeds; GPT-medium is reported from a single run. The best optimizer per column is in \textbf{bold}.}
\label{tab:transfer_combined}
\small
\setlength{\tabcolsep}{4pt}
\resizebox{\textwidth}{!}{
\begin{tabular}{llcccc}
\toprule
\multicolumn{6}{c}{\textit{Vision: linear-probe accuracy ($\uparrow$)}} \\
\midrule
Backbone & Optimizer & EuroSAT & Flowers102 & Food101 & Stanford-Cars \\
\midrule
\multirow{3}{*}{ResNet-18}
 & Adam & $94.52 \pm 0.13$ & $82.80 \pm 1.13$ & $66.72 \pm 0.48$ & $49.24 \pm 1.21$ \\
 & SGD  & $94.73 \pm 0.25$ & $84.76 \pm 0.52$ & $68.67 \pm 0.21$ & $52.48 \pm 0.28$ \\
 & Muon & $\mathbf{94.86 \pm 0.20}$ & $\mathbf{85.84 \pm 0.39}$ & $\mathbf{68.93 \pm 0.16}$ & $\mathbf{52.96 \pm 0.13}$ \\
\midrule
\multirow{3}{*}{ViT-S}
 & Adam & $95.58 \pm 0.32$ & $84.57 \pm 0.48$ & $71.58 \pm 0.12$ & $44.75 \pm 2.00$ \\
 & SGD  & $95.46 \pm 0.62$ & $84.69 \pm 0.43$ & $69.22 \pm 0.16$ & $40.37 \pm 1.54$ \\
 & Muon & $\mathbf{95.80 \pm 0.32}$ & $\mathbf{85.99 \pm 0.21}$ & $\mathbf{71.86 \pm 0.24}$ & $\mathbf{46.43 \pm 0.98}$ \\
\midrule
\multicolumn{6}{c}{\textit{Language: instruction-tuning perplexity ($\downarrow$)}} \\
\midrule
Backbone & Optimizer & Alpaca & Dolly & WizardLM & \\
\midrule
\multirow{2}{*}{GPT-2}
 & Adam & $16.25 \pm 0.77$ & $8.99 \pm 0.93$ & $51.91 \pm 6.85$ & \\
 & Muon & $\mathbf{15.12 \pm 0.58}$ & $\mathbf{8.61 \pm 0.78}$ & $\mathbf{45.75 \pm 5.57}$ & \\
\midrule
\multirow{2}{*}{GPT-medium$^{*}$}
 & Adam & $12.53$ & $7.85$ & $27.28$ & \\
 & Muon & $\mathbf{12.23}$ & $\mathbf{7.62}$ & $\mathbf{26.68}$ & \\
\bottomrule
\multicolumn{6}{l}{\footnotesize $^{*}$ Single run, std not reported.} \\
\end{tabular}
}
\end{table*}

\paragraph{Layer-wise analysis on GPT-2 Medium.}
\label{app:24d_results}
As shown in Figure~\ref{fig:24d_hidden} and aggregated in Table~\ref{tab:24d_hidden_avg}, Muon attains higher margin, higher eRank, and lower Top10E than Adam across the 24 transformer layers of GPT-2 Medium, consistent with the trends reported in Section~\ref{sec:experiments}. These results are consistent with Muon learning hidden representations that are both more robust and spectrally more diverse than those produced by Adam.
\begin{figure}[t]
    \centering
    \includegraphics[width=0.32\linewidth]{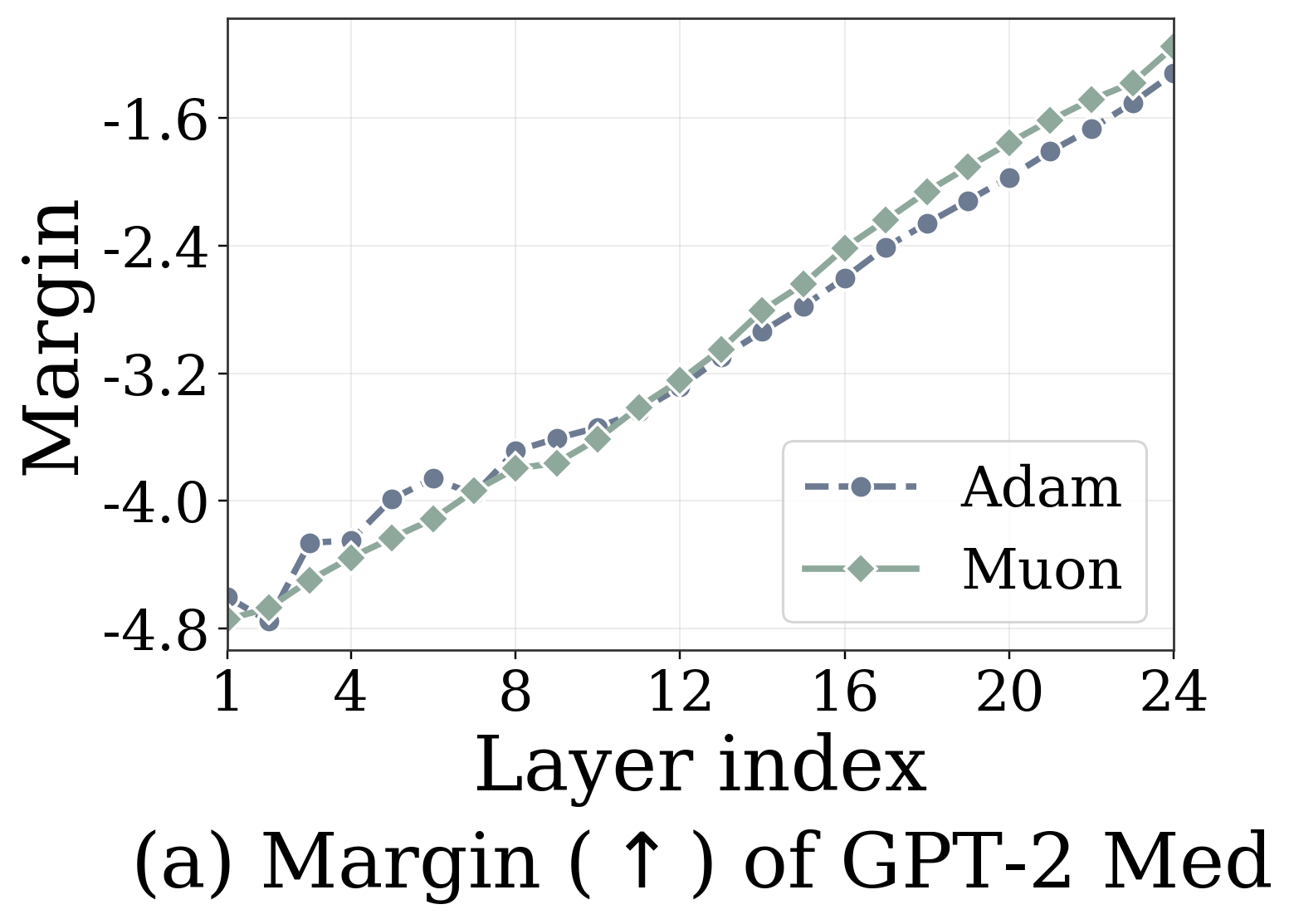}
    \hspace{0.005\linewidth}
\includegraphics[width=0.32\linewidth]{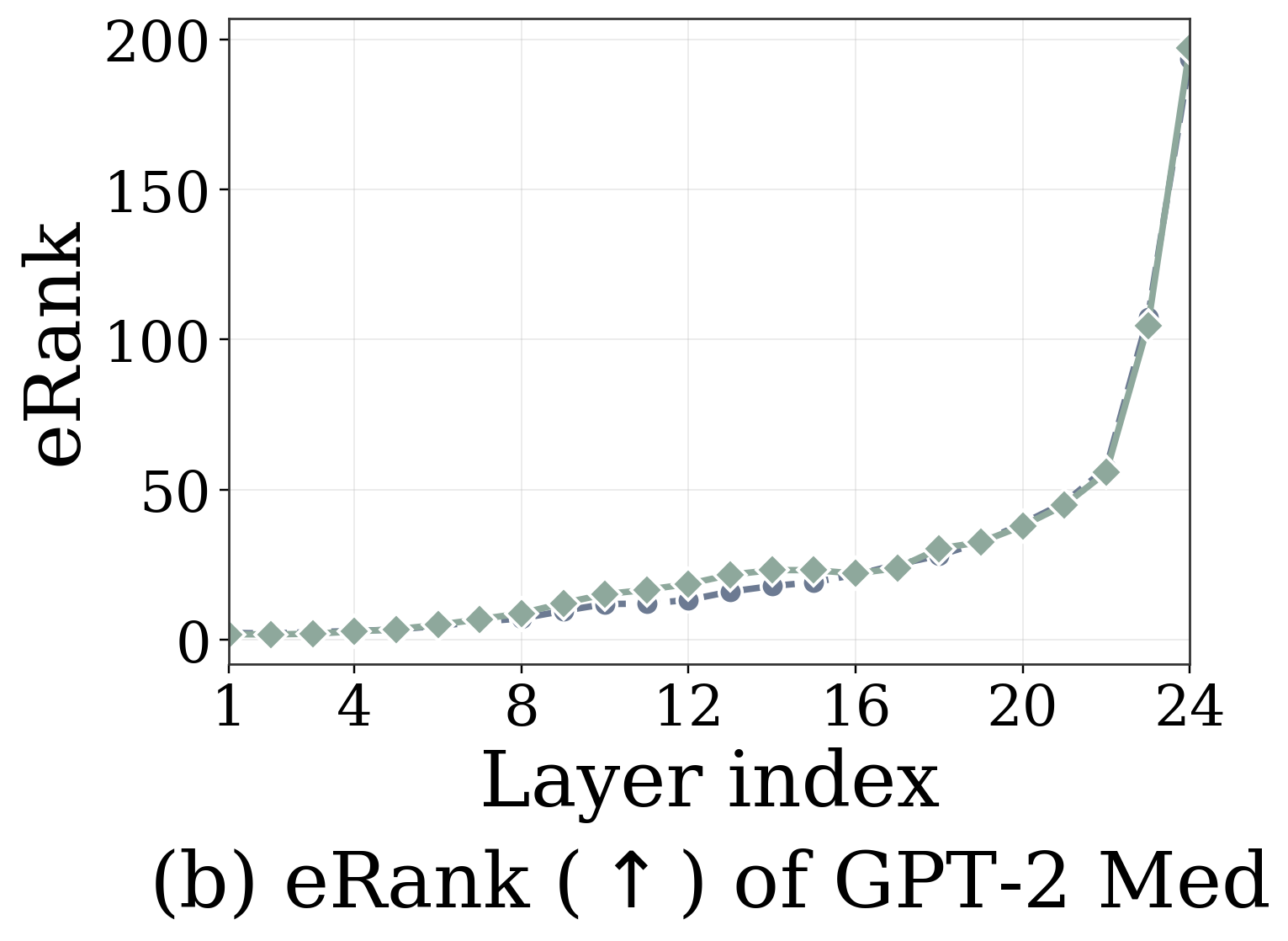}
\hspace{0.005\linewidth}
\includegraphics[width=0.32\linewidth]{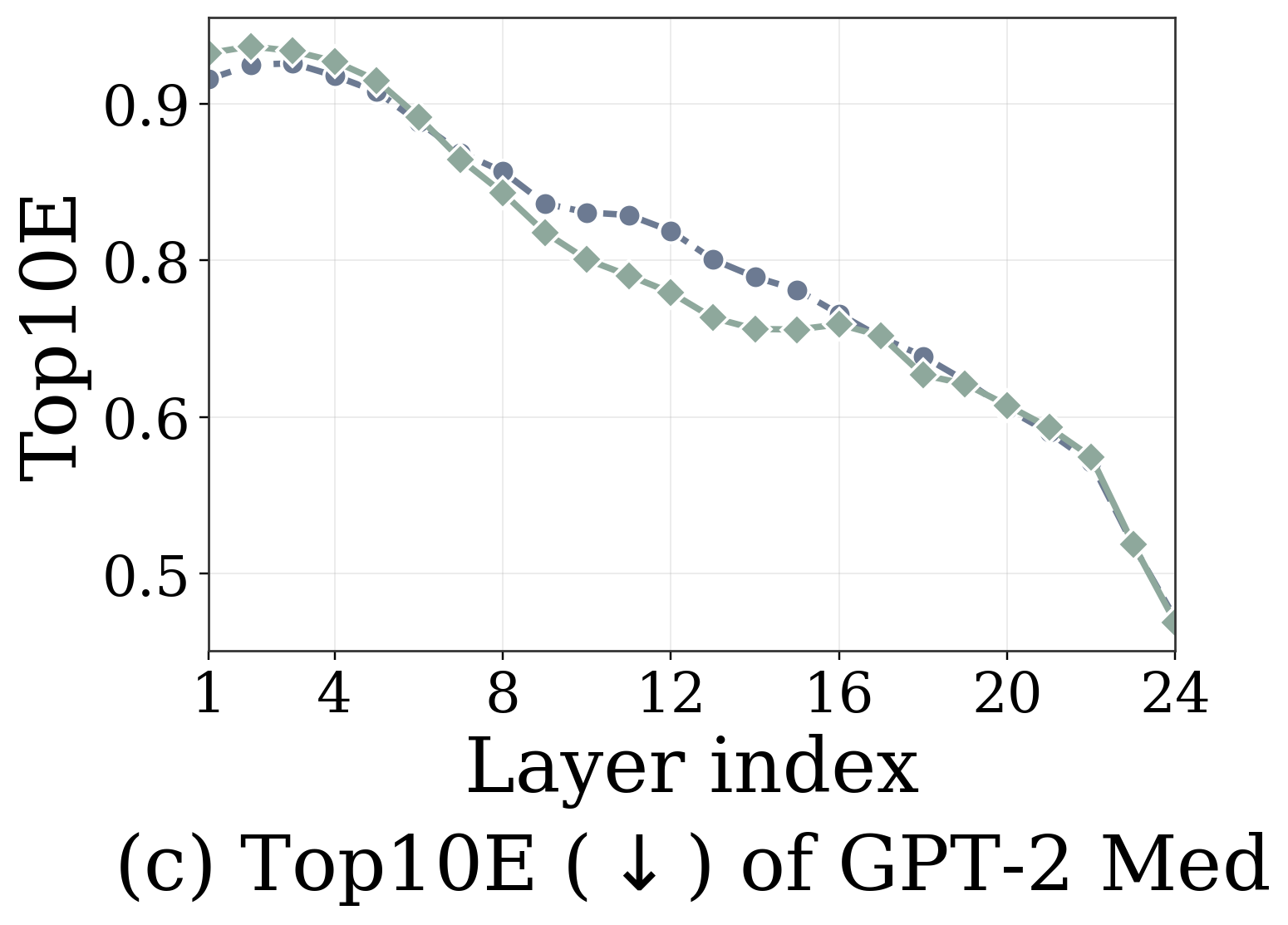}
    \caption{Layer-wise analysis on GPT-2 Medium. (a) Classification margin. (b) Effective rank and (c) Top-10 energy fraction of hidden representations. Overall, Muon yields higher margin, higher eRank, and lower Top10E than Adam.}
    \label{fig:24d_hidden}
\end{figure}

\begin{table}[h]
\centering
\caption{Layer-wise averages of margin, eRank, and Top10E on GPT-2 Medium trained with Adam and Muon. Best per column in bold.}
\label{tab:24d_hidden_avg}
\begin{tabular}{lccc}
\toprule
Optimizer & Margin ($\uparrow$) & eRank ($\uparrow$) & Top10E ($\downarrow$) \\
\midrule
Adam & $-3.0717$ & $28.34$ & $0.7453$ \\
Muon & $\mathbf{-3.0370}$ & $\mathbf{29.67}$ & $\mathbf{0.7330}$ \\
\bottomrule
\end{tabular}
\end{table}

%% file: appendix/Notations.tex
\section{Notations}

\paragraph{Basic notation.}
For a positive integer $k$, let $[k]=\{1,\dots,k\}$. For $d\in\mathbb N$, let $\mathbf{1}_d\in\mathbb{R}^d$ denote the all-ones vector, $I_d\in\mathbb{R}^{d\times d}$ the identity matrix, and $J_d=\mathbf{1}_d\mathbf{1}_d^\top\in\mathbb{R}^{d\times d}$ the all-ones matrix. When the dimension is clear from context, we may omit the subscript and write simply $\mathbf{1}$, $I$, and $J$. For matrices of the same size, $\langle A,B\rangle_F=\mathrm{tr}(A^\top B)$ denotes the Frobenius inner product, while $\|A\|_F$ and $\|A\|_{\mathrm{op}}$ denote the Frobenius norm and operator norm, respectively. We use $\otimes$ for the Kronecker product and $[A,B]$ for horizontal concatenation.

\paragraph{Hierarchical labels and features.}
Fix integers $m,n\ge 2$ and let $C=mn$. The label set is $\mathcal C=[m]\times[n]$. For each class $c=(g,r)\in\mathcal C$, define
$x_{c}=[a\cdot e_g,\; b\cdot  e_{g,r}]^\top \in \mathbb R^{m+C},\, 0<a\le b.$
The population is uniform on $\mathcal C$. Let $R\in\{0,1\}^{m\times C}$ be the coarse-membership matrix whose column indexed by $c=(g,r)$ is $e_g$, and define $F = [aR^\top,\; bI_C]^\top \in \mathbb R^{(m+C)\times C}.$
For a linear classifier \(W\in\mathbb{R}^{C\times(m+C)}\), let
\(Z(W)=\Pi WF\in\mathbb{R}^{C\times C}\) denote the centered representation, where the centering matrix $\Pi$ is defined in the next paragraph. The population cross-entropy is defined as
\(
\mathcal L(W)
=
-{1}/{C}\sum_{c\in\mathcal C}\log\bigl(\soft(Z(W))_{c,c}\bigr)
\), where $\soft(\cdot)$ denotes the column-wise softmax.
We will also write \(\mathcal L(Z)\) to denote the same loss evaluated
directly at a logit matrix \(Z\in\mathbb{R}^{C\times C}\).

\paragraph{Centering and hierarchical projectors.}
We define the following matrix
\begin{align*}
    J=J_C,\qquad \Pi=I_C-\frac{1}{C}J,\qquad
S=R^\top R = I_m\otimes J_n,\qquad
P_f=I_C-\frac{1}{n}S,\qquad
P_c=\frac{1}{n}S-\frac{1}{C}J.
\end{align*}
Then $\Pi=P_f+P_c$. We write $
r_f=\operatorname{rank}(P_f)=m(n-1),\,  r_c=\operatorname{rank}(P_c)=m-1,\, \delta = a^2 n / b^2$
and $s=\sqrt{a^2n+b^2}.$

\paragraph{Canonical plane.}
Define the canonical directions 
$Q_f = [\,0,\ P_f\,] \in \mathbb{R}^{C\times(m+C)}$ and 
$Q_c = {1}/{s}\,[\,aP_cR^\top,\ bP_c\,] \in \mathbb{R}^{C\times(m+C)}$, 
and let $\mathcal{W}_{\mathrm{cano}} = \mathrm{span}\{Q_f, Q_c\}$ denote the canonical plane in weight space.

\paragraph{Symmetric logits and effective coordinates.}
A logit matrix \(M\) is called symmetric if it lies in \(\mathcal M = \bigl\{uP_f+vP_c+\zeta J/C : u,v,\zeta\in\mathbb R\bigr\}.\) Since \(\Pi P_f = P_f\), \(\Pi P_c = P_c\), and \(\Pi J = 0\) by Lemma~\ref{lem:projectors-full}, the centered part \(\Pi M = uP_f + vP_c\) depends only on \(u\) and \(v\). We write \(u(M), v(M)\) for these coefficients, so that \(\Pi M = u(M)\,P_f + v(M)\,P_c\). For a checkpoint \(W\), we write \(u(W) = u(WF)\) and \(v(W) = v(WF)\).

\paragraph{Effective rank.}
For a nonzero matrix $M$ with singular values $\sigma_1,\dots,\sigma_r>0$, define
\[
p_i=\frac{\sigma_i^2}{\sum_{j=1}^r \sigma_j^2},
\qquad
\mathrm{eRank}(M)=
\exp\!\Bigl(-\sum_{i=1}^r p_i\log p_i\Bigr).
\]

\paragraph{Margin.} For each class \(c\in\mathcal C\), \(\gamma(W;x_c,c)=(W x_c)_c-\max_{j\ne c}(W x_c)_j.\) We define the average margin as \(\Gamma(W)=1/C\sum_{c\in\mathcal C}\gamma(W;x_c,c).\)

\paragraph{Optimization dynamics.}
We recall the three continuous-time dynamics analyzed in this paper; see Section~\ref{sec:case_study} for the full setup. All flows are initialized at \(W^\mathsf{A}(0)=0\) for \(\mathsf{A}\in\{\mathsf{GD},\mathsf{Muon},\mathsf{Adam}\}\), with momentum set to zero:
\begin{itemize}[leftmargin=2em]
    \item[(GD)] 
    \(\dot W^{\mathsf{GD}}(t)=-\nabla_W\mathcal L(W^{\mathsf{GD}}(t))\).
    \item[(Muon)] 
    \(\dot W^{\mathsf{Muon}}(t)=-\spec(\nabla_W\mathcal L(W^{\mathsf{Muon}}(t)))\),
    where \(\spec(G)=UV^\top\) for SVD \(G=U\Sigma V^\top\).
    \item[(Adam)]
    \(\dot W^{\mathsf{Adam}}(t)=-\sgn(\nabla_W\mathcal L(W^{\mathsf{Adam}}(t)))\),
    where \(\sgn(\cdot)\) is the entrywise sign operator.
\end{itemize}

\paragraph{Exact-loss checkpoints.}
For any method $\mathsf{A}\in\{\mathsf{GD},\mathsf{Muon},\mathsf{Adam}\}$, write $Z^\mathsf{A}(t)=Z(W^\mathsf{A}(t))=\Pi W^\mathsf{A}(t)F$ for the centered logit matrix at time $t$. Given $\epsilon\in(0,\log C)$, we write $W_\epsilon^\mathsf{A}$ for the exact-loss checkpoint satisfying $\mathcal L_{ }(W_\epsilon^\mathsf{A})=\epsilon.$ When discrete-time iterates are needed, we use $W_t^\mathsf{A}$ for the corresponding iterates and interpret $W_\epsilon^\mathsf{A}$ via the standard first-crossing interpolation.

%% file: appendix/proof_thm.tex
\section{Proofs of Theorems}
\label{sec:proofs}

\subsection{Proof of Theorem~\ref{thm:main-margin}}
\label{app:proof-margin}

\begin{proof}

We first establish that the matched-loss times used in the theorem are well defined. The proof of the following proposition is deferred to Appendix~\ref{proof_gd-matched-loss-time}.
\begin{proposition}[Matched-loss time along GD, Adam and Muon]
\label{lem:gd-matched-loss-time}
Assume \(a>0\), \(b>0\), and \(m,n\ge2\). Along the zero-initialized GD, Adam and Muon flows, the losses
are strictly decreasing from \(\log C\)
to \(0\). Hence, for every \(\epsilon\in(0,\log C)\), there exist
unique times \(t_\epsilon^{\mathsf{GD}}>0\), \(t_\epsilon^{\mathsf{Adam}}>0\) and \(t_\epsilon^{\mathsf{Muon}}>0\) such that
\[
\mathcal L(W^{\mathsf{GD}}(t_\epsilon^{\mathsf{GD}}))=\epsilon, \qquad
\mathcal L(W^{\mathsf{Adam}}(t_\epsilon^{\mathsf{Adam}}))=\epsilon ,
\qquad
\mathcal L(W^{\mathsf{Muon}}(t_\epsilon^{\mathsf{Muon}}))=\epsilon.
\]
\end{proposition}
This proposition guarantees that the exact-loss checkpoints $W_\epsilon^{\mathsf{GD}}$, $W_\epsilon^{\mathsf{Adam}}$, and $W_\epsilon^{\mathsf{Muon}}$ are well defined for every $\epsilon \in (0, \log C)$, so that the matched-loss margin comparison in the theorem makes sense. With this in hand, we prove the theorem in four steps. The main idea is to reduce the margin comparison to a comparison of one scalar ratio.

\begin{itemize}
    \item \textbf{Step 1: Reduction of optimizer dynamics to canonical coordinates.} We show that the raw outputs of GD, Muon, and Adam all lie in the symmetric class $\mathcal{M}$, i.e., $W^\mathsf{A}(t)F=u(t)P_f+v(t)P_c+\zeta(t)J/C$ for some $u(t), v(t), \zeta(t)$ (with $\zeta(t)=0$ for GD and Muon). For Muon and Adam, this gives explicit trajectories. For GD, it gives a two-dimensional ODE for $u(t)$ and $v(t)$.

    \item \textbf{Step 2: Monotonicity of the margin in the representation imbalance ratio \(\rho=v/u\).} We analyze the margin in $\mathcal{M}$.
    At a fixed loss level, the margin along a fixed-ratio ray depends only
    on \(\rho\), and is strictly decreasing in \(\rho\) on \([1,\infty)\).

    \item \textbf{Step 3: Ratio comparisons.} We compare the ratios induced by the three
    optimizers. Muon has a smaller ratio than Adam, which proves
    part~\textnormal{(i)}. For GD, we show that above an explicit
    terminal-loss threshold \(\epsilon_\star\), GD has a larger ratio than
    Muon. By Step 2, this implies that Muon has larger margin.

    \item \textbf{Step 4: Converting the explicit threshold into the asymptotic window.} We estimate the scale of \(\epsilon_\star\).
    This shows that the asymptotic loss windows in the theorem lie above
    \(\epsilon_\star\), so the explicit-threshold comparison from Step 3
    applies and proves part~\textnormal{(ii)}.
\end{itemize}

\medskip
\noindent\textbf{Step 1: Reduction of optimizer dynamics to canonical coordinates.}

To compare the margins of Muon, Adam, and GD, we first need a precise description of their trajectories in logit space. The key observation is that all three optimizers admit a three-dimensional reduction. We summarize the required structural facts below.

\begin{proposition}[GD logit dynamics and matched-loss parametrization]
\label{prop:gd-centered-logit}
Assume \(a>0\), \(b>0\), and \(m,n\ge2\). Let
\(
s=\sqrt{a^2n+b^2}.
\)
Along the zero-initialized GD flow, the raw output and centered representation satisfy
\[
    W^{\mathsf{GD}}(t)F= Z^{\mathsf{GD}}(t)
    =
    u^{\mathsf{GD}}(t)P_f+v^{\mathsf{GD}}(t)P_c,
\]
where \(u^{\mathsf{GD}}(t)\) and \(v^{\mathsf{GD}}(t)\) solve
\[
\begin{cases}
\dot u^{\mathsf{GD}}(t)
=
\displaystyle
\frac{b^2}{m}\,
\frac{
e^{-u^{\mathsf{GD}}(t)}
+
(m-1)e^{-q^{\mathsf{GD}}(t)}
}{
D^{\mathsf{GD}}(t)
},
\\[1.2em]
\dot v^{\mathsf{GD}}(t)
=
\displaystyle
s^2
\frac{
e^{-q^{\mathsf{GD}}(t)}
}{
D^{\mathsf{GD}}(t)
}
\end{cases}
\]
with $
u^{\mathsf{GD}}(0)=v^{\mathsf{GD}}(0)=0$ and
\[
q^{\mathsf{GD}}(t)
=
\frac{(n-1)u^{\mathsf{GD}}(t)+v^{\mathsf{GD}}(t)}{n},
\qquad
D^{\mathsf{GD}}(t)
=
1+(n-1)e^{-u^{\mathsf{GD}}(t)}
+n(m-1)e^{-q^{\mathsf{GD}}(t)}.
\]
\end{proposition}
This proposition shows that the GD trajectory in logit space stays in the canonical span of \(P_f\) and \(P_c\), and is governed by a closed two-dimensional ODE.

\begin{proposition}[Muon logit trajectory]
\label{prop:muon-centered-logit}
Assume \(a>0\), \(b>0\), and \(m,n\ge2\). Let
\(
s=\sqrt{a^2n+b^2}.
\)
Along the zero-initialized idealized Muon flow, the raw output and centered representation satisfy
\[
    W^{\mathsf{Muon}}(t)F= Z^{\mathsf{Muon}}(t)
    =
    u^{\mathsf{Muon}}(t)P_f+v^{\mathsf{Muon}}(t)P_c,
\]
where
\(
    u^{\mathsf{Muon}}(t)=bt,
    \,
    v^{\mathsf{Muon}}(t)=st .
\)
\end{proposition}
This proposition gives the corresponding canonical-coordinate description for the Muon logits. In this case, both coordinates are linear functions of time.

\begin{proposition}[Adam logit trajectory]
\label{prop:signgd-centered-logit}
Assume \(a>0\), \(b>0\), and \(m,n\ge2\). Along the zero-initialized
Adam flow, the raw output satisfies
\[
    W^{\mathsf{Adam}}(t)F
    =
    u^{\mathsf{Adam}}(t)P_f
    +
    v^{\mathsf{Adam}}(t)P_c
    +
    \frac{\zeta^{\mathsf{Adam}}(t)}{C}\cdot J,
\]
where
\(
    u^{\mathsf{Adam}}(t)=2bt,
    \,
    v^{\mathsf{Adam}}(t)=2(an+b)t,
\)
and
\(
    \zeta^{\mathsf{Adam}}(t)
    =
    t\bigl(an(2-m)+b(2-mn)\bigr).
\)
Equivalently, the centered representation  satisfies
\[
     Z^{\mathsf{Adam}}(t)
    =
    u^{\mathsf{Adam}}(t)P_f
    +
    v^{\mathsf{Adam}}(t)P_c.
\]
\end{proposition}
This proposition gives the analogous description for the Adam logits. The only new feature is the additional \(J\)-component in the raw output, which disappears after centering and leaves a two-dimensional canonical trajectory. The proofs of these propositions are deferred to Appendix~\ref{proof_prop:gd-centered-logit}, \ref{proof_prop:muon-centered-logit} and \ref{proof_prop:signgd-centered-logit}.

\medskip
\noindent\textbf{Step 2: Monotonicity of the margin in the representation imbalance ratio \(\rho\).}

The next ingredient is that, at a fixed loss level, the matched-loss margin in $\mathcal{M}$ is controlled by the ratio $\rho=v/u$.

\begin{proposition}[Matched-loss margin as a function of the ratio]
\label{lem:fixed-ratio-margin}
Fix \(\epsilon\in(0,\log C)\). For each fixed \(\rho\ge1\), consider logits of the form
\(
M=rP_f+\rho rP_c+\zeta J/C .
\)
Among such logits, there is a unique scale \(r>0\) for which the corresponding logits reach loss \(\epsilon\). Moreover, any such logits reaching loss \(\epsilon\) have the same average margin.
We denote this common margin by \(r_\epsilon(\rho).\) Then the mapping \(r_\epsilon(\rho)\) is strictly decreasing in \(\rho\) on \([1,\infty)\).
\end{proposition}

The proof is deferred to Appendix~\ref{proof_fixed-ratio-margin}. This proposition is the bridge from margins to ratios: at matched loss, the smaller ratio \(\rho=v/u\) gives the larger margin.

\medskip
\noindent\textbf{Step 3: Ratio comparisons.}

We now use the reduction from Step 2 to compare the margin induced by three optimizers through their ratios \(\rho=v/u\). First consider Muon and Adam. By Step 1, we have
\[
    \rho_{\mathsf{Muon}}
    =
    \frac{v^{\mathsf{Muon}}(t)}
         {u^{\mathsf{Muon}}(t)}
    =
    \sqrt{1+\frac{a^2n}{b^2}},\qquad
    \rho_{\mathsf{Adam}}
    =
    \frac{v^{\mathsf{Adam}}(t)}
         {u^{\mathsf{Adam}}(t)}
    =
    1+\frac{an}{b}.
\]
Direct calculation gives
\(
    \rho_{\mathsf{Adam}}>\rho_{\mathsf{Muon}}>1
\). By Step 2, the margin is strictly
decreasing in \(\rho\) on \((1,\infty)\). Hence for every \(\epsilon\in(0,\log C)\), we have
\[
\Gamma(W_\epsilon^{\mathsf{Muon}})
=
r_\epsilon(\rho_{\mathsf{Muon}})
>
r_\epsilon(\rho_{\mathsf{Adam}})
=
\Gamma(W_\epsilon^{\mathsf{Adam}}).
\]
This proves part \textnormal{(i)}.

We next compare the margin of Muon with that of GD. Let
\(
\theta
=
\sqrt{1+{a^2n}/{b^2}}-1,
\,
\rho_M
=
1+\theta
=
\rho_{\mathsf{Muon}},
\,
\epsilon_\star
=
mn(1+m\theta)^{-n/\theta}.
\) By Lemma~\ref{lem:eps-star-scale}, we have \(0<\epsilon_\star<\log C .\) By Proposition~\ref{lem:gd-matched-loss-time}, for any fixed $\epsilon \in [\epsilon_\star, \log C)$, there exists a unique $t_\epsilon^{\mathsf{GD}}$ with $\mathcal{L}(W^{\mathsf{GD}}(t_\epsilon^{\mathsf{GD}})) = \epsilon$. The following proposition compares the ratio at this time with $\rho_{\mathsf{Muon}}$.


\begin{proposition}[GD is more coarse-biased than Muon above the explicit threshold]
\label{lem:gd-ratio-threshold}
Assume \(0<a\le b\). If
\(
\mathcal L(W^{\mathsf{GD}}(t))\ge \epsilon_\star,
\)
then we have
\(
{v^{\mathsf{GD}}(t)}/{u^{\mathsf{GD}}(t)}
>
\rho_{M} .
\)
Equivalently, for every \(\epsilon\in[\epsilon_\star,\log C)\),
\(
{
v^{\mathsf{GD}}(t_\epsilon^{\mathsf{GD}})
}/{
u^{\mathsf{GD}}(t_\epsilon^{\mathsf{GD}})
}
>
\rho_{M} .
\)
\end{proposition}

The proof is deferred to Appendix~\ref{Proof_lem:gd-ratio-threshold}. Since $\epsilon \geq \epsilon_\star$, Proposition~\ref{lem:gd-ratio-threshold} gives
\[\rho_{\mathsf{GD}}(\epsilon)=\frac{v^{\mathsf{GD}}(t_\epsilon^{\mathsf{GD}})}{u^{\mathsf{GD}}(t_\epsilon^{\mathsf{GD}})}>\rho_M=\rho_{\mathsf{Muon}} .\]
Since the margin is strictly decreasing in $\rho$ by Step 2, we have 
\[
\Gamma(W_\epsilon^{\mathsf{Muon}})=r_\epsilon(\rho_{\mathsf{Muon}})>r_\epsilon\big(\rho_{\mathsf{GD}}(\epsilon)\big)=\Gamma(W_\epsilon^{\mathsf{GD}}).
\]
Hence Muon has strictly larger margin than GD for every $\epsilon\in[\epsilon_\star,\log C)$.

\medskip
\noindent\textbf{Step 4: Converting the explicit threshold into the asymptotic window.}

It remains to show that the theorem's asymptotic loss windows lie above
\(\epsilon_\star\). By Lemma~\ref{lem:eps-star-scale}, for every
\(0<c<c_0\),
\(\epsilon_\star=o\bigl(mn(1+m\sqrt n)^{-c\sqrt n}\bigr)\). Hence, if
\(\epsilon=\Omega\bigl(mn(1+m\sqrt n)^{-c\sqrt n}\bigr)\), then
\(\epsilon\ge\epsilon_\star\) for all sufficiently large \(m,n\). The
explicit-window comparison from Step 3 applies.

Finally, when \(m=\Theta(n)\), Lemma~\ref{lem:eps-star-scale} gives
\(\epsilon_\star=o\bigl(\exp(-c'\sqrt n\log n)\bigr)\) for every
\(0<c'<c_0'\). Therefore, if
\(\epsilon=\Omega\bigl(\exp(-c'\sqrt n\log n)\bigr)\), then
\(\epsilon\ge\epsilon_\star\) for all sufficiently large \(n\), and the same explicit-window comparison proves the claim. This proves part \textnormal{(ii)}.\\
Thus, we conclude the proof of Theorem~\ref{thm:main-margin}. 
\end{proof}

\subsection{Proof of Theorem \ref{thm:erank-merged}}
\label{sec:proof_of_thmerank}
\begin{proof}
We first verify that the matched-loss times in the theorem are well defined. By Proposition~\ref{lem:gd-matched-loss-time}, along the GD, Adam, and Muon flows, the losses are strictly decreasing from $\log C$ to $0$. Hence $t_\epsilon^{\mathsf{GD}}$, $t_\epsilon^{\mathsf{Adam}}$, and $t_\epsilon^{\mathsf{Muon}}$ are well defined for every $\epsilon \in (0, \log C)$.

We then prove the theorem in four steps. The proof follows the same structure as Theorem~\ref{thm:main-margin}; the only difference is Step 2, where we replace the margin monotonicity with an effective-rank monotonicity.

\begin{itemize}
    \item \textbf{Step 1: Reduction of optimizer dynamics to canonical coordinates.} We show that the centered representations of GD, Muon,
    and Adam all stay in the same two-dimensional canonical space,
    \(Z(t)=u(t)P_f+v(t)P_c\). For Muon and Adam, this gives
    explicit trajectories. For GD, it gives a two-dimensional ODE for
    \(u(t)\) and \(v(t)\).

    \item \textbf{Step 2: Monotonicity of effective rank in the representation imbalance ratio \(\rho=v/u\).} We compute the singular values of \(uP_f+vP_c\).
    This shows that the effective rank depends only on the ratio
    \(\rho\), and is strictly decreasing in \(\rho\) whenever
    \(\rho>1\).

    \item \textbf{Step 3: Ratio comparisons.} We compare the ratios induced by the three
    optimizers. Muon has a smaller ratio than Adam, which proves
    part~\textnormal{(i)}. For GD, we show that above an explicit
    terminal-loss threshold \(\epsilon_\star\), GD has a larger ratio than
    Muon. By Step 2, this implies that Muon has larger effective rank.

    \item \textbf{Step 4: Converting the explicit threshold into the asymptotic window.} We estimate the scale of \(\epsilon_\star\).
    This shows that the asymptotic loss windows in the theorem lie above
    \(\epsilon_\star\), so the explicit-threshold comparison from Step 3
    applies and proves part~\textnormal{(ii)}.
\end{itemize}

\medskip
\noindent\textbf{Step 1: Reduction of optimizer dynamics to canonical coordinates.}

By Propositions~\ref{prop:gd-centered-logit},
\ref{prop:muon-centered-logit}, and \ref{prop:signgd-centered-logit}, the representations of the three optimizers have the following forms. For Adam, Proposition~\ref{prop:signgd-centered-logit} gives
\[
Z^{\mathsf{Adam}}(t)
=
u^{\mathsf{Adam}}(t)P_f
+
v^{\mathsf{Adam}}(t)P_c,
\]
where
\(
u^{\mathsf{Adam}}(t)=2bt
\)
and
\(
v^{\mathsf{Adam}}(t)=2(an+b)t
\). Similarly, Propositions~\ref{prop:gd-centered-logit} and
\ref{prop:muon-centered-logit} give
\[
Z^{\mathsf{GD}}(t)
=
u^{\mathsf{GD}}(t)P_f+v^{\mathsf{GD}}(t)P_c,
\qquad
Z^{\mathsf{Muon}}(t)
=
u^{\mathsf{Muon}}(t)P_f+v^{\mathsf{Muon}}(t)P_c
\]
where \(
u^{\mathsf{Muon}}(t)=bt,\,
v^{\mathsf{Muon}}(t)=st
,\,s=\sqrt{a^2n+b^2}\). 
For GD, Proposition~\ref{prop:gd-centered-logit} gives the corresponding canonical coordinates \(u^{\mathsf{GD}}(t),v^{\mathsf{GD}}(t)\) through the two-dimensional ODE stated there.

\medskip
\noindent\textbf{Step 2: Monotonicity of effective rank in the representation imbalance ratio \(\rho\).}

By Step 1, for each optimizer
\(\mathsf{A}\in\{\mathsf{Muon},\mathsf{Adam},\mathsf{GD}\}\), the centered representation lies
in the two-dimensional canonical space
\[
     Z^{\mathsf{A}}(t)
    =
    u^{\mathsf{A}}(t)P_f+v^{\mathsf{A}}(t)P_c.
\]
Thus, to compare the effective ranks of different optimizers, it suffices
to understand the effective rank of matrices of the form
\(
    M=uP_f+vP_c,
    \,
    u>0,\, v>0.
\)
Let
\(
    \rho={v}/{u}.
\)
Because \(P_f\) and \(P_c\) are orthogonal projections with ranks
\(
    r_f=\operatorname{rank}(P_f)=m(n-1),
    \,
    r_c=\operatorname{rank}(P_c)=m-1,
\)
the matrix \(M\) acts by multiplication by \(u\) on \(\operatorname{Im}(P_f)\),
by multiplication by \(v\) on \(\operatorname{Im}(P_c)\), and vanishes on the
orthogonal complement. Hence the nonzero singular values of \(M\) are \(u\) with multiplicity $r_f$, $v$ with multiplicity  $r_c$. Therefore, normalizing the squared singular values by their total mass
$r_f u^2 + r_c v^2 = u^2(r_f + r_c\rho^2)$,
we obtain
\[
\mathrm{eRank}(M)
=
\exp\!\left(
\log(r_f+r_c\rho^2)
-
\frac{2r_c\rho^2}{r_f+r_c\rho^2}\log \rho
\right).
\]
We denote this quantity by \(E(\rho)\). Thus
\(\mathrm{eRank}(uP_f+vP_c)\) depends only on the ratio
\(\rho=v/u\). Moreover, direct calculation gives
\[
\frac{d}{d\rho}\log E(\rho)
=
-\frac{4r_fr_c\rho}{(r_f+r_c\rho^2)^2}\log \rho,
\]
which is strictly negative for every \(\rho>1\). Hence \(E(\rho)\) is strictly decreasing on \((1,\infty)\).
Consequently, once the optimizer trajectories are written in canonical
coordinates, comparing effective ranks is equivalent to comparing their
ratios \(\rho=v/u\): among ratios larger than one, the smaller ratio gives
the larger effective rank.

\medskip
\noindent\textbf{Step 3: Ratio comparisons.}

We now use the reduction from Step 2 to compare the effective ranks induced by three optimizers through their ratios \(\rho=v/u\). From Step 1, Muon and Adam have fixed ratios
\[
\rho_{\mathsf{Muon}}=\sqrt{1+\frac{a^2n}{b^2}},
\qquad
\rho_{\mathsf{Adam}}=1+\frac{an}{b}.
\]
Hence \(\rho_{\mathsf{Adam}}>\rho_{\mathsf{Muon}}>1\). By Step 2, the effective rank is strictly decreasing in \(\rho\) on \((1,\infty)\). Therefore, for every
\(\epsilon\in(0,\log C)\), we have
\(\mathrm{eRank}\bigl( Z(W_\epsilon^{\mathsf{Muon}})\bigr)>\mathrm{eRank}\bigl( Z(W_\epsilon^{\mathsf{Adam}})\bigr)\).
This proves part \textnormal{(i)}.

We next compare the effective rank of Muon with that of GD. Let
\(\theta=\sqrt{1+{a^2n}/{b^2}}-1,
\,\rho_M=1+\theta=\rho_{\mathsf{Muon}},\,\epsilon_\star=mn(1+m\theta)^{-n/\theta}.\) By Lemma~\ref{lem:eps-star-scale}, we have \(0<\epsilon_\star<\log C .\) We first prove the comparison on the explicit loss window
\(\epsilon\in[\epsilon_\star,\log C).\) By
Proposition~\ref{lem:gd-matched-loss-time}, for any fixed $\epsilon$ in this window, there exists a unique time
\(t_\epsilon^{\mathsf{GD}}\) satisfying
\(
\mathcal L(W^{\mathsf{GD}}(t_\epsilon^{\mathsf{GD}}))=\epsilon .
\)
Since \(\epsilon\ge\epsilon_\star\), Proposition~\ref{lem:gd-ratio-threshold}
gives
\[
\frac{
v^{\mathsf{GD}}(t_\epsilon^{\mathsf{GD}})
}{
u^{\mathsf{GD}}(t_\epsilon^{\mathsf{GD}})
}
>
\rho_M=\rho_{\mathsf{Muon}}.
\]

Since effective rank is strictly decreasing in \(\rho\) by Step 2, we obtain \(\mathrm{eRank}\bigl( Z(W_\epsilon^{\mathsf{Muon}})\bigr)
>
\mathrm{eRank}\bigl( Z(W_\epsilon^{\mathsf{GD}})\bigr)\). Hence Muon has strictly larger effective rank than GD for every
\(\epsilon\in[\epsilon_\star,\log C)\).

\medskip
\noindent\textbf{Step 4: Converting the explicit threshold into the asymptotic window.}

The argument is identical to Step 4 in the proof of 
Theorem~\ref{thm:main-margin}: by Lemma~\ref{lem:eps-star-scale}, 
the asymptotic loss windows in the theorem lie above \(\epsilon_\star\), 
so the explicit-threshold comparison from Step 3 yields
\[
\mathrm{eRank}\bigl(Z(W_\epsilon^{\mathsf{Muon}})\bigr)
> \mathrm{eRank}\bigl(Z(W_\epsilon^{\mathsf{GD}})\bigr).
\]
This proves part \textnormal{(ii)}. Thus, we conclude the proof of Theorem~\ref{thm:erank-merged}. 
\end{proof}

%% file: appendix/proof_prop.tex
\section{Proofs of Supporting Propositions}
\label{app:prop}

\subsection{Proof of Proposition~\ref{prop:gd-centered-logit}}
\label{proof_prop:gd-centered-logit}
\begin{proof}
By Lemma~\ref{lem:W-reduction}, whenever
\(W=\mu_fQ_f+\mu_cQ_c\), \(u=b\mu_f\), and \(v=s\mu_c\), we have
\[
\nabla_W\mathcal L(W)
=
\frac{b\,\alpha(u,v)}{C}\cdot Q_f
+
\frac{s\,\beta(u,v)}{C}\cdot Q_c .
\]
Thus, the GD vector field is tangent to \(\operatorname{span}\{Q_f,Q_c\}\).
Since \(W^{\mathsf{GD}}(0)=0\), the GD trajectory stays in this plane:
\(
W^{\mathsf{GD}}(t)=\mu_f(t)Q_f+\mu_c(t)Q_c .
\)
Define
\(
u^{\mathsf{GD}}(t)=b\mu_f(t),
\,
v^{\mathsf{GD}}(t)=s\mu_c(t).
\)
Then Lemma~\ref{lem:canonical-full} gives
\(
W^{\mathsf{GD}}(t)F
=
u^{\mathsf{GD}}(t)P_f+v^{\mathsf{GD}}(t)P_c.
\)
Since \(\Pi P_f=P_f\) and \(\Pi P_c=P_c\) by Lemma~\ref{lem:projectors-full}, we have
\[
Z^{\mathsf{GD}}(t)
=
 \Pi W^{\mathsf{GD}}(t)F
=
u^{\mathsf{GD}}(t)P_f+v^{\mathsf{GD}}(t)P_c.
\]

It remains to derive the ODE for \(u^{\mathsf{GD}}\) and \(v^{\mathsf{GD}}\).
From the GD flow and the expression for the gradient above, we have
\[
\dot u^{\mathsf{GD}}(t)
=
-\frac{b^2\alpha(u^{\mathsf{GD}}(t),v^{\mathsf{GD}}(t))}{C},
\qquad
\dot v^{\mathsf{GD}}(t)
=
-\frac{s^2\beta(u^{\mathsf{GD}}(t),v^{\mathsf{GD}}(t))}{C}.
\]
Substituting the formulas for \(\alpha\) and \(\beta\) from
Lemma~\ref{lem:loss-reduction}, with
$q^{\mathsf{GD}}(t)=\bigl((n-1)u^{\mathsf{GD}}(t)+v^{\mathsf{GD}}(t)\bigr)/n$
and
\(
D^{\mathsf{GD}}(t)
=
1+(n-1)e^{-u^{\mathsf{GD}}(t)}
+n(m-1)e^{-q^{\mathsf{GD}}(t)},
\)
we have
\[
\begin{cases}
\dot u^{\mathsf{GD}}(t)
=
\frac{b^2}{m}\,
\frac{
e^{-u^{\mathsf{GD}}(t)}
+
(m-1)e^{-q^{\mathsf{GD}}(t)}
}{
D^{\mathsf{GD}}(t)
},\\[6pt]
\dot v^{\mathsf{GD}}(t)
=
s^2\,
\frac{
e^{-q^{\mathsf{GD}}(t)}
}{
D^{\mathsf{GD}}(t)
}.
\end{cases}
\]
Finally, since \(W^{\mathsf{GD}}(0)=0\), we have
\(
u^{\mathsf{GD}}(0)=v^{\mathsf{GD}}(0)=0.
\)
This proves Proposition~\ref{prop:gd-centered-logit}.
\end{proof}

\subsection{Proof of Proposition~\ref{prop:muon-centered-logit}}
\label{proof_prop:muon-centered-logit}
\begin{proof}
By Lemma~\ref{lem:W-reduction}, whenever \(W=\mu_fQ_f+\mu_cQ_c\), \(u=b\mu_f\), and
\(v=s\mu_c\), we have
\[
\nabla_W\mathcal L(W)
=
\frac{b\,\alpha(u,v)}{C}Q_f
+
\frac{s\,\beta(u,v)}{C}Q_c,
\qquad
\alpha(u,v)<0,\quad \beta(u,v)<0.
\]
Moreover, Lemma~\ref{lem:canonical-full} shows that \(Q_f\) and \(Q_c\)
are orthogonal partial isometries with orthogonal row and column spaces.
Therefore, we have
\(
\spec\!\bigl(\nabla_W\mathcal L(W)\bigr)
=
-Q_f-Q_c
\)
on the canonical plane. Hence the idealized Muon flow satisfies
\[
\dot W^{\mathsf{Muon}}(t)=Q_f+Q_c,
\qquad
W^{\mathsf{Muon}}(0)=0,
\]
and consequently, we have
\(
W^{\mathsf{Muon}}(t)=t(Q_f+Q_c).
\) Multiplying by \(F\) and using Lemma~\ref{lem:canonical-full}, we obtain
\(
W^{\mathsf{Muon}}(t)F
=
t(Q_fF+Q_cF)
=
t(bP_f+sP_c).
\)
Since \(\Pi(uP_f+vP_c)=uP_f+vP_c\) by Lemma~\ref{lem:projectors-full}, we have
\(
Z^{\mathsf{Muon}}(t)
=
 \Pi W^{\mathsf{Muon}}(t)F
=
t(bP_f+sP_c).
\)
Equivalently, we have
\[
 Z^{\mathsf{Muon}}(t)
=
u^{\mathsf{Muon}}(t)P_f+v^{\mathsf{Muon}}(t)P_c,
\]
where
\(
u^{\mathsf{Muon}}(t)=bt,
\,
v^{\mathsf{Muon}}(t)=st.
\)
This proves Proposition~\ref{prop:muon-centered-logit}.
\end{proof}

\subsection{Proof of Proposition~\ref{prop:signgd-centered-logit}}
\label{proof_prop:signgd-centered-logit}
\begin{proof}
Set
\(
N=
\big[\,2R^\top-\mathbf 1_C\mathbf 1_m^\top,\ 2I_C-J\,\big].
\)
We first show that \(W^{\mathsf{Adam}}(t)=tN\). For this candidate
trajectory, direct calculation gives
\[
W^{\mathsf{Adam}}(t)F
=
t\bigl(a(2S-J)+b(2I_C-J)\bigr).
\]
By Lemma~\ref{lem:loss-reduction},
\(
\soft(W^{\mathsf{Adam}}(t)F)-I_C=\alpha_tP_f+\beta_tP_c
\)
for some \(\alpha_t<0\) and \(\beta_t<0\). Let \(p_{0,t},p_{1,t},p_{2,t}\)
be the correct, sibling, and other-group probabilities induced by
\(W^{\mathsf{Adam}}(t)F\). Then, we have
\(
\nabla_W\mathcal L(W^{\mathsf{Adam}}(t))
=
1/C\,[\,aM_tR^\top,\ bM_t\,],
\,
M_t=\soft(W^{\mathsf{Adam}}(t)F)-I_C .
\)
For a row \(c=(g,r)\), the coarse block satisfies
\[
(M_tR^\top)_{c,g}
=
(p_{0,t}-1)+(n-1)p_{1,t}
=
-n(m-1)p_{2,t}<0,
\]
and, for \(g'\neq g\),
\(
(M_tR^\top)_{c,g'}=np_{2,t}>0.
\)
The fine block satisfies
\(
(M_t)_{c,c}=p_{0,t}-1<0,
\,
(M_t)_{c,c'}>0
\,(c'\neq c).
\)
Thus, every entry of the gradient has a fixed nonzero sign along the
candidate trajectory, and hence we have
\(
-\sgn\!\bigl(\nabla_W\mathcal L(W^{\mathsf{Adam}}(t))\bigr)
=
N.
\)
Therefore, the Adam flow reduces to \(\dot W^{\mathsf{Adam}}(t)=N\) with
\(W^{\mathsf{Adam}}(0)=0\), so
\(
W^{\mathsf{Adam}}(t)=tN.
\) Using \(S=nP_c+J/m\), \(I_C=P_f+P_c+J/C\), and \(C=mn\), we get \(2S-J = 2nP_c + n(2-m)J/C\) and \(2I_C-J = 2P_f+2P_c+(2-C)J/C\).
Therefore, we have
\[
W^{\mathsf{Adam}}(t)F
=
2bt\,P_f
+
2(an+b)t\,P_c
+
\frac{t\bigl(an(2-m)+b(2-mn)\bigr)}{C}J.
\]
Thus, we have
\(
u^{\mathsf{Adam}}(t)=2bt,
\,
v^{\mathsf{Adam}}(t)=2(an+b)t,
\,
\zeta^{\mathsf{Adam}}(t)
=
t\bigl(an(2-m)+b(2-mn)\bigr).
\)
Finally, since \(\Pi P_f=P_f\), \(\Pi P_c=P_c\), and \(\Pi J=0\) by Lemma~\ref{lem:projectors-full}, we have
\(
Z^{\mathsf{Adam}}(t)
=
u^{\mathsf{Adam}}(t)P_f
+
v^{\mathsf{Adam}}(t)P_c.
\)
This proves Proposition~\ref{prop:signgd-centered-logit}.
\end{proof}
\subsection{Proof of Proposition \ref{lem:fixed-ratio-margin}}
\label{proof_fixed-ratio-margin}
\begin{proof}
Fix \(\epsilon\in(0,\log C)\). The proof proceeds in three steps.

\begin{itemize}
    \item \textbf{Step 1.} We compute the loss and margin for logits of the form
    \(M=rP_f+\rho rP_c+\zeta J/C\).

    \item \textbf{Step 2.} For each fixed \(\rho\ge1\), we show that there is a unique \(r>0\) for which the corresponding logits reach loss \(\epsilon\).
    We also show that all logits with this scale have the same margin,
    which makes \(r_\epsilon(\rho)\) well-defined.

    \item \textbf{Step 3.} We compare the matched-loss scales corresponding to
    two ratios \(\rho_1<\rho_2\), and show that \(r_\epsilon(\rho)\) is
    strictly decreasing in \(\rho\).
\end{itemize}

\medskip
\noindent\textbf{Step 1.}
Fix \(\rho\ge1\), and consider logits of the form
\(
M=rP_f+\rho rP_c+\zeta J/C .
\)
Define
\(
\kappa_\rho={(n-1+\rho)}/{n},
\,
D_\rho(r)
=
1+(n-1)e^{-r}
+n(m-1)e^{-\kappa_\rho r}.
\)
By Lemma~\ref{lem:loss-reduction}, the loss is
\(
\mathcal L(M)=\log D_\rho(r),
\)
which is independent of \(\zeta\). Since \(\rho\ge1\),
Lemma~\ref{lem:source-margin-symmetry} gives that
\(M_{c,c}-\max_{j\ne c}M_{j,c}=r
,\,\text{for all } c\in\mathcal C,\)
which is also independent of \(\zeta\).

\medskip
\noindent\textbf{Step 2.}
For fixed \(\rho\ge1\), reaching loss \(\epsilon\) is equivalent to
\(
D_\rho(r)=e^\epsilon .
\)
The function \(D_\rho(r)\) is continuous and strictly decreasing in \(r\),
with
\(
D_\rho(0)=C,
\,
\lim_{r\to\infty}D_\rho(r)=1.
\)
Since \(e^\epsilon\in(1,C)\), there exists a unique positive solution
\(r\) to \(D_\rho(r)=e^\epsilon\). Therefore, among logits of the form
\(M=rP_f+\rho rP_c+\zeta J/C\), there is a unique scale \(r>0\) that reaches
loss \(\epsilon\).

By Step 1, all logits with this matched-loss scale have the same margin,
namely \(r\), independently of \(\zeta\). Thus the common margin
defined in the lemma is well-defined, and it equals this unique
matched-loss scale:
\(
r_\epsilon(\rho)=r .
\)

\medskip
\noindent\textbf{Step 3.}
Take \(1\le\rho_1<\rho_2\). Then
\(\kappa_{\rho_1}<\kappa_{\rho_2}\), so for every \(r>0\),
\(
D_{\rho_2}(r)<D_{\rho_1}(r).
\)
Let \(r_i=r_\epsilon(\rho_i)\). Since \(D_{\rho_1}(r_1)=e^\epsilon\), we have
\(
D_{\rho_2}(r_1)<e^\epsilon .
\)
Because \(D_{\rho_2}\) is strictly decreasing in \(r\), its unique solution
to \(D_{\rho_2}(r)=e^\epsilon\) must satisfy \(r_2<r_1\). Therefore
\(\rho\mapsto r_\epsilon(\rho)\) is strictly decreasing on
\([1,\infty)\). This proves Proposition \ref{lem:fixed-ratio-margin}.
\end{proof}

\subsection{Proof of Proposition~\ref{lem:gd-matched-loss-time}}
\label{proof_gd-matched-loss-time}
\begin{proof} We start with the dynamics of GD. Since \(W^{\mathsf{GD}}(t)\) follows gradient flow, we have
\[
\frac{d}{dt}\mathcal L(W^{\mathsf{GD}}(t))
=
\left\langle
\nabla_W\mathcal L(W^{\mathsf{GD}}(t)),
\dot W^{\mathsf{GD}}(t)
\right\rangle_F
=
-\left\|\nabla_W\mathcal L(W^{\mathsf{GD}}(t))\right\|_F^2
<0
\]
for every finite \(t>0\). Hence \(\mathcal L(W^{\mathsf{GD}}(t))\) is
strictly decreasing.

By Proposition~\ref{prop:gd-centered-logit}, the GD centered representation satisfies
\(
Z^{\mathsf{GD}}(t)
=
u^{\mathsf{GD}}(t)P_f+v^{\mathsf{GD}}(t)P_c,
\)
and therefore Lemma~\ref{lem:loss-reduction} gives
\(
\mathcal L(W^{\mathsf{GD}}(t))
=
\log D^{\mathsf{GD}}(t).
\)
At initialization, \(u^{\mathsf{GD}}(0)=v^{\mathsf{GD}}(0)=0\), so
\(D^{\mathsf{GD}}(0)=C\), and hence
\(
\mathcal L(W^{\mathsf{GD}}(0))=\log C.
\)

It remains to show that the loss converges to \(0\). By
Proposition~\ref{prop:gd-centered-logit}, we have
\[
\dot u^{\mathsf{GD}}(t)
=
\frac{b^2}{m}
\frac{
e^{-u^{\mathsf{GD}}(t)}
+
(m-1)e^{-q^{\mathsf{GD}}(t)}
}{
D^{\mathsf{GD}}(t)
}
\ge
\frac{b^2}{mC}e^{-u^{\mathsf{GD}}(t)}.
\]
Direct calculation gives
\[
\frac{d}{dt}e^{u^{\mathsf{GD}}(t)}
=
e^{u^{\mathsf{GD}}(t)}
\dot u^{\mathsf{GD}}(t)
\ge
\frac{b^2}{mC}.
\]
Therefore, \(u^{\mathsf{GD}}(t)\to\infty\). Since
\(v^{\mathsf{GD}}(t)\ge0\), we also have
\(q^{\mathsf{GD}}(t)\to\infty\). Hence
\(D^{\mathsf{GD}}(t)\to1\), so
\(
\mathcal L(W^{\mathsf{GD}}(t))
=
\log D^{\mathsf{GD}}(t)
\to0.
\)

Thus, the loss is continuous and strictly decreasing from \(\log C\) to
\(0\). By the intermediate value theorem, for every
\(\epsilon\in(0,\log C)\), there exists a unique time
\(t_\epsilon^{\mathsf{GD}}>0\) such that
\(
\mathcal L(W^{\mathsf{GD}}(t_\epsilon^{\mathsf{GD}}))=\epsilon .
\)
Next, we consider the dynamics of Muon and Adam. By Proposition~\ref{prop:muon-centered-logit}, the raw output of Muon satisfies
\(
W^{\mathsf{Muon}}(t)F=bt\,P_f+st\,P_c .
\)
Applying Lemma~\ref{lem:loss-reduction} with \(u=bt\) and \(v=st\), we get
\[
\mathcal L(W^{\mathsf{Muon}}(t))
=
\mathcal L(W^{\mathsf{Muon}}(t)F)
=
\log\!\left(
1+(n-1)e^{-bt}
+n(m-1)e^{-((n-1)b+s)t/n}
\right).
\]
This expression is continuous and strictly decreasing in \(t\), equals
\(\log C\) at \(t=0\), and converges to \(0\) as \(t\to\infty\). Similarly, by Proposition~\ref{prop:signgd-centered-logit}, the raw output of Adam satisfies
\(
W^{\mathsf{Adam}}(t)F
=
2bt\,P_f
+
2(an+b)t\,P_c
+
\zeta^{\mathsf{Adam}}(t)J/C .
\)
The additive \(J/C\) component does not affect the loss. Applying
Lemma~\ref{lem:loss-reduction} with \(u=2bt\) and \(v=2(an+b)t\), we get
\[
\mathcal L(W^{\mathsf{Adam}}(t))
=
\mathcal L(W^{\mathsf{Adam}}(t)F)
=
\log\!\left(
1+(n-1)e^{-2bt}
+n(m-1)e^{-2((n-1)b+an+b)t/n}
\right).
\]
This expression is also continuous and strictly decreasing in \(t\), equals
\(\log C\) at \(t=0\), and converges to \(0\) as \(t\to\infty\).

The existence and uniqueness of \(t_\epsilon^{\mathsf{Muon}}\) and
\(t_\epsilon^{\mathsf{Adam}}\) then follow from the intermediate value theorem.
This proves Proposition~\ref{lem:gd-matched-loss-time}.
\end{proof}




\subsection{Proof of Proposition
\ref{lem:gd-ratio-threshold}}
\label{Proof_lem:gd-ratio-threshold}
\begin{proof}
Set
\[
w(t)=v(t)-u(t),
\qquad
g(t)=w(t)-\theta u(t)
=
v(t)-(1+\theta)u(t).
\]
For \(t>0\), since \(u(t)>0\) by Proposition~\ref{prop:gd-centered-logit}, we have
\(
{v(t)}/{u(t)}>\rho_M=1+\theta
\) if and only if \(
g(t)>0.
\)
Thus it suffices to show that \(g(t)>0\) whenever
\(\mathcal L(W^{\mathsf{GD}}(t))\ge\epsilon_\star\).

We proceed in three steps:
\begin{itemize}
    \item \textbf{Step 1.} We show that if \(g(t)\) first reaches zero, then the corresponding
    value of \(u(t)\) must be at least
    \[
    U_\star=\frac{n}{\theta}\log(1+m\theta).
    \]
    \item \textbf{Step 2.} We use this lower bound on \(u(t)\) to show that the GD loss at any
    such first hitting time is already strictly below \(\epsilon_\star\).
    \item \textbf{Step 3.} We conclude by contradiction that, as long as
    \(\mathcal L(W^{\mathsf{GD}}(t))\ge\epsilon_\star\), the hitting event cannot have
    occurred. Hence \(g(t)>0\), equivalently
    \(v(t)/u(t)>\rho_M\).
\end{itemize}

\medskip
\noindent\textbf{Step 1: a lower bound at the first hitting time.}
By Proposition~\ref{prop:gd-centered-logit}, with
\(
q(t)=u(t)+{w(t)}/{n},
\,
D(t)=1+(n-1)e^{-u(t)}+n(m-1)e^{-q(t)},
\)
the GD coordinates satisfy
\begin{equation}
\label{eq:ratio-lemma-u-ode}
\dot u(t)
=
\frac{b^2e^{-u(t)}}{mD(t)}
\left(1+(m-1)e^{-w(t)/n}\right),
\end{equation}
and
\begin{equation}
\label{eq:ratio-lemma-w-ode}
\dot w(t)
=
\dot v(t)-\dot u(t)
=
\frac{b^2e^{-u(t)}}{mD(t)}
\left((1+m\delta)e^{-w(t)/n}-1\right).
\end{equation}
At initialization,
\(
\dot u(0)={b^2}/{C},
\,
\dot v(0)={s^2}/{C}.
\)
Therefore, we have
\(
\lim_{t\downarrow0}{v(t)}/{u(t)}
=
{\dot v(0)}/{\dot u(0)}
=
{s^2}/{b^2}
=
1+\delta.
\)
Since
\(
1+\delta>\sqrt{1+\delta}=1+\theta=\rho_M,
\)
we have \(g(t)>0\) for all sufficiently small \(t>0\). If \(g(t)>0\) for all \(t>0\), then the lemma is already proved. Otherwise,
define the first hitting time
\(
t_\star=\inf\{t>0:\ g(t)\le0\}.
\)
By construction, we have \(g(t)>0\) for all \(0<t<t_\star\), and \(
g(t_\star)=0,
\,
\dot g(t_\star)\le0.
\)
On the boundary \(g=0\), we have \(w=\theta u\). Substituting
\(w=\theta u\) into Eqn.~\eqref{eq:ratio-lemma-u-ode} and
\eqref{eq:ratio-lemma-w-ode}, and using
\(
1+\delta=(1+\theta)^2,
\)
we have
\begin{align}
\dot g
&=
\dot w-\theta\dot u =
\frac{b^2e^{-u}}{mD}
(1+\theta)
\left((1+m\theta)e^{-\theta u/n}-1\right).
\label{eq:ratio-lemma-g-boundary}
\end{align}
Hence \(\dot g>0\) whenever
\(
u<U_\star,
\,
U_\star=\frac{n}{\theta}\log(1+m\theta).
\)
This contradicts \(\dot g(t_\star)\le0\). Therefore every first hitting time
must satisfy
\begin{equation}
\label{eq:ratio-lemma-u-lower}
u(t_\star)\ge U_\star
=
\frac{n}{\theta}\log(1+m\theta).
\end{equation}

\medskip
\noindent\textbf{Step 2: an upper bound on the loss at the first hitting time.}
By definition, at \(t_\star\), we have \(g(t_\star)=0\), which is equivalent to 
\(
v(t_\star)=(1+\theta)u(t_\star).
\)
Consequently, we have
\(
q(t_\star)
=
{\left((n-1)u(t_\star)+v(t_\star)\right)}/{n}=\left(1+{\theta}/{n}\right)u(t_\star).
\)
By Lemma~\ref{lem:loss-reduction} and direct calculation, we have
\begin{align}
\mathcal L(W^{\mathsf{GD}}(t_\star))
&=
\log\left(
1+(n-1)e^{-u(t_\star)}
+n(m-1)e^{-(1+\theta/n)u(t_\star)}
\right) \notag \\
&\le
\log\left(1+mn\,e^{-u(t_\star)}\right) \notag \\
&<
mn\,e^{-u(t_\star)} .
\label{eq:ratio-lemma-loss-hit-raw}
\end{align}
Using Eqn.~\eqref{eq:ratio-lemma-u-lower}, we get
\(
e^{-u(t_\star)}
\le
e^{-U_\star}
=
(1+m\theta)^{-n/\theta}.
\)
Combining this with Eqn.~\eqref{eq:ratio-lemma-loss-hit-raw} gives
\begin{equation}
\label{eq:ratio-lemma-loss-hit}
\mathcal L(W^{\mathsf{GD}}(t_\star))
<
mn(1+m\theta)^{-n/\theta}
=
\epsilon_\star.
\end{equation}

\medskip
\noindent\textbf{Step 3: the explicit loss window.}
Suppose, for contradiction, that there exists \(t>0\) such that
\(
\mathcal L(W^{\mathsf{GD}}(t))\ge\epsilon_\star
\) and \(
{v(t)}/{u(t)}\le\rho_M.
\)
Equivalently, \(g(t)\le0\). Since \(g(t)>0\) for all sufficiently small
\(t>0\), the first hitting time \(t_\star\) exists and satisfies
\(t_\star\le t\). By Proposition~\ref{lem:gd-matched-loss-time}, the GD loss is strictly
decreasing along the trajectory. Hence
\(
\mathcal L(W^{\mathsf{GD}}(t))
\le
\mathcal L(W^{\mathsf{GD}}(t_\star))
<
\epsilon_\star,
\)
where the strict inequality follows from Eqn.~\eqref{eq:ratio-lemma-loss-hit}.
This contradicts \(\mathcal L(W^{\mathsf{GD}}(t))\ge\epsilon_\star\). Therefore,
\(
\mathcal L(W^{\mathsf{GD}}(t))\ge\epsilon_\star
\) implies \(
{v(t)}/{u(t)}>\rho_M .
\) The equivalent exact-loss statement follows immediately by taking
\(t=t_\epsilon^{\mathsf{GD}}\). This proves Proposition
\ref{lem:gd-ratio-threshold}.
\end{proof}

%% file: appendix/proof_lemma.tex
\section{Proofs of Supporting Lemmas}
\label{sec:proofs-support}
\begin{lemma}[Projector algebra]
\label{lem:projectors-full}
The matrices \(P_f\) and \(P_c\) satisfy
\[
P_f^2=P_f,\qquad P_c^2=P_c,\qquad P_fP_c=P_cP_f=0,\qquad \Pi=P_f+P_c.\]
In particular, \(\Pi P_f=P_f\) and \(\Pi P_c=P_c\). Moreover, \(\Pi J=0\). 
Their ranks are
\(
r_f=\operatorname{rank}(P_f)=m(n-1),
\,
r_c=\operatorname{rank}(P_c)=m-1.
\)
\end{lemma}

\begin{proof}
Since \(S=I_m\otimes J_n\), we have \(S^2=nS\), \(SJ=JS=nJ\), and \(J^2=CJ\).
Direct expansion gives
\[
P_f^2=\Big(I_C-\frac1nS\Big)^2=I_C-\frac2nS+\frac1{n^2}S^2=I_C-\frac1nS=P_f,
\]
and similarly, we have
\[
P_c^2=\Big(\frac1nS-\frac1C J\Big)^2
=\frac1{n^2}S^2-\frac2{nC}SJ+\frac1{C^2}J^2
=\frac1nS-\frac1C J=P_c.
\]
Direct calculation gives
\[
P_fP_c=\Big(I_C-\frac1nS\Big)\Big(\frac1nS-\frac1CJ\Big)=0,
\qquad
P_cP_f=0, \qquad
P_f+P_c=I_C-\frac1nS+\frac1nS-\frac1CJ=\Pi.
\]
Combining \(P_f^2=P_f\), \(P_cP_f=0\), and \(P_f+P_c=\Pi\) yields \(\Pi P_f=P_f\); the identity \(\Pi P_c=P_c\) follows symmetrically. Finally, using \(SJ = nJ\) and \(J^2 = CJ\), direct expansion gives \(P_fJ = J - SJ/n = 0\) and \(P_cJ = SJ/n - J^2/C = 0\), hence \(\Pi J = P_fJ + P_cJ = 0\). The rank statements follow from the dimensions of the within-group and between-group contrast subspaces. 
This proves Lemma~\ref{lem:projectors-full}.
\end{proof}

\begin{lemma}[Loss on the exact symmetry class]
\label{lem:loss-reduction}
Let
\(
M=u P_f+v P_c+\zeta {J}/{C}\in\mathcal M,\,
q={\big((n-1)u+v}\big)/{n},
\,
D(u,v)=1+(n-1)e^{-u}+n(m-1)e^{-q}.
\)
Then every training example has the same loss, and we have
\[
\mathcal L_{ }(M)=\log D(u,v).
\]
Moreover, we have
\(
\soft(M)-I_C=\alpha(u,v)P_f+\beta(u,v)P_c,
\)
where \(\soft(M)\) denotes the columnwise softmax matrix and
\[
\alpha(u,v)
=
-\frac{n\big(e^{-u}+(m-1)e^{-q}\big)}{D(u,v)},
\qquad
\beta(u,v)
=
-\frac{mn\,e^{-q}}{D(u,v)}.
\]
In particular, \(\alpha(u,v)<0\) and \(\beta(u,v)<0\).
\end{lemma}

\begin{proof}
Fix a column \(j=(g,r)\).
Since \(\zeta J/C\) adds the same scalar \(\zeta/C\) to every entry of every column,
it does not affect the columnwise softmax.
Hence it suffices to compute \(\soft(uP_f+vP_c)\).

In column \(j\), there are three logit values:
\[
z_{\mathrm{corr}}
=
u\Big(1-\frac1n\Big)+v\Big(\frac1n-\frac1C\Big),
\qquad
z_{\mathrm{sib}}
=
-\frac{u}{n}+v\Big(\frac1n-\frac1C\Big),\qquad
z_{\mathrm{other}}
=
-\frac{v}{C}.
\]
Therefore, we have
\(
z_{\mathrm{corr}}-z_{\mathrm{sib}}=u,
\,
z_{\mathrm{corr}}-z_{\mathrm{other}}={\left((n-1)u+v\right)}/{n}=q.
\)
Hence the correct probability is
\(
p_0=1/{D(u,v)},
\)
each sibling probability is
\(
p_1={e^{-u}}/{D(u,v)},
\)
and each other-group probability is
\(
p_2={e^{-q}}/{D(u,v)}.
\)
Thus every training example has the same loss and
\(
\mathcal L_{ }(M)=-\log p_0=\log D(u,v).
\)

Now the \(j\)-th column of \(\soft(M)-I_C\) takes the value
\(
p_0-1\) {on the correct entry,}
\(
p_1\) {on siblings,} \(p_2\) {on other-group entries.}
This is exactly the three-level pattern of a matrix of the form \(\alpha P_f+\beta P_c\).
Matching the other-group value gives $-{\beta}/{C}=p_2$, which gives
\(
\beta=-Cp_2=-{mn\,e^{-q}}/{D(u,v)}.
\)
Matching the difference between a correct entry and a sibling entry gives
\(
\alpha=(p_0-1)-p_1
=
-{n\big(e^{-u}+(m-1)e^{-q}\big)}/{D(u,v)}.
\)
The negativity of \(\alpha,\beta\) is immediate.
This proves Lemma~\ref{lem:loss-reduction}.
\end{proof}

\begin{lemma}[Canonical feature directions]
\label{lem:canonical-full}
Let \(s=\sqrt{a^2n+b^2}\). The following identities hold.

\begin{enumerate}
    \item The canonical directions are orthogonal:
    \[
    Q_fQ_f^\top=P_f,\qquad
    Q_cQ_c^\top=P_c,\qquad
    Q_fQ_c^\top=0.
    \]

    \item Multiplication by the feature matrix \(F\) gives the fine and
    coarse logit components:
    \[
    Q_fF=bP_f,
    \qquad
    Q_cF=sP_c.
    \]

    \item For every \(u,v\in\mathbb R\), the canonical weight
    $W_{\mathrm{cano}}(u,v)=u/b\,Q_f+v/s\,Q_c$
    satisfies
    \[
    W_{\mathrm{cano}}(u,v)F
    =
    \Pi W_{\mathrm{cano}}(u,v)F
    =
    uP_f+vP_c.
    \]
\end{enumerate}
\end{lemma}

\begin{proof}
By Lemma~\ref{lem:projectors-full}, we have \(P_fP_c=0\). Moreover, by direct calculation, we have \(P_fR^\top=0\). Hence, we have
\[
Q_fQ_c^\top=\frac1s\big(0\cdot aRP_c + P_f\cdot bP_c\big)=0,
\qquad
Q_fQ_f^\top=P_f^2=P_f.
\]
Also, direct calculation gives
\[
Q_cQ_c^\top
=
\frac1{s^2}\big(a^2P_cR^\top RP_c+b^2P_c^2\big)
=
\frac1{s^2}\big(a^2P_cSP_c+b^2P_c\big)
=
\frac1{s^2}(a^2n+b^2)P_c
=P_c\]
\[Q_fF=[\,0,\ P_f\,]\begin{bmatrix}aR\\ bI_C\end{bmatrix}=bP_f,
\]
and
\[
Q_cF
=
\frac1s[\,aP_cR^\top,\ bP_c\,]\begin{bmatrix}aR\\ bI_C\end{bmatrix}
=
\frac1s(a^2P_cR^\top R+b^2P_c)
=
\frac1s(a^2P_cS+b^2P_c)
=
\frac1s(a^2n+b^2)P_c
=
sP_c.
\]
The final identity follows immediately. This proves Lemma~\ref{lem:canonical-full}.
\end{proof}

\begin{lemma}[Canonical-plane closure in weight space]
\label{lem:W-reduction}
Let
\(
\mathcal W_{ \mathrm{cano}}=\operatorname{span}\{Q_f,Q_c\}\subset\mathbb R^{C\times(m+C)}.
\)
If \(W\in\mathcal W_{ \mathrm{cano}}\) and
\(
W=\mu_fQ_f+\mu_cQ_c,
\,
u=b\mu_f,\, v=s\mu_c,
\) we have
\(
WF=uP_f+vP_c\in\mathcal M,
\)
and
\(
\mathcal L_{ }(W)=\log D(u,v),
\,
\soft(WF)-I_C=\alpha(u,v)P_f+\beta(u,v)P_c.
\)
Moreover, we have
\(
\nabla_W\mathcal L_{ }(W)
=
{b\,\alpha(u,v)}/{C}\cdot Q_f
+
{s\,\beta(u,v)}/{C}\cdot Q_c
\in \mathcal W_{ \mathrm{cano}}.
\)
In particular, the plane \(\mathcal W_{ \mathrm{cano}}\) is invariant under GD.
Since \(\alpha(u,v),\beta(u,v)<0\), it is also invariant under idealized Muon.
\end{lemma}

\begin{proof}
By Lemma~\ref{lem:canonical-full}, direct calculation gives
\[
WF=\mu_fQ_fF+\mu_cQ_cF=b\mu_f\,P_f+s\mu_c\,P_c=uP_f+vP_c\in\mathcal M.
\]
By Lemma~\ref{lem:loss-reduction}, we have
\(
\mathcal L_{ }(W)=\log D(u,v),
\,
\soft(WF)-I_C=\alpha(u,v)P_f+\beta(u,v)P_c.
\) By the chain rule, we have
\(
\nabla_W\mathcal L_{ }(W)
=
1/C\big(\soft(WF)-I_C\big)F^\top
=
\frac1C(\alpha P_f+\beta P_c)[\,aR^\top,\ bI_C\,].
\)
Using \(P_fR^\top=0\), we get
\(
(\alpha P_f+\beta P_c)[\,aR^\top,\ bI_C\,]
=
[\,a\beta P_cR^\top,\ b\alpha P_f+b\beta P_c\,]
=
b\alpha Q_f+s\beta Q_c.
\)
Therefore, we have
\[
\nabla_W\mathcal L_{ }(W)
=
\frac{b\,\alpha(u,v)}{C}\cdot Q_f+\frac{s\,\beta(u,v)}{C}\cdot Q_c
\in\mathcal W_{ \mathrm{cano}}.
\]

Hence GD preserves \(\mathcal W_{ \mathrm{cano}}\).
For Muon, since \(\alpha(u,v),\beta(u,v)<0\) and \(Q_f,Q_c\) are orthogonal partial isometries,
the compact SVD of \(\nabla_W\mathcal L_{ }(W)\) is obtained by concatenating the two blocks, so
\[
\spec\!\big(\nabla_W\mathcal L_{ }(W)\big)
=
-Q_f-Q_c\in\mathcal W_{ \mathrm{cano}}.
\]
Thus, Muon also preserves \(\mathcal W_{ \mathrm{cano}}\). The final identity follows immediately. This proves Lemma~\ref{lem:W-reduction}.
\end{proof}

\begin{lemma}[Margin on the symmetry class]
\label{lem:source-margin-symmetry}
Let
\(
M=uP_f+vP_c+\zeta{J}/{C}\in\mathcal M,
\,
u,v\ge0.
\)
Then every class has the same margin, i.e.,
\(
M_{c,c}-\max_{j\ne c}M_{j,c}=\min\left\{
u,\,
{\left((n-1)u+v\right)}/{n}
\right\}
\quad \text{for all } c\in\mathcal C.
\)
In particular, if \(v\ge u\), then we have
\(
M_{c,c}-\max_{j\ne c}M_{j,c}
=u
\) for all \(c\in\mathcal C\).
\end{lemma}

\begin{proof}
Fix a column \(c=(g,r)\). The additive term
\(\zeta J/C\) shifts all logits in this column by the same constant, so it does
not affect the margin. Hence it suffices to compute the margin of \(uP_f+vP_c\).

For the correct class, a sibling class in the same coarse group, and a class in
a different coarse group, the logits are respectively
\[
z_{\mathrm{corr}}
=
u\left(1-\frac1n\right)
+
v\left(\frac1n-\frac1C\right),
\qquad
z_{\mathrm{sib}}
=
-\frac{u}{n}
+
v\left(\frac1n-\frac1C\right),
\qquad
z_{\mathrm{other}}
=
-\frac{v}{C}.
\]
Therefore, we have
\(
z_{\mathrm{corr}}-z_{\mathrm{sib}}=u,
\,
z_{\mathrm{corr}}-z_{\mathrm{other}}
=
{\left((n-1)u+v\right)}/{n}.
\)
The largest incorrect logit is the larger of \(z_{\mathrm{sib}}\) and
\(z_{\mathrm{other}}\), so the margin equals
\(
\min\left\{
u,\,
{\left((n-1)u+v\right)}/{n}
\right\}.
\)
This value is independent of \(c\), and therefore every class has the same margin. If \(v\ge u\), then we have
\(
{\left((n-1)u+v\right)}/{n}\ge u,
\)
so the margin reduces to \(u\).
This proves Lemma~\ref{lem:source-margin-symmetry}.
\end{proof}

\begin{lemma}[Scale of the explicit threshold]
\label{lem:eps-star-scale}
Fix constants \(0<a\le b\), and let \(C=mn\). For \(m,n\ge2\), define
\[
\theta
=
\sqrt{1+\frac{a^2n}{b^2}}-1,
\qquad
\epsilon_\star
=
mn(1+m\theta)^{-n/\theta}.
\]
Then the following statements hold.

\begin{enumerate}
    \item
    The explicit threshold lies in the valid loss range:
    \(0<\epsilon_\star<1<\log C .
    \)

    \item
    There exists a constant \(c_0>0\), depending only on \(a,b\), such that
    for every \(0<c<c_0\),
    \(
    \epsilon_\star
    =
    o\!\left(
    mn(1+m\sqrt n)^{-c\sqrt n}
    \right)
    \)
    along any sequence \(m=m(n)\ge2\) with \(n\to\infty\). Consequently, if
    \(
    \epsilon
    =
    \Omega\!\left(
    mn(1+m\sqrt n)^{-c\sqrt n}
    \right)
    \)
    for some \(0<c<c_0\), then
    \(
    \epsilon\ge \epsilon_\star
    \)
    for all sufficiently large \(n\).

    \item 
    If \(m=\Theta(n)\), then we have
    \(
    \epsilon_\star
    =
    \exp\!\left(-\Theta(\sqrt n\log n)\right)
    =
    n^{-\Theta(\sqrt n)}.
    \)
    In particular, there exists a constant \(c'_0>0\), depending only on
    \(a,b\), such that for every \(0<c'<c'_0\), we have
    \(
    \epsilon_\star
    =
    o\!\left(\exp(-c'\sqrt n\log n)\right)
    =
    o\!\left(n^{-c'\sqrt n}\right).
    \)
    Consequently, if
    \(
    \epsilon
    =
    \Omega\!\left(\exp(-c'\sqrt n\log n)\right)
    =
    \Omega\!\left(n^{-c'\sqrt n}\right)
    \)
    for some \(0<c'<c'_0\), then we have
    \(
    \epsilon\ge \epsilon_\star
    \)
    for all sufficiently large \(n\).
\end{enumerate}
\end{lemma}

\begin{proof}
First, \(\epsilon_\star>0\) is immediate. Since \(a\le b\), we have
\(
\theta
=
\sqrt{1+{a^2n}/{b^2}}-1
\le
\sqrt{1+n}-1
\le n.
\)
Thus, \(n/\theta\ge1\). By Bernoulli's inequality, we have
\(
(1+m\theta)^{n/\theta}
\ge
1+({n}/{\theta})\cdot m\theta
=
1+mn.
\)
Therefore, we have
\[
0<\epsilon_\star
=
mn(1+m\theta)^{-n/\theta}
\le
\frac{mn}{1+mn}
<1<\log C,
\]
where \(C=mn\ge4\).

We next prove the general high-dimensional scale. Since \(a,b\) are fixed
positive constants, we have
\(
\theta
=
\sqrt{1+{a^2n}/{b^2}}-1
=
\Theta(\sqrt n).
\)
Hence there exist constants \(c_1,c_2>0\), depending only on \(a,b\), such
that for all sufficiently large \(n\), we have
\(
c_1\sqrt n\le \theta\le c_2\sqrt n.
\)
It follows that
\[
\frac{n}{\theta}\ge \frac{1}{c_2}\sqrt n,
\qquad
1+m\theta\ge 1+c_1m\sqrt n.
\]
Let
\(
\phi(x)={\log(1+c_1x)}/{\log(1+x)},
\, x>0.
\)
Since \(\phi(x)\to1\) as \(x\to\infty\), there exists a constant
\(c_3>0\), depending only on \(a,b\), such that
\(
\log(1+c_1x)\ge c_3\log(1+x)
\)
for all sufficiently large \(x\). Applying this with \(x=m\sqrt n\), we get
\(
\log(1+m\theta)
\ge
c_3\log(1+m\sqrt n)
\)
for all sufficiently large \(m,n\). Therefore, we have
\[
\log\epsilon_\star
=
\log(mn)-\frac{n}{\theta}\log(1+m\theta) \notag 
\le
\log(mn)
-
\frac{c_3}{c_2}\sqrt n\log(1+m\sqrt n).
\]
Let \(c_0=c_3/c_2\). Then, we have
\(
\epsilon_\star
\le
mn(1+m\sqrt n)^{-c_0\sqrt n}
\)
for all sufficiently large \(m,n\). For any \(0<c<c_0\), we have
\[
\frac{\epsilon_\star}
     {mn(1+m\sqrt n)^{-c\sqrt n}}
\le
(1+m\sqrt n)^{-c\sqrt n}
\to0.
\]
Thus, we have
\(
\epsilon_\star
=
o\!\left(
mn(1+m\sqrt n)^{-c\sqrt n}
\right).
\)

Finally, suppose \(m=\Theta(n)\). Since \(a,b\) are fixed, we have
\(
\theta=\Theta(\sqrt n),
\,
m\theta=\Theta(n^{3/2}),
\,
\frac{n}{\theta}=\Theta(\sqrt n).
\)
From the definition of \(\epsilon_\star\), we obtain
\(\log\epsilon_\star = \log(mn) - (n/\theta)\log(1+m\theta)\)
Moreover, we have
\(
\log(mn)=\Theta(\log n),
\,
\log(1+m\theta)=\Theta(\log n).
\)
Direct calculation gives
\(
\log\epsilon_\star
=
-\Theta(\sqrt n\log n),
\)
and, therefore, we have
\(
\epsilon_\star
=
\exp(-\Theta(\sqrt n\log n))
=
n^{-\Theta(\sqrt n)}.
\)
The final little-\(o\) statement follows by choosing \(c'>0\) smaller than
the implicit constant in the lower bound of
\(-\log\epsilon_\star=\Theta(\sqrt n\log n)\). This proves Lemma~\ref{lem:eps-star-scale}.
\end{proof}